\documentclass{article}
\usepackage{arxivstyle,times}
\usepackage{amsmath,amssymb,amsthm,mathtools}
\usepackage{graphicx,booktabs,multirow}
\usepackage{xcolor}
\usepackage{hyperref}
\usepackage{url}
\usepackage{microtype}

\newtheorem{theorem}{Theorem}
\newtheorem{proposition}{Proposition}
\newtheorem{corollary}{Corollary}
\newtheorem{lemma}{Lemma}
\newtheorem{definition}{Definition}
\newtheorem{assumption}{Assumption}
\theoremstyle{remark}
\newtheorem{remark}{Remark}

\newcommand{\R}{\mathbb{R}}
\newcommand{\E}{\mathbb{E}}
\newcommand{\Pd}{\mathbb{P}}
\newcommand{\one}{\mathbf{1}}
\newcommand{\cG}{\mathcal{G}}
\newcommand{\cC}{\mathcal{C}}
\newcommand{\cX}{\mathcal{X}}
\newcommand{\cY}{\mathcal{Y}}
\newcommand{\cD}{\mathcal{D}}
\newcommand{\cA}{\mathcal{A}}
\newcommand{\cL}{\mathcal{L}}
\newcommand{\cP}{\mathcal{P}}
\newcommand{\tOh}{\widetilde{O}}
\newcommand{\ttau}{\widetilde{\tau}}
\newcommand{\ReCal}{\textsc{ReCal}}
\DeclareMathOperator{\dtv}{d_{\mathrm{TV}}}

\title{Replicable Conformal Prediction}

\iclrfinalcopy
\author{
Marios Papamichalis\thanks{\raggedright Human Nature Lab, Yale University, New Haven, CT 06511, USA. \href{mailto:marios.papamichalis@yale.edu}{\nolinkurl{marios.papamichalis@yale.edu}}}
\and Regina Ruane\thanks{\raggedright Department of Statistics and Data Science, The Wharton School, University of Pennsylvania, Philadelphia, PA, USA. \href{mailto:ruanej@wharton.upenn.edu}{\nolinkurl{ruanej@wharton.upenn.edu}}}
\and Theofanis Papamichalis\thanks{\raggedright Department of Economics, Yale University, New Haven, CT, USA. \href{mailto:theofanis.papamichalis@yale.edu}{\nolinkurl{theofanis.papamichalis@yale.edu}}}
}

\begin{document}
\maketitle

\begin{abstract}
Two analysts who calibrate a conformal predictor on independent samples deploy different set-valued classifiers with probability one, because the calibrated threshold is an order statistic of continuous scores. We ask when the identical classifier can be deployed, and at what cost. Exact agreement is impossible: a procedure that returns one fixed output with high probability cannot be uniformly valid, and with a shared seed an exactly replicable procedure must ignore its data. Rounding the threshold upward on a shared-seed random grid (replicable calibration, \ReCal) makes the deployed classifier identical across analysts with probability at least $1-\rho$, preserves marginal coverage, keeps training-conditional coverage within $\varepsilon$ of target, and needs $O(\kappa^{2}\alpha(1-\alpha)/(\varepsilon^{2}\rho^{2}))$ calibration points at miscoverage level $\alpha$, where $\kappa$ bounds the local score-density ratio. Lower bounds with explicit constants show the sample cost and a coverage inflation of order $\sqrt{\alpha(1-\alpha)/n}/\rho$ at calibration size $n$ are unavoidable for any threshold calibrator; a consistent plug-in removes the factor $\kappa^{2}$ asymptotically. Without any seed, a deterministic grid confines all analysts to two adjacent classifiers, optimally. Selecting the most favorable of $M$ recalibrations changes a \ReCal{} classifier with probability at most $(M-1)\rho$; the same selection silently undercovers standard split conformal. Experiments on real ImageNet outputs, a four-hospital site split, four language-model families, and synthetic testbeds match the theory, including the predicted $(\varepsilon\rho)^{-2}$ sample frontier.
\end{abstract}

\section{Introduction}\label{sec:intro}

Conformal prediction converts a trained model into a set-valued classifier with a finite-sample coverage guarantee \citep{vovk2005,lei2018,angelopoulos2021gentle}: a threshold $\ttau$ is calibrated on $n$ held-out points, the classifier $\cC(x)=\{y:s(x,y)\le\ttau\}$ is deployed, and the probability that the true label is covered is at least $1-\alpha$. In deployment the calibrated classifier is an artifact: hashed, cached, versioned, audited. Three examples motivate this paper: two hospitals validate one frozen diagnostic model on their own held-out data before a joint deployment; a regulator re-runs a vendor's calibration during an audit; a serving system caches conformal answer sets for language-model decoding \citep{quach2024}. In each case the question is whether two independent calibrations produce the identical classifier, a stronger event than achieving similar coverage.

With standard split conformal they never do. The calibrated threshold is an order statistic of continuously distributed scores, so two independent calibrations agree with probability zero (Proposition~\ref{prop:vanilla}), and realized coverage fluctuates across calibration draws according to an exact Beta law \citep{vovk2012}. The instability is exploitable: a party that redraws its calibration set $M$ times and keeps the most favorable draw deploys sets that undercover by an amount on the order of $\sqrt{\log M}$ times the coverage standard deviation (Corollary~\ref{cor:attack}), while the artifact passes any support check (Section~\ref{sec:exp}, Figure~\ref{fig:hero}b). A selective-inference failure at the deployment layer \citep{berk2013,taylor2015}: candidate-wise validity does not survive selection \citep{fithian2014}; replicability restores selection robustness (Corollary~\ref{cor:attack}). Cross-sample instability is documented at benchmark scale \citep{prastalo2026}; for prediction sets, built to be audited, it undermines that function.

A separate line of work formalizes agreement across independent samples. An algorithm is $\rho$-replicable if two runs on independent samples that share an internal random seed return the identical output with probability at least $1-\rho$ \citep{ilps2022}; the theory covers costs, list relaxations, and connections to privacy and stability \citep{bun2023,dixon2023,chase2023,hopkins2024}. To our knowledge the framework has not previously been applied to conformal prediction or distribution-free uncertainty quantification; this paper does so and characterizes the achievable trade-offs for conformal calibration.

\begin{figure}[t]
\centering
\includegraphics[width=0.70\linewidth]{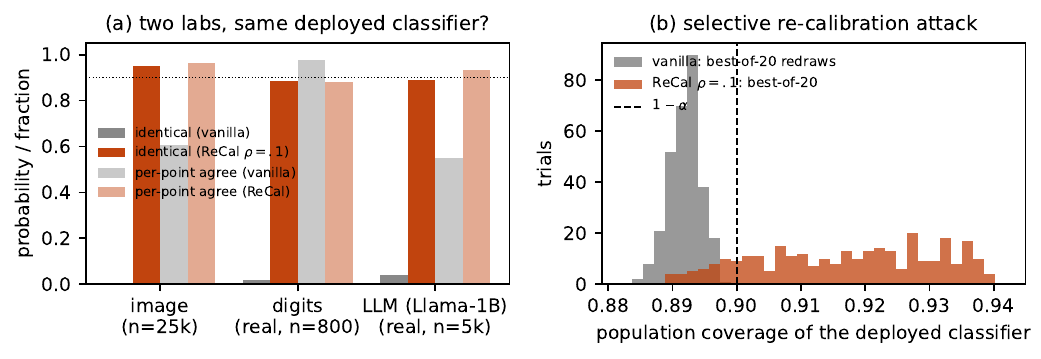}
\caption{(a) Identity (solid) and pointwise agreement (light) for two independent calibrations, derandomized standard versus \ReCal{} at $\rho{=}0.1$; baseline identity is zero for continuous scores (Proposition~\ref{prop:vanilla}), and between two honest analysts Llama-3.2-1B's answer sets change on $45\%$ of contexts ($76\%$ GPT-2; Table~\ref{tab:realllm}). (b) An adversary redraws the calibration set $M{=}20$ times and deploys the lowest-coverage draw, chosen by full-pool coverage (an oracle stress test): standard split conformal then undercovers systematically, undetectably by a support check; under \ReCal{} all $M$ draws yield one classifier except with probability at most $\min\{1,(M{-}1)\rho\}$ ($.19$ at $\rho{=}.01$; vacuous at the plotted $.1$; Corollary~\ref{cor:attack}, F5).}
\label{fig:hero}
\end{figure}

\paragraph{Questions.} First, can any procedure return one fixed classifier with high probability and remain valid for every distribution? Second, we determine which relaxed agreements are achievable, at what sample and set-size cost, and whether that cost is necessary. Third, we quantify what replicable calibration guarantees against selective recalibration.

\paragraph{Answers.} Exact agreement is impossible in a strong sense: any procedure that outputs one fixed classifier with probability $1-\delta$ and is uniformly valid to within $\varepsilon$ cannot exist when $\varepsilon+\delta<\min\{\alpha,1-\alpha\}$; with a shared seed, exact replicability forces the output law to be the same under every distribution, so the deployed sets are uninformative (Theorem~\ref{thm:impossible}). Rounding the threshold upward on a random grid whose offset is generated from one shared seed (\ReCal; Algorithm~1) makes the deployed classifier $\rho$-replicable, preserves the marginal guarantee, and keeps a two-sided training-conditional band, at sample complexity $O(\kappa^{2}\alpha(1-\alpha)/(\varepsilon^{2}\rho^{2}))$ under a local margin condition (Theorem~\ref{thm:upper}). A direct argument shows $n\ge\tfrac{9}{16384}\,\alpha(1-\alpha)/(\varepsilon^{2}\rho^{2})$ is necessary for any threshold calibrator (Theorem~\ref{thm:lower}), and the same machinery bounds the set-size cost: any replicable threshold calibrator inflates coverage by $\Omega(\sqrt{\alpha(1-\alpha)/n}/\rho)$, which \ReCal{} attains up to $\kappa^{2}$ (Corollary~\ref{cor:sizelb}); a consistent pilot removes the prior constant and the $\kappa^{2}$ asymptotically (Proposition~\ref{prop:adaptive}). Without a shared seed, single-answer replication is impossible below failure probability $1/2$; a deterministic grid places every analyst, with probability $1-\delta$, in a common two-element list of adjacent classifiers at rate $O(\kappa^{2}\log(1/\delta)/\varepsilon^{2})$, optimal within the margin class (Theorem~\ref{thm:twolist}).

\paragraph{What exact identity provides, and what it costs.} Identity is the relevant event wherever artifacts are hashed, cached, or certified, and it bounds selection: with a pre-registered seed, $M$ calibration redraws yield a second distinct classifier with probability at most $\min\{1,(M-1)\rho\}$, and expected coverage degrades by at most the same amount (Corollary~\ref{cor:attack}); continuous stability notions such as one-point stability \citep{ndiaye2022} do not provide this. Identity has a cost beyond set size: \ReCal{} concentrates disagreement into rare full-grid-cell events, so on small-sample real data pointwise agreement can decrease while identity rises from $0$ to $0.9$; we therefore report a full spectrum of replication metrics (Definition~\ref{def:spectrum}).

\paragraph{Contributions.} \textbf{(C1)} Impossibility theory: obliviousness of exact replication, necessity of shared randomness, and a seedless two-element-list frontier with matching optimality (Theorems~\ref{thm:impossible} and \ref{thm:twolist}). \textbf{(C2)} \ReCal{}, a replicable split-conformal calibrator with preserved validity and explicit constants, with corollaries for unequal calibration sizes, selective recalibration, and robustness to score and population mismatch, and a pilot procedure with confidence bands whose single prior input ($\hat\kappa\ge\kappa$) is identified, stress-tested, and asymptotically unnecessary (Theorem~\ref{thm:upper}, Corollaries~\ref{cor:protocol}, \ref{cor:twolabs}, and \ref{cor:attack}, Corollary~\ref{cor:robust} in Appendix~\ref{app:upper}, Propositions~\ref{prop:pilot} and~\ref{prop:adaptive}). \textbf{(C3)} Lower bounds with the $\alpha(1-\alpha)$ dependence, proved directly with explicit constants and attained up to the factor $\kappa^{2}$, for the sample size and for the set-size cost, with a lemma localizing where the cost surfaces (Theorem~\ref{thm:lower}, Corollary~\ref{cor:sizelb}, Lemma~\ref{lem:churn}). \textbf{(C4)} An instrumented version of the standard conformal evaluation \citep{romano2020,angelopoulos2021sets} adding replication metrics (a set-valued, calibration-resampling analogue of the churn analysis of \citet{prastalo2026}), applied to a synthetic image-scale testbed, a real small-data task, and next-token logits from four language-model families ($50$k--$152$k vocabularies), with direct measurements of the frontier and the selection attack; all synthetic results regenerate deterministically on one CPU core.

\section{Related work}\label{sec:related}

\textbf{Replicability.} \citet{ilps2022} introduced $\rho$-replicability with replicable statistical queries and a coin-problem lower bound of $\Omega(1/\rho^{2})$ at fixed tolerance; \citet{bun2023} connected replicability, differential privacy, and perfect generalization (with an inherent quadratic overhead in one direction); \citet{chase2023,chase2024} tie replicability to global stability and Borsuk--Ulam obstructions underlying list-size lower bounds; \citet{dixon2023} formalize list/certificate replicability; \citet{hopkins2024} resolve the $N$-coin problem via isoperimetry; replicable bandits, RL, and testing followed \citep{karbasi2023,eaton2026,liu2024}. Closest algorithmically is the replicable quantile of \citet[Thm.~B.6]{kalavasis2024} ($\tOh(\log^{2}R/(\alpha^{2}\rho^{2}))$ over an $R$-point grid via statistical-query (SQ) binary search); relative to it we contribute the conformal semantics (set-level identity with the finite-sample guarantee kept, two-sided \emph{coverage} bands, unequal-$n$ analysts), constants via the exact Beta law (no $\log^{2}R$; Remark~\ref{rem:constants}), the coverage-metric lower bound, and the obliviousness/list theory; their construction is our distribution-free fallback (Proposition~\ref{prop:distfree}).

\textbf{Conformal prediction.} Foundations and training-conditional behavior: \citet{vovk2005,lei2018,vovk2012,bian2023}; APS/RAPS define the scores and evaluation we instrument \citep{romano2020,angelopoulos2021sets}; nested families \citep{gupta2022} coincide with our output class; conformal risk control (CRC) and conformal-LM stopping \citep{angelopoulos2024crc,quach2024} calibrate one scalar, so \ReCal{} applies unchanged (we instantiate a risk-controlled variant in Section~\ref{sec:exp}). One-point stability \citep{ndiaye2022} and e-value derandomization \citep{bashari2023} are within-run notions, orthogonal to cross-sample identity. \citet{balinsky2025} ask when one calibration set may be reused across predictions; we ask when two different sets yield one classifier. \citet[App.~C]{scholten2025} make a single run deterministic via input-hash jitter; we adopt the device for tie-breaking (Remark~\ref{rem:jitter}) and show it leaves cross-sample identity at zero; post-selection inference conditions on the selection \citep{berk2013,fithian2014,taylor2015}, while Corollary~\ref{cor:attack} makes the selected object nearly deterministic, a complementary route (Table~\ref{tab:imagenet}). Private prediction sets \citep{angelopoulos2022private} privatize the same quantile; the generic DP-to-replicability conversion \citep{bun2023} costs quadratic overhead and a shared seed anyway. Prediction churn and multiplicity are documented deployment problems \citep{fard2016,marx2020,prastalo2026}; replicable calibration yields one verifiable classifier.

\section{Preliminaries: calibration and replication metrics}\label{sec:prelim}

\textbf{Split conformal.} $(X_i,Y_i)\overset{\text{iid}}{\sim}P$ on $\cX\times\cY$; $s:\cX\times\cY\to\R$ is a fixed measurable score (a frozen artifact); $S_i=s(X_i,Y_i)$ has continuous law $F$ (Remark~\ref{rem:jitter}); $q=F^{-1}(1-\alpha)$. With $k=\lceil(1-\alpha)(n+1)\rceil$ and $\alpha\ge1/(n+1)$, the threshold is $\ttau_n=S_{(k)}$, the deployed classifier is $\cC_\tau(\cdot)=\{y:s(\cdot,y)\le\tau\}$ at $\tau=\ttau_n$, and $\Pd(Y_{n+1}\in\cC_{\ttau_n}(X_{n+1}))\ge1-\alpha$. For continuous $F$, $F(\ttau_n)\overset{d}{=}\mathrm{Beta}(k,n{+}1{-}k)$, the exact law of realized coverage \citep{vovk2012}.

\textbf{Replicability.} A calibrator $\cA$ maps $(\cD,r)$ to an output, $r$ a shared random string. $\cA$ is \emph{$\rho$-replicable} over a class $\cP$ if $\Pd_{\cD,\cD'\sim P^{\otimes n},r}[\cA(\cD;r)\neq\cA(\cD';r)]\le\rho$ for all $P\in\cP$ \citep{ilps2022}; it is \emph{$(\ell,\delta)$-list replicable} if for every $P$ there is $\cL(P)$, $|\cL(P)|\le\ell$, with $\Pd[\cA(\cD;r)\in\cL(P)]\ge1-\delta$ \citep{dixon2023}. Both are preserved under deterministic post-processing. For threshold calibrators, identical $\tau$ gives identical set maps at \emph{every} test input.

\begin{definition}[Replication spectrum]\label{def:spectrum}
For two runs producing set maps $\cC,\cC'$ we track four criteria: \emph{(i) identity} $\Pd[\cC\equiv\cC']$; \emph{(ii) pointwise agreement} $\E_{X}\,\Pd[\cC(X)=\cC'(X)]$; \emph{(iii) churn mass} $\E_{X}\,|\cC(X)\,\triangle\,\cC'(X)|$, also reported conditionally on $\cC(X)\neq\cC'(X)$; \emph{(iv) list membership} as above.
\end{definition}
Identity implies (ii) and (iii) at their extreme values and places both outputs in a common singleton list; the criteria are not totally ordered, and, as the experiments show, maximizing (i) can worsen the conditional version of (iii). We report the full spectrum throughout.

\begin{remark}[Continuity by deterministic jitter]\label{rem:jitter}
Replace $s$ by $\widetilde s(x,y)=s(x,y)-u(x)\,g(x,y)$ with fixed measurable $u:\cX\to[0,1)$, $g>0$, used identically by all analysts (adapted from \citealp{scholten2025}): ties break identically and marginal validity is unaffected. Continuity of the induced law holds when $u(X)$ is atomless and independent of $(s(X,Y),g(X,Y))$, and more generally when its conditional law given that pair is atomless; marginal atomlessness alone is insufficient ($s=u=x$, $g\equiv1$ gives $\widetilde s\equiv0$), and a finite-range hash satisfies the hypotheses only approximately (exact conditions, counterexamples, and learned scores: Appendix~\ref{app:upper}).
\end{remark}

\begin{proposition}[Standard split conformal never replicates]\label{prop:vanilla}
If $F$ is continuous, two independent calibrations give $\Pd[\ttau_n=\ttau_n']=0$; the map-equality events (pointwise, or $P_X$-almost everywhere) have probability zero in the completed joint law; and realized coverage fluctuates by the Beta law above. (Full statement, boundary levels, and proof: Appendix~\ref{app:upper}.)
\end{proposition}

\section{Exact replication is impossible and forces obliviousness}\label{sec:impossible}

We first show that exact replication is incompatible with nontrivial validity, which motivates the $(\varepsilon,\rho)$ relaxation and the shared seed. Let $\cY$ be finite, $|\cY|=K\ge2$, and $\cP_{\mathrm{ac}}$ the class of distributions on $\R^{d}\times\cY$ with absolutely continuous $X$-marginal (labels arbitrary). Call a set-valued procedure \emph{uniformly $\varepsilon$-valid} if $\Pd_{\cD,r,(X,Y)\sim P}[Y\in\widehat\cC(X)]\in[1-\alpha-\varepsilon,\,1-\alpha+\varepsilon]$ for every $P\in\cP_{\mathrm{ac}}$; the two-sided requirement excludes the trivial $\widehat\cC\equiv\cY$.

\begin{theorem}[Obliviousness of exact replication]\label{thm:impossible}
\emph{(a)} If $\varepsilon+\delta<\min(\alpha,1-\alpha)$, no $(1,\delta)$-list-replicable procedure is uniformly $\varepsilon$-valid. Every seedless $\rho$-replicable procedure is $(1,\rho)$-list replicable, so no seedless procedure is uniformly $\varepsilon$-valid and $\rho$-replicable when $\varepsilon+\rho<\min(\alpha,1-\alpha)$. \emph{(b)} With a shared seed $R$, any exactly ($0$-)replicable procedure equals a fixed measurable $v_\star(R)$ almost surely under every $P\in\cP_{\mathrm{ac}}$, so its output law is $P$-independent, and if it is also uniformly $\varepsilon$-valid, there is a fixed kernel $p$ with $p(x,y)\in[1-\alpha-\varepsilon,1-\alpha+\varepsilon]$ for a.e.\ $x$ and every $y$, coverage under every $P$ equals $\E_P[p(X,Y)]$, and $(1-\alpha-\varepsilon)K\le\E|\widehat\cC(x)|\le(1-\alpha+\varepsilon)K$ for a.e.\ $x$. (Proof: Appendix~\ref{app:impossible}.)
\end{theorem}

Part (b) shows that an output law that does not vary with the sample cannot depend on the distribution, so its validity is vacuous; the exact endpoint is attained only by data-oblivious procedures (Appendix~\ref{app:impossible}'s example, exactly $(1,\min\{\alpha,1-\alpha\})$-list replicable at $\varepsilon=0$). Determinism does not circumvent (a): sample randomness remains, and the same appendix gives a deterministic, uniformly $0$-valid procedure with atomless output law. Shared randomness is therefore necessary, and the following sections develop the $(\varepsilon,\rho)$ relaxation.

\section{\ReCal: replicable calibration with a shared seed}\label{sec:method}

\begin{figure}[t]
\centering
\fbox{\begin{minipage}{0.965\linewidth}\small
\textbf{Algorithm 1 (\ReCal).} \emph{Shared:} frozen score artifact $s$ (hash-verified; with the Remark~\ref{rem:jitter} jitter), level $\alpha$, width $\beta$ (from $(\varepsilon,\rho)$ via Corollary~\ref{cor:protocol} or Proposition~\ref{prop:pilot}, safety factor $\hat\kappa$), one seed giving $u\sim\mathrm{Unif}[0,\beta)$. \emph{Each analyst, on its own data:}
(1) $\ttau\leftarrow S_{(k)}$, $k=\lceil(1-\alpha)(n+1)\rceil$;\quad
(2) $\tau\leftarrow u+\beta\lceil(\ttau-u)/\beta\rceil$ (round \emph{up} to the shared grid);\quad
(3) deploy $\cC_\tau$.
\emph{Seedless variant:} $u=0$; all analysts then land in a common (unknown to them) two-element list of adjacent classifiers (Theorem~\ref{thm:twolist}).
\end{minipage}}
\vspace{-0.4em}
\end{figure}

Rounding \emph{up} can only enlarge sets, so the classical guarantee survives unconditionally; the analysis controls the probability that two independent thresholds straddle a shared grid point, and the coverage inflation of at most one cell. Both are governed by a local margin condition.

\begin{assumption}[Local margin]\label{ass:margin}
There are $\Delta>0$, $0<f_{\min}\le f_{\max}<\infty$ with $F$ admitting a density $f\in[f_{\min},f_{\max}]$ on $I=[q-\Delta,q+\Delta]$; write $\kappa=f_{\max}/f_{\min}$. (Local only; for probability-scale scores such as APS, $f\approx1$ near the quantile under approximate calibration, the probability integral transform.)
\end{assumption}

\begin{theorem}[Upper bound]\label{thm:upper}
Let $F$ be continuous, let Assumption~\ref{ass:margin} hold, let $k=\lceil(1-\alpha)(n+1)\rceil\le n$, and let $n\ge4/(f_{\min}\Delta)$. Two analysts run \ReCal{} with shared $(\beta,u)$, $u\sim\mathrm{Unif}[0,\beta)$, on independent samples of size $n$. Then:
\textbf{(i)} $\displaystyle\Pd[\tau\neq\tau']\le\frac{B_n}{f_{\min}\beta}+4e^{-nf_{\min}^{2}\Delta^{2}/2}$, where $B_n:=\sqrt{2\alpha(1-\alpha)/n}+\sqrt2/n$;
\textbf{(ii)} for every fixed offset $u$, marginal coverage is $\ge k/(n+1)\ge1-\alpha$, with no margin condition; and for $\delta$ with $e_n(\delta):=\sqrt{\log(2/\delta)/(2n)}+2/n\le f_{\min}(\Delta-\beta)$, with probability $\ge1-\delta$ over the sample and then \emph{simultaneously for every} $u\in[0,\beta)$, $1-\alpha-e_n(\delta)\le F(\tau)\le1-\alpha+e_n(\delta)+f_{\max}\beta$;
\textbf{(iii)} on $\{\tau=\tau'\}$ the two deployed classifiers (same score) coincide at every test input. (Proof: Appendix~\ref{app:upper}.)
\end{theorem}
Neither (i) nor the marginal claim uses $\beta\le\Delta/2$; only the conditional band does, through $\Delta-\beta$. The proof is exact finite-sample distribution theory: (i)--(iii) are read off the $\mathrm{Beta}(k,n{+}1{-}k)$ law of $F(\ttau_n)$ instead of uniform-deviation bounds, the source of the constants in Remark~\ref{rem:constants}.

\begin{corollary}[The $(\varepsilon,\rho)$ protocol]\label{cor:protocol}
Given $(\varepsilon,\rho,\delta)$ with $\varepsilon\le f_{\max}\Delta$, set $\beta=\varepsilon/(2f_{\max})$ (so $\beta\le\Delta/2$). With the explicit finite-sample remainder $n_0$ of \eqref{eq:protocol-nzero} in Appendix~\ref{app:upper} (for fixed $\alpha,\delta$ and margin constants, $n_0=o((\varepsilon\rho)^{-2})$),
$n\ge\max\{32\kappa^{2}\alpha(1-\alpha)/(\varepsilon^{2}\rho^{2}),\,8\log(2/\delta)/\varepsilon^{2},\,n_0\}$
makes \ReCal{} $\rho$-replicable with marginal coverage $\ge1-\alpha$ and, with probability $1-\delta$ over the calibration sample and for every realized shared offset, $F(\tau)\in[1-\alpha-\varepsilon/2,\,1-\alpha+\varepsilon]$.
\end{corollary}

\begin{corollary}[Unequal calibration sizes]\label{cor:twolabs}
Analysts sharing $(s,\alpha,\beta,u)$ with sizes $n_A\neq n_B$, where $\alpha\ge\max\{1/(n_A{+}1),1/(n_B{+}1)\}$ and $n_A,n_B\ge4/(f_{\min}\Delta)$, satisfy, for \emph{every} $\beta>0$,
$\Pd[\tau_A\neq\tau_B]\le\big[\sqrt{\alpha(1-\alpha)}(n_A^{-1/2}+n_B^{-1/2})+3/(n_A\wedge n_B)\big]/(f_{\min}\beta)+4e^{-(n_A\wedge n_B)f_{\min}^{2}\Delta^{2}/2}$. Each analyst separately retains marginal coverage $\ge1-\alpha$, and per-analyst conditional bands hold as in Theorem~\ref{thm:upper}(ii) with a union bound (Appendix~\ref{app:upper}).
\end{corollary}
The bound permits $n_A\neq n_B$: analysts need not coordinate sample sizes.

\begin{corollary}[Bound on selective recalibration]\label{cor:attack}
Suppose an offset $u\sim\mathrm{Unif}[0,\beta)$ is drawn independently of all data, fixed, and used in every run. An adversary draws $\cD_1,\dots,\cD_M$ i.i.d.\ and deploys $\cA(\cD_{J};u)$ for an arbitrary data-dependent selection $J$. If $\cA$ is $\rho$-replicable ex ante, i.e.\ $\E_u\Pd_{\cD,\cD'}[\cA(\cD;u)\ne\cA(\cD';u)]\le\rho$, then
$\Pd[\exists j\le M:\cA(\cD_j;u)\neq\cA(\cD_1;u)]\le c_M:=\min\{1,(M-1)\rho\}$,
so with probability at least $1-c_M$ every selection deploys the same classifier; and if each unselected run is marginally valid (true for \ReCal{} with $k\le n$), then, writing $\mathrm{cov}_P(\cC)=\Pd_{(X,Y)\sim P}(Y\in\cC(X))$, every selection rule retains $\E\,\mathrm{cov}_P(\cA(\cD_J;u))\ge1-\alpha-c_M$. Against standard split conformal the selected coverage drops by approximately $a_M\sigma_n$, where $\sigma_n$ is the standard deviation of the $\mathrm{Beta}(k,n{+}1{-}k)$ coverage law and $a_M=\E\max_{j\le M}Z_j$ for i.i.d.\ standard normal $Z_j$ ($a_M\sim\sqrt{2\log M}$ only as $M\to\infty$); the exact finite-$M$ formula, the $\mathrm{Beta}(1,M)$ law of the selected coverage quantile, and its total-variation distance $\tfrac{M-1}{M}M^{-1/(M-1)}$ from the honest law are in Appendix~\ref{app:upper}. (Measurement: Section~\ref{sec:exp}.)
\end{corollary}
Candidate-wise marginal validity does not by itself survive adaptive selection; the display above is what replicability adds. A support-only audit of the deployed threshold has zero power against this selection, while the positive total-variation gap permits tests whose power exceeds their size.

Imperfect sharing degrades gracefully: with $\|s_A{-}s_B\|_\infty\le\eta_s$ and population total variation $\eta_P$, the mismatch bound gains terms linear in $2\eta_s{+}\eta_P$, motivating the hash-verified score artifact (Corollary~\ref{cor:robust}, Appendix~\ref{app:upper}).

\begin{proposition}[Pilot rule with a confidence band]\label{prop:pilot}
Split a public pilot of $n_p=2m$ scores (disjoint from all calibration data): $\hat q_p$ is the conformal quantile of the first half, $N_W$ counts second-half scores in $W=[\hat q_p-h,\hat q_p+h]$ with $h\le\Delta/2$, and $\hat f_{\mathrm L}:=\max\{N_W,1\}/\big((1+c)\,2hm\big)$ for an accuracy $c\in(0,1)$. If $e_m(\delta_p/3)\le f_{\min}\Delta/2$ and $m\ge3\log(3/\delta_p)/(2c^{2}hf_{\min})$, then on a pilot event $\cG_p$ of probability $\ge1-\delta_p$: $W\subseteq I$ and $\tfrac{1-c}{1+c}f_{\min}\le\hat f_{\mathrm L}\le f_{\max}$. Set $\beta:=2\hat\kappa B_n/(\hat f_{\mathrm L}\rho)$, $B_n$ as in Theorem~\ref{thm:upper}, with shared offset $u=\beta V$, $V\sim\mathrm{Unif}[0,1)$. For every pilot realization in $\cG_p$, under Appendix~\ref{app:upper}'s explicit side conditions ($\overline\beta_n\le\Delta/2$, tail $\le\rho/2$), two analysts satisfy $\Pd[\tau\ne\tau'\mid\cD_p]\le\rho$; unconditionally $\Pd[\tau\ne\tau']\le\rho+\delta_p$, made exactly $\rho_\star$ by running the construction at $\rho_0=(\rho_\star-\delta_p)/(1-\delta_p)$ (for $0<\delta_p<\rho_\star$). Marginal coverage $\ge1-\alpha$ holds for \emph{every} pilot and offset realization, and the conditional band holds with rounding excess at most $\tfrac{2(1+c)}{1-c}\hat\kappa\kappa B_n/\rho$. A valid $\hat\kappa\ge\kappa$ remains the one input that cannot be certified from data, and $\hat\kappa=1.5$ was consistent with the target in every configuration of the safety study (Section~\ref{sec:exp}). (Full statement and proof: Appendix~\ref{app:upper}.)
\end{proposition}

\begin{remark}[Magnitude of the constants]\label{rem:constants}
At $(\alpha,\varepsilon,\rho,\delta)=(0.1,0.02,0.1,0.05)$, $\kappa=1$, the calibration size required by three analyses of the \emph{same} rounding scheme: exact-Beta (Corollary~\ref{cor:protocol}) $7.2\times10^{5}$; a Dvoretzky--Kiefer--Wolfowitz (DKW)/union analysis ($32\log(8/\rho)/(\varepsilon^{2}\rho^{2})$) $3.5\times10^{7}$; SQ binary search \citep[after][]{kalavasis2024} $\gtrsim\log_2^{2}(2^{32})/(\varepsilon^{2}\rho^{2})\approx2.6\times10^{8}$. These factors of $50$--$350$ quantify the benefit of Assumption~\ref{ass:margin}; the empirical anchor (Appendix~\ref{app:constants}) places the certified constants within about $5\times$ of the measured frontier.
\end{remark}

\begin{proposition}[The plug-in is asymptotically exact, without $\hat\kappa$]\label{prop:adaptive}
Fix $\alpha\in(0,1)$, $\rho\in(0,1]$, and a score law with density $f$ continuous and positive at $q=F^{-1}(1-\alpha)$; let $\hat f_n$ be any shared pilot with $\hat f_n\to_pf(q)$, $\hat f_n\ge f_0>0$, independent of both calibration samples, and run \ReCal{} with $\beta_n=2B_n/(\hat f_n\rho)$, $\hat\kappa{=}1$. Then, with $Z$ standard normal,
(i) $\Pd[\tau\neq\tau']\to\E[\min(\rho|Z|/2,1)]\le\rho\sqrt{2/\pi}/2<2\rho/5$;
(ii) $\E F(\tau)-\tfrac{k}{n+1}=(1+o(1))\,B_n/\rho$; and
(iii) $\Pd[F(\tau)\le\tfrac{k}{n+1}+e_n(\delta)+(2+o(1))B_n/\rho]\ge1-\delta-o(1)$.
The width, hence the premium, carries no $\kappa$: the factor $\kappa^{2}$ in Corollary~\ref{cor:protocol} is a finite-sample certificate, not an asymptotic necessity, and (ii)--(iii) meet Corollary~\ref{cor:sizelb} within an absolute factor; a $\sqrt{2/\pi}$ width makes (i) exactly $\rho$. (Proof: Appendix~\ref{app:upper}; empirical check: Appendix~\ref{app:exp}.)
\end{proposition}

\textbf{Extensions} (\ReCal$^{+}$, distribution-free fallback, CRC, conformal-LM): Remark~\ref{rem:extensions}, Appendix~\ref{app:distfree}.

\section{Lower bounds: samples and set size}\label{sec:lower}

\begin{theorem}[Lower bound]\label{thm:lower}
Fix an integer $n\ge1$ and $\alpha\in(0,1)$, $0<\varepsilon\le\min(\alpha,1-\alpha)/4$, $0<\rho\le1$, $0\le\delta\le\tfrac1{16}$. Let $\cA:[0,1]^n\times\mathsf R\to\R$ be jointly measurable, with a shared seed $r\sim\nu$ independent of the data. If, for \emph{every} atomless Borel law $P$ on $[0,1]$ with CDF $F_P$, $\cA$ is $\rho$-replicable and $\Pd[F_P(\cA(\cD;r))\in[1-\alpha\pm\varepsilon]]\ge1-\delta$, then
$n\;\ge\;\tfrac{9(1-8\delta)^{2}}{4096}\,\tfrac{\alpha(1-\alpha)}{\varepsilon^{2}\rho^{2}}\;\ge\;\tfrac{9}{16384}\,\tfrac{\alpha(1-\alpha)}{\varepsilon^{2}\rho^{2}}$.
The proof is self-contained, incurs no logarithmic loss, and uses accuracy only at the two endpoints of the hard family, each satisfying Assumption~\ref{ass:margin} with $\kappa=1$, $\Delta=\varepsilon$, constant density in $[2\min(\alpha,1-\alpha),2]$; the whole family satisfies it with $\Delta=\tfrac14$, $f_{\min}=\min(\alpha,1-\alpha)$, $f_{\max}=2$. The bound concerns exact replication of the scalar threshold; Appendix~\ref{app:lower} states the exact scope (set-map-only replicability is not automatically covered). (Proof: Appendix~\ref{app:lower}.)
\end{theorem}

The proof plants a coin: under $P_m$ with density $2m$ on $[0,\tfrac12)$ and $2(1-m)$ on $[\tfrac12,1]$, $m\in[1-\alpha\pm2\varepsilon]$, two-sided accuracy forces $\mathrm{sign}(\tau-\tfrac12)$ to reveal the endpoint; $N=\#\{S_i<\tfrac12\}\sim\mathrm{Bin}(n,m)$ is sufficient with an ancillary complement, so the replicable coin-problem bounds \citep{ilps2022,hopkins2024} transfer with gap $4\varepsilon$ at bias $\approx1-\alpha$, whose binomial variance contributes the $\alpha(1-\alpha)$ factor, matching Corollary~\ref{cor:protocol} up to $\kappa^{2}$ and logarithms. Scope: the threshold restriction is necessary, since a fixed set can meet the band on this family without learning anything (Appendix~\ref{app:lower}), though over all of $\cP_{\mathrm{ac}}$ even that is impossible (Theorem~\ref{thm:impossible}), and nested families are exactly the conformal output class; because interior members have vanishing margin width, the bound constrains procedures replicable over \emph{all} continuous distributions; whether procedures replicable only over a fixed-margin class can be faster, that is, whether the factor $\kappa^{2}$ is necessary at finite $n$ is open; asymptotically it is not (Proposition~\ref{prop:adaptive}, Section~\ref{sec:disc}).

\begin{corollary}[Set-size cost of replicability]\label{cor:sizelb}
In the setting of Theorem~\ref{thm:lower}, let $0\le\delta\le\tfrac1{32}$ and $0<e\le\min(\alpha,1-\alpha)/4$. If $\cA$ is $\rho$-replicable and, for every atomless $P$, satisfies the validity side $\Pd[F_P(\cA(\cD;r))\ge1-\alpha]\ge1-\delta$ and the inflation cap $\Pd[F_P(\cA(\cD;r))\le1-\alpha+e]\ge1-\delta$, then
$e\;\ge\;\tfrac{3(1-16\delta)}{32}\,\sqrt{\alpha'(1-\alpha')/n}\,/\rho\;\ge\;\tfrac{3}{70}\,\sqrt{\alpha(1-\alpha)/n}\,/\rho,\qquad \alpha':=\alpha-e/2.$
(Proof: Appendix~\ref{app:lower}.)
\end{corollary}
The bound is necessarily worst-case: for one fixed $P$ the oracle $\tau\equiv F_P^{-1}(1-\alpha)$ is $0$-replicable, exactly valid, with zero inflation. \ReCal{} attains the frontier up to $\kappa^{2}$: Theorem~\ref{thm:upper}(ii) with the protocol width caps inflation at $e_n(\delta)+f_{\max}\beta=O(\kappa^{2}\sqrt{\alpha(1-\alpha)/n}/\rho)$. The expectation-only analogue is open.

\begin{lemma}[Churn localization]\label{lem:churn}
For any two runs of a shared-grid threshold rule with outputs $\tau,\tau'\in u+\beta\mathbb{Z}$ and any input $x$,
$|\cC_\tau(x)\,\triangle\,\cC_{\tau'}(x)|=\#\{y:\ \tau\wedge\tau'<s(x,y)\le\tau\vee\tau'\}$, and on $\{\tau\neq\tau'\}$ the interval is a disjoint union of full grid cells, so the churn is at least the label mass of one full cell. Consequently $\E_X\big[|\cC_\tau(X)\triangle\cC_{\tau'}(X)|\,\big|\,\tau\neq\tau'\big]\ge\min_{C}m_\beta(C)$, the minimum over grid cells $C$ meeting the localization window of Theorem~\ref{thm:upper}(ii) of $m_\beta(C):=\E_X\#\{y:s(X,y)\in C\}$. (Proof: Appendix~\ref{app:lower}.)
\end{lemma}
Together they explain F6 and F8: the unavoidable inflation is paid in score mass, and a grid converts it into $m_\beta$ labels per cell, in set size always and in churn on the rare mismatches; the concentration measured in Section~\ref{sec:exp} is this geometry, not an artifact of \ReCal.

\section{The seedless setting: an optimal two-element list}\label{sec:twolist}

\begin{theorem}[Seedless $2$-list, optimally]\label{thm:twolist}
Fix public constants $\Delta>0$, $0<f_{\min}\le f_{\max}<\infty$; let $\mathfrak F$ be the nonempty class of continuous score distributions satisfying Assumption~\ref{ass:margin} with these constants, $\kappa=f_{\max}/f_{\min}$, $\beta=\varepsilon/(2f_{\max})\le\Delta/2$, and $0<\varepsilon<\min(\alpha,1-\alpha)$. \emph{(a)} For $F\in\mathfrak F$ and $n\ge\lceil32\kappa^{2}\log(2/\delta)/\varepsilon^{2}\rceil$, deterministic-grid \ReCal{} ($u=0$) satisfies, with probability $\ge1-\delta$ over the calibration sample, the \emph{joint} event that $\tau$ lies in a sample-independent two-element list $\cL_{n,\delta}(F)=\{g,g+\beta\}\subset\beta\mathbb Z$, depending only on $F$ and the public parameters and hence unknown to the analysts, \emph{and} $F(\tau)\in[1-\alpha-\tfrac{\varepsilon}{4\kappa},\,1-\alpha+\tfrac\varepsilon2+\tfrac{\varepsilon}{4\kappa}]\subseteq[1-\alpha\pm\varepsilon]$; any $m$ analysts all land in the common list with probability $\ge1-m\delta$, marginal validity is preserved, and no $\rho^{-2}$ factor appears. \emph{(b)} For $\delta<\tfrac12$, list size two is optimal within $\mathfrak F$: no procedure (any sample size, any seed law) has both a deterministic singleton output $v_F$ and the two-sided band, each with probability $\ge1-\delta$, for every $F\in\mathfrak F$; this is already impossible on $\{\mathrm{Unif}[t,t+f_{\min}^{-1}]\}\subseteq\mathfrak F$, whose membership needs only $f_{\min}\Delta\le\min(\alpha,1-\alpha)$, forced by $\mathfrak F\neq\varnothing$ (Lemma~\ref{lem:translate}). The optimality claim is about numerical thresholds; it transfers to classifier maps whenever $t\mapsto\cC_t$ is injective on the admissible range (Appendix~\ref{app:twolist}). (Proof: Appendix~\ref{app:twolist}.)
\end{theorem}

This quantifies, in a deployed-statistics setting, the general separations between list replicability and replicability \citep{dixon2023,chase2023}: one replicable answer costs a seed \emph{and} $1/\rho^{2}$; two answers cost neither. For a regulator who cannot coordinate randomness, ``every analyst's classifier is one of two adjacent thresholds, $\beta$ apart'' is a checkable property (F7, Section~\ref{sec:exp}).

\begin{proposition}[Full conformal is not rescuable by rounding]\label{prop:fullconf}
Consider full conformal prediction with the $1$-nearest-neighbor score, calibration size $n\ge2$, and data law $P$ absolutely continuous with respect to Lebesgue measure on $\R^{2}$. For a rank cutoff $m$, let $\widehat\cC^{(m)}_\cD(x)$ be the set of labels whose full-conformal rank among the $n{+}1$ scores is at least $m$ (Appendix~\ref{app:twolist} fixes the tie convention). For every nonendpoint cutoff $m\in\{1,\dots,n\}$, two independent runs satisfy $\Pd[\widehat\cC^{(m)}_\cD\equiv\widehat\cC^{(m)}_{\cD'}]=0$ (map equality read in the completed law); with any shared seed $U$ and measurable cutoff $M(U)$, the agreement probability equals $\Pd[M(U)\in\{0,n{+}1\}]$ exactly. The usual rule $\{\widehat p_\cD>\alpha\}$, where $\widehat p_\cD$ is the full-conformal $p$-value, is the cutoff $\lfloor\alpha(n{+}1)\rfloor$, so agreement is zero for every $\alpha\in[1/(n{+}1),1)$, with no integrality condition, and trivially one below; $n\ge2$ is sharp. This covers exactly the sample-independent, seed-dependent monotone transformations of the $p$-value grid followed by a strict threshold; within this class, stabilizing the refitted training map is the remaining route \citep{ilps2022,eaton2026}. (Full statement, conventions, and proof: Appendix~\ref{app:twolist}.)
\end{proposition}

\section{Experiments}\label{sec:exp}

\textbf{Experimental tracks.} \emph{(1) Image-classification track:} an operating-point-matched synthetic testbed at ImageNet dimensions ($K{=}1{,}000$, pool $50$k, top-1 $79\%$), $n{=}25{,}000$, with APS and RAPS scores \citep{romano2020,angelopoulos2021sets}; \emph{(1b) real ImageNet:} ResNet-50 cross-validated probabilities on the ILSVRC-2012 validation set ($N{=}50{,}000$, $K{=}1{,}000$, top-1 $72.7\%$) \citep{northcutt2021}, via $\log p$ logits, same interface (Appendix~\ref{app:exp}). \emph{(2)} a small-$n$ real track (scikit-learn digits, $n{=}800$; Appendix~\ref{app:exp}). \emph{(3) LLM track (real):} next-token logits from four locally run model families (GPT-2, Pythia-1.4B, Qwen2.5-1.5B, Llama-3.2-1B; $K$ from $50{,}257$ to $151{,}936$) on WikiText-103 ($N{=}10{,}000$; $5{,}000$ for Llama), $n{=}5{,}000$: the answer-set/top-$p$ regime of \citet{quach2024}, real probabilities. \emph{(4) Clinical site-split track:} the four UCI Heart Disease hospitals (Cleveland, Budapest, Long Beach VA, Zurich; $123$--$303$ patients per site) \citep{detrano1989}, a frozen Cleveland-trained score calibrated across sites (Appendix~\ref{app:exp}). The testbeds (Appendix~\ref{app:exp}) make population coverage exact and the frontier measurable; \texttt{--logits} swaps in real image logits unchanged; Appendix~\ref{app:exp} lists their limitations. Protocol: shared deterministic jitter (Remark~\ref{rem:jitter}); public pilot for $\hat f$ (image $0.84$; digits $1.31$, alternate build $1.38$; four LLM families $0.64$--$1.10$; Pythia sweep $0.96$--$1.15$; Appendix~\ref{app:exp}), so Assumption~\ref{ass:margin} is mild here, as expected for probability-scale scores; the implemented hash jitter meets Remark~\ref{rem:jitter}'s hypotheses only approximately, the residual $.01$--$.05$ baseline identities being the finite-pool ties (Appendix~\ref{app:exp}); $\hat\kappa{=}1$ unless stated, with $\hat\kappa{=}1.5$ the recommended default (both reported; Table~\ref{tab:safety}, Appendix~\ref{app:exp}), the pilot's plain point estimate used without Proposition~\ref{prop:pilot}'s factor-$2$ margin ($\hat\kappa$ absorbs both slacks); $\pm$ denotes the standard deviation across pairs or trials; identity rates carry binomial standard errors (Appendix~\ref{app:exp}).

\begin{table}[t]
\centering\footnotesize
\caption{\textbf{Image track} ($n{=}25{,}000$; $100$ pairs; sizes/agreement on $10{,}000$ eval points, first $25$ pairs). ``Ident.'': both analysts deploy the identical classifier; ``Agree'': fraction of test points with equal sets. Baseline identity is $0$ by Proposition~\ref{prop:vanilla} (the $0.00$ entries instantiate that result; the remaining columns are the empirical content).}
\label{tab:imagenet}
\begin{tabular}{lcccc}
\toprule
Method & Coverage & Avg.\ size & Ident. & Agree \\
\midrule
APS (fresh randomization) & $.9003\pm.0020$ & $15.0$ & $.00$ & $.49$ \\
APS (derandomized jitter) & $.9003\pm.0020$ & $15.0$ & $.00$ & $.61$ \\
RAPS (derandomized) & $.9001\pm.0020$ & $6.9$ & $.00$ & $.87$ \\
\ReCal-APS $\rho{=}.1$ & $.9127\pm.0066$ & $18.0$ & $.95$ & $.97$ \\
\ReCal-APS $\rho{=}.01$ & $.9631\pm.0238$ & $56.1$ & $1.00$ & $1.00$ \\
\ReCal-RAPS $\rho{=}.1$ & $.9156\pm.0080$ & $8.2$ & $.96$ & $.88$ \\
Det.-grid $2$-list ($\beta$ as $\rho{=}.1$) & $.9177\pm.0000$ & $19.1$ & $1.00$ & $1.00$ \\
\bottomrule
\end{tabular}
\end{table}

\begin{table}[t]
\centering\small
\caption{\textbf{Real next-token calibration, four model families} (WikiText-103, $n{=}5{,}000$, $100$ pairs; cells show vocabulary and top-1). Identity is at or within one standard error of its target on every model (Llama $.89$ at target $.90$); at $\rho{=}.05$ the identities are $.95/.96/.97/.96$. Qwen's size increase is vocabulary-tail geometry, not a coverage failure (coverage inflation within the $f\beta$ bound; F8).}
\label{tab:realllm}
\begin{tabular}{llcccc}
\toprule
Model & Method & Coverage & Avg.\ size & Ident. & Agree \\
\midrule
\multirow{2}{*}{GPT-2 ($50{,}257$; $.40$)}
 & APS (derand.) & $.9001\pm.0041$ & $353$ & $.02$ & $.24$ \\
 & \ReCal{} $\rho{=}.1$ & $.9270\pm.0156$ & $546$ & $.90$ & $.86$ \\
\midrule
\multirow{2}{*}{Pythia-1.4B ($50{,}304$; $.50$)}
 & APS (derand.) & $.8998\pm.0048$ & $108.5$ & $.00$ & $.40$ \\
 & \ReCal{} $\rho{=}.1$ & $.9275\pm.0154$ & $171.6$ & $.93$ & $.97$ \\
\midrule
\multirow{2}{*}{Qwen2.5-1.5B ($151{,}936$; $.52$)}
 & APS (derand.) & $.9006\pm.0039$ & $96.8$ & $.00$ & $.47$ \\
 & \ReCal{} $\rho{=}.1$ & $.9541\pm.0267$ & $1165$ & $.96$ & $.92$ \\
\midrule
\multirow{2}{*}{Llama-3.2-1B ($128{,}256$; $.52$)}
 & APS (derand.) & $.9000\pm.0044$ & $83.7$ & $.04$ & $.55$ \\
 & \ReCal{} $\rho{=}.1$ & $.9297\pm.0183$ & $132.8$ & $.89$ & $.93$ \\
\bottomrule
\end{tabular}
\end{table}

\paragraph{Findings.}
\textbf{(F1) Baselines do not replicate, and the magnitude of disagreement is large.} The $0.00$ identity entries are Proposition~\ref{prop:vanilla}. On real ImageNet two calibrations agree on $65\%$ of points ($87\%$ derandomized) with identity $0.00$; on the four real language models two honest analysts' answer sets agree on only $24$--$55\%$ of contexts and differ by $5$--$20$ tokens per context (Tables~\ref{tab:realllm}--\ref{tab:small}; synthetic long-tail in Appendix~\ref{app:exp}); derandomization \citep{scholten2025} raises agreement ($.49\to.61$ image, $.20\to.24$ GPT-2) with identity unchanged. Small nonzero baseline identity rates ($.01$--$.05$) are finite-pool tie artifacts (Appendix~\ref{app:exp}).
\textbf{(F2) \ReCal{} attains its replicability targets; small-sample plug-in error requires a safety factor.} Synthetic image track: identity $.95$ at $\rho{=}.1$, $1.00$ at $\rho{=}.01$ over $100$ pairs ($0/100$ certifies only $\le.03$ by the rule of three); achieved non-replication tracks the target from below across the $\rho$ sweep (Figure~\ref{fig:rho}a, Appendix~\ref{app:exp}); the qualification is small-$n$ plug-in error. On the real language models the $\hat\kappa{=}1$ plug-in achieved $.060$--$.103$ against the $.10$ target, at or below within one standard error (Table~\ref{tab:safety}, Appendix~\ref{app:exp}); on real ImageNet, $.87$ at $\hat\kappa{=}1$ (within one standard error of target) and $.98$ at $\hat\kappa{=}1.5$. Pre-registering $\hat\kappa{=}1.5$ brings every configuration to target, up to one statistical tie (Pythia $.103$; digits and RAPS in Table~\ref{tab:safety}), at Proposition~\ref{prop:pilot}'s cost; it is the recommended default.
\textbf{(F3) Validity is preserved, and the set-size cost is quantified.} Every \ReCal{} calibration in this paper attained realized coverage $\ge1-\alpha$, concentrating on grid atoms as predicted (all tracks and sweeps; Figure~\ref{fig:beta}, Appendix~\ref{app:exp}). At $n{=}25{,}000$ the cost is $+19\%$ set size at $\rho{=}.1$ on the synthetic track and $+17\%$ on real ImageNet (Table~\ref{tab:imagenet}, Appendix~\ref{app:exp}); the appendix cost curve gives $+17/+11/+91\%$ at $\rho{=}.1/.2/.02$ (baseline differs slightly; Figure~\ref{fig:cost}, Appendix~\ref{app:exp}); coverage inflation is $.013$ at $\rho{=}.1$ (bound $.023$). On the real language models at $n{=}5{,}000$ the $\rho{=}.1$ cost is $55$--$59\%$ (GPT-2, Pythia, Llama); Qwen's $+1104\%$ is F8's tail-geometry effect, not a coverage effect. At $n{=}800$ the same targets cost $+238\%$ at $\rho{=}.1$ ($+4\%$ at $\rho{=}.2$): Theorem~\ref{thm:lower} at work; small $n$ forces larger $\rho$ or larger sets. On the four-hospital track ($n{\approx}120$--$200$), cross-site identity $.84$--$.96$ at $\rho{\in}\{.3,.5\}$ costs near-vacuous sets ($1.7$--$1.9$ of $2$) at coverage $.97$--$.99$ under real shift: the small-$n$ wall on real sites (Appendix~\ref{app:exp}). By Corollary~\ref{cor:sizelb}, these premiums are, up to $\kappa^{2}$, a price any replicable threshold calibrator must pay.
\textbf{(F4) The measured sample-complexity frontier is consistent with the predicted rate.} With $\beta$ set by $\varepsilon$ alone, measured mismatch decays approximately as $n^{-1/2}$ (fitted exponents $-.49,-.50,-.43,-.49$), and the crossing $n^\ast(\varepsilon,\rho)$ of each fitted mismatch curve at level $\rho$ fits slope $2.09$ (bootstrap CI $[1.99,2.28]$) in $1/(\varepsilon\rho)$ (Figure~\ref{fig:nstar}, Appendix~\ref{app:exp}); an $\alpha$ sweep at fixed width tracks $B_n(\alpha)/(\hat f\beta)$ at proportionality $.97$, $R^{2}{=}.99$ (Appendix~\ref{app:exp}); $\kappa^{2}$ is not separately tested. The constants: fitted $n^\ast{\approx}3.6{\times}10^{4}$ at $(\varepsilon,\rho)=(.04,.1)$ against the conservative $1.8\times10^{5}$ of Corollary~\ref{cor:protocol}.
\textbf{(F5) The selection attack is demonstrated and bounded.} Selecting the lowest-coverage of $M{=}20$ standard-split recalibrations (full-pool coverage, an oracle stress test) yields coverage $.8920\pm.0024$ over $300$ trials (honest $.9003$), below nominal and invisible to a support check. Against \ReCal{} at $\rho{=}.1$ the selected coverage is $.9170$, above $1-\alpha$: each candidate is an upward-rounded, marginally valid threshold, so selection only reduced rounding inflation; in general Corollary~\ref{cor:attack} bounds selected coverage by $1-\alpha-c_M$. At $\rho{=}.01$ the identity bound is active: two or more distinct classifiers appeared in $1.3\%$ of trials (bound $(M{-}1)\rho{=}.19$; achieved mismatch $0$), and selected coverage equals honest coverage to three decimals. At $\rho{=}.1$, $M{=}20$ the union bound exceeds one ($31\%$ of trials had two or more distinct outputs) and the protection is guarantee preservation; both regimes are reported (Figure~\ref{fig:hero}b).
\textbf{(F6) Identity concentrates disagreement into rare, large events.} \ReCal{} converts frequent small disagreements into rare one-cell disagreements, exactly Lemma~\ref{lem:churn}'s geometry (Table~\ref{tab:spectrum}, Appendix~\ref{app:exp}). Where baseline churn already exceeds a cell, \ReCal{} improves every metric (GPT-2: $20.1\to0.0$ marginal tokens, $60$ mismatch-free pairs); where it is smaller than a cell, identity has a real cost on the continuous metrics (Llama: $5.0\to15.1$ tokens for identity $0\to.90$; the small-$n$ digits track shows the same pattern, Appendix~\ref{app:exp}). Auditability, caching, or selection robustness justify the cost; others can use $\rho{=}.2$ ($+4\%$), and by Lemma~\ref{lem:churn} no grid rule shrinks conditional churn below one cell's mass.
\textbf{(F7) Seedless outputs form two-element lists.} All deterministic-grid outputs fall on two adjacent grid values (empirical mass $1.000$, all tracks including real ImageNet); within-pair identity, not guaranteed, was $1.00$ (image), $.63$ (digits); the adjacency check of Section~\ref{sec:twolist} held in every pair.
\textbf{(F8) On real models the set-size cost depends on vocabulary tail geometry.} The coverage cost is uniform and small (inflation $.027$--$.054$, below the bound $\approx.06$; Table~\ref{tab:realllm}), but the set-size translation of one probability cell depends on how many tokens lie above the quantile: Qwen2.5's $152$k tail turns a $.094$ cell into $+1{,}068$ tokens; GPT-2's $.058$ cell, $+193$. Rank regularization removes it: \ReCal-RAPS on Qwen reaches identity $.92$ below standard-APS size ($53$ tokens, coverage $.921$). Sizes above are at $\rho{=}.1$; tighter targets grow the effect (Table~\ref{tab:small}; cell mass: Lemma~\ref{lem:churn}). On the $50$k-vocabulary models the same density variation makes $\hat\kappa{=}1$ an underestimate (identity $.72/.76$, coverage $\ge.90$); for long-tail vocabularies use tail regularization and $\hat\kappa>1$.

\section{Discussion, limitations, and future work}\label{sec:disc}

Replicable calibration makes cross-analyst agreement checkable on the deployed artifact: one classifier or a certified pair, with matched sample and set-size costs and a selection bound.

Three limitations point to future work. First, the real-image run uses cross-validated probability outputs \citep{northcutt2021}, not one frozen network's raw logits, and the clinical sites are small ($n{\le}303$); frozen production-scale deployment remains future work. Second, Proposition~\ref{prop:adaptive} removes $\hat\kappa$ and $\kappa^{2}$ asymptotically under continuity at the quantile; finite-$n$ uniform removal is open. Third, Proposition~\ref{prop:fullconf} covers cutoff stabilization only; training-map and selection replicability remain open.

\bibliographystyle{arxivstyle}
\bibliography{references}

\clearpage
\appendix
\begin{center}{\LARGE\bf Appendix}\end{center}

\section{Proofs for Section~\ref{sec:impossible}}\label{app:impossible}

\paragraph{Measurable formulation.}
Fix integers $d,n\geq1$, a target miscoverage level $\alpha\in(0,1)$,
and a finite label set $\cY$ with $|\cY|=K\geq2$.  Equip $\cY$ and
$2^{\cY}$ with their discrete sigma-fields.  Let
$\cP_{\mathrm{ac}}$ be the class of all Borel probability laws on
$\mathbb R^d\times\cY$ whose $X$-marginal is absolutely continuous with
respect to $d$-dimensional Lebesgue measure.  Let
$(\mathsf R,\mathcal R)$ and $(\mathsf V,\mathcal V)$ be standard
Borel spaces, let $\nu$ be a fixed probability law on $\mathsf R$, and
let $R\sim\nu$ be independent of the calibration sample and test pair.
A procedure is a jointly measurable map
\[
  \cA:(\mathbb R^d\times\cY)^n\times\mathsf R
  \longrightarrow\mathsf V.
\]
Each $v\in\mathsf V$ deterministically encodes a measurable deployed
set map $\cC_v:\mathbb R^d\to2^{\cY}$, and
\[
  \Gamma(v,x,y):=\one\{y\in\cC_v(x)\}
\]
is jointly measurable.  All randomness affecting the deployed map is
included in $R$ before $v$ is output; no fresh randomness is drawn when
$\cC_v(x)$ is evaluated.  Equality, list membership, and replication
below concern the encoded output in $\mathsf V$.  If replication is
instead intended to mean equality of deployed maps, assume explicitly
that $v\mapsto\cC_v$ is injective on the admissible output family.

For $\cD\sim P^{\otimes n}$, $R\sim\nu$, and an independent test pair
$(X,Y)\sim P$, define the marginal coverage
\[
  \operatorname{Cov}_P(\cA)
  :=\Pd\!\left\{Y\in\cC_{\cA(\cD;R)}(X)\right\},
\]
where the probability averages over $\cD$, $R$, and the independent
test pair.  The procedure is \emph{uniformly two-sided
$\varepsilon$-valid} (the main text's uniform $\varepsilon$-validity) if
\[
  \operatorname{Cov}_P(\cA)
  \in[1-\alpha-\varepsilon,\,1-\alpha+\varepsilon]
  \qquad\text{for every }P\in\cP_{\mathrm{ac}}.
\]

For $\delta\in[0,1]$, the procedure is
\emph{$(1,\delta)$-list replicable} over $\cP_{\mathrm{ac}}$ if, for
every $P\in\cP_{\mathrm{ac}}$, there is a deterministic set
$L_P\subseteq\mathsf V$, depending on $P$ but not on the realized
sample or seed, such that
\[
  |L_P|\leq1
  \quad\text{and}\quad
  (P^{\otimes n}\otimes\nu)
  \{\cA(\cD;R)\in L_P\}\geq1-\delta.
\]
For probability laws $P,Q$ on a common measurable space
$(\Omega,\mathcal F)$, we use the convention
$\dtv(P,Q):=\sup_{A\in\mathcal F}|P(A)-Q(A)|$.
When two runs have no shared randomness, their complete seeds are
independent copies $R,R'\sim\nu$.  Thus seedless
$\rho$-replicability means
\[
  \Pd\!\left\{
    \cA(\cD;R)\neq\cA(\cD';R')
  \right\}\leq\rho
  \qquad\text{for every }P\in\cP_{\mathrm{ac}},
\]
where $\cD,\cD'\overset{\mathrm{iid}}{\sim}P^{\otimes n}$ and
$\cD,\cD',R,R'$ are mutually independent.  Shared-seed replication
uses the same $R$ in both runs.  Finally, put
$\ell_\varepsilon:=\max\{0,1-\alpha-\varepsilon\}$ and
$u_\varepsilon:=\min\{1,1-\alpha+\varepsilon\}$.

\subsection{Proof of Theorem~\ref{thm:impossible}(a)}

\begin{proof}
Suppose, toward a contradiction, that $\cA$ is both
$(1,\delta)$-list replicable and uniformly two-sided
$\varepsilon$-valid, where
$\varepsilon+\delta<\min\{\alpha,1-\alpha\}$.  In particular,
$\delta<1/2$.  Let
\[
  \mu_P
  :=\mathcal L_{P^{\otimes n}\otimes\nu}
       (\cA(\cD;R)).
\]
For each $P\in\cP_{\mathrm{ac}}$, list replicability and
$1-\delta>0$ imply that $L_P=\{v_P\}$ for some $v_P\in\mathsf V$ and
\[
  \mu_P(\{v_P\})\geq1-\delta.
\]
Because $1-\delta>1/2$, this heavy atom is unique.

We first show that $v_P$ is independent of $P$.  By data processing,
invariance of total variation under multiplication by the common seed
law, and the product bound,
\begin{align*}
  \dtv(\mu_P,\mu_Q)
  &\leq
  \dtv(P^{\otimes n}\otimes\nu,Q^{\otimes n}\otimes\nu)\\
  &=\dtv(P^{\otimes n},Q^{\otimes n})
  \leq 1-\{1-\dtv(P,Q)\}^n
  \leq n\,\dtv(P,Q).
\end{align*}
Consequently, if $n\dtv(P,Q)<1-2\delta$, then
\[
  \mu_Q(\{v_P\})
  \geq1-\delta-n\dtv(P,Q)>\delta.
\]
If $v_P\neq v_Q$, the corresponding singletons are disjoint, and
\[
  1
  \geq\mu_Q(\{v_P\})+\mu_Q(\{v_Q\})
  >\delta+(1-\delta)=1,
\]
a contradiction.  Thus $v_P=v_Q$ whenever
$n\dtv(P,Q)<1-2\delta$.

For arbitrary $P,Q\in\cP_{\mathrm{ac}}$, put
$P_t=(1-t)P+tQ$, $0\leq t\leq1$.  The class
$\cP_{\mathrm{ac}}$ is convex, and
\[
  \dtv(P_s,P_t)
  =|s-t|\dtv(P,Q)\leq|s-t|.
\]
Choose an integer $m$ such that $n/m<1-2\delta$.  Applying the local
argument successively to $P_{j/m}$, $j=0,\ldots,m$, gives $v_P=v_Q$.
Hence there is one $v^\star\in\mathsf V$ such that
\[
  \mu_P(\{v^\star\})\geq1-\delta
  \qquad\text{for every }P\in\cP_{\mathrm{ac}}.
\]

For $v\in\mathsf V$, define
\[
  c_P(v)
  :=\int_{\mathbb R^d\times\cY}\Gamma(v,x,y)\,P(dx,dy).
\]
The function $v\mapsto c_P(v)$ is measurable by joint measurability of
$\Gamma$ and Tonelli's theorem.  Independence of the test pair gives
\[
  \operatorname{Cov}_P(\cA)
  =\int_{\mathsf V}c_P(v)\,\mu_P(dv).
\]
Since $0\leq c_P\leq1$,
\begin{align*}
  |\operatorname{Cov}_P(\cA)-c_P(v^\star)|
  &\leq\int_{\mathsf V\setminus\{v^\star\}}
       |c_P(v)-c_P(v^\star)|\,\mu_P(dv)
  \leq\delta.
\end{align*}
Uniform two-sided validity therefore implies, for every
$P\in\cP_{\mathrm{ac}}$,
\[
  c_P(v^\star)
  \in[1-\alpha-\varepsilon-\delta,\,
       1-\alpha+\varepsilon+\delta].
\]
The assumed inequality makes the lower endpoint strictly positive and
the upper endpoint strictly smaller than one.

For each $y\in\cY$, let
\[
  G_y:=\{x\in\mathbb R^d:y\in\cC_{v^\star}(x)\}.
\]
Each $G_y$ is Lebesgue measurable.  If some $G_y$ has positive
Lebesgue measure, the Lebesgue density theorem supplies a density point
$x_0\in G_y$.  For $h>0$, let
\[
  P_{x_0,h,y}
  :=\operatorname{Unif}(B(x_0,h))\otimes\delta_y.
\]
This law belongs to $\cP_{\mathrm{ac}}$, and
\[
  c_{P_{x_0,h,y}}(v^\star)
  =\frac{\operatorname{Leb}(G_y\cap B(x_0,h))}
         {\operatorname{Leb}(B(x_0,h))}
  \longrightarrow1
  \qquad(h\downarrow0),
\]
contradicting the uniform strict upper bound.  Hence every $G_y$ is
Lebesgue-null.  Since $\cY$ is finite and $P_X$ is absolutely
continuous, for every $P\in\cP_{\mathrm{ac}}$,
\[
  c_P(v^\star)
  =\sum_{y\in\cY}P\{X\in G_y,Y=y\}=0,
\]
contradicting the uniform strict lower bound.  This proves the
singleton-list impossibility.

For the seedless consequence, fix $P\in\cP_{\mathrm{ac}}$ and let
$V=\cA(\cD;R)$ and $V'=\cA(\cD';R')$.  These variables are independent
with common law $\mu_P$.  Because $\mathsf V$ is standard Borel, its
diagonal is measurable.  If $\{v_j:j\in J_P\}$ is the at most
countable set of point atoms of $\mu_P$ and
$a_j:=\mu_P(\{v_j\})$, Tonelli's theorem gives
\[
  \Pd\{V=V'\}
  =\int_{\mathsf V}\mu_P(\{v\})\,\mu_P(dv)
  =\sum_{j\in J_P}a_j^2.
\]
If this collision probability is at least $1-\rho>0$, then $J_P$ is
nonempty.  A summable family of positive masses has a largest member.
Indeed, if its positive supremum were not attained, infinitely many
members would exceed half that supremum, contradicting summability.
Write $a_P:=\max_{j\in J_P}a_j$.  Then
\[
  1-\rho
  \leq\sum_{j\in J_P}a_j^2
  \leq a_P\sum_{j\in J_P}a_j
  \leq a_P.
\]
Thus the singleton containing a largest atom captures the output with
probability at least $1-\rho$.  Since this holds for every $P$, the
procedure is $(1,\rho)$-list replicable.  The final assertion follows
by applying the first part with $\delta=\rho$.
\end{proof}

\subsection{Proof of Theorem~\ref{thm:impossible}(b)}

\begin{proof}
Because $\mathsf V$ is standard Borel, its diagonal
$\Delta_{\mathsf V}:=\{(v,v):v\in\mathsf V\}$ belongs to
$\mathcal V\otimes\mathcal V$.  Hence all equality events below are
measurable.

Choose a probability density $\varphi$ on $\mathbb R^d$ satisfying
$\varphi(x)>0$ for every $x$ (for example, the standard Gaussian
density), and define $P_0\in\cP_{\mathrm{ac}}$ by
\[
  P_0(dx,\{y\})=\frac{1}{K}\varphi(x)\,dx,
  \qquad y\in\cY.
\]
This law dominates every $P\in\cP_{\mathrm{ac}}$.  Indeed, if a Borel
set $A\subseteq\mathbb R^d\times\cY$ satisfies $P_0(A)=0$, then every
section $A_y:=\{x:(x,y)\in A\}$ is Lebesgue-null, because
\[
  0=P_0(A)=\frac1K\sum_{y\in\cY}\int_{A_y}\varphi(x)\,dx
\]
and $\varphi>0$.  Therefore
\[
  P(A)
  \leq P_X\!\left(\bigcup_{y\in\cY}A_y\right)=0,
\]
so $P\ll P_0$ and hence $P^{\otimes n}\otimes\nu
\ll P_0^{\otimes n}\otimes\nu$.

Apply exact shared-seed replication under $P_0$.  With
\[
  f(d,d',r)
  :=\one\{\cA(d;r)\neq\cA(d';r)\},
\]
joint measurability of $\cA$ and measurability of
$\Delta_{\mathsf V}$ imply that $f$ is measurable.  Exact replication
and Tonelli's theorem give
\begin{align*}
  0
  &=\int f(d,d',r)\,
      P_0^{\otimes n}(dd)P_0^{\otimes n}(dd')\nu(dr)\\
  &=\int\!\left[
      \int f(d,d',r)\,P_0^{\otimes n}(dd)\nu(dr)
    \right]P_0^{\otimes n}(dd').
\end{align*}
Thus, for $P_0^{\otimes n}$-almost every $d'$, the bracketed integral
is zero.  Choose one such $d_0$ and define
$v_\star(r):=\cA(d_0;r)$.
This is a measurable section of $\cA$, and
\[
  \cA(d;r)=v_\star(r)
  \qquad
  (P_0^{\otimes n}\otimes\nu)\text{-almost surely}.
\]
By the domination established above, the same identity holds
$P^{\otimes n}\otimes\nu$-almost surely for every
$P\in\cP_{\mathrm{ac}}$.  This proves the claimed data-oblivious
representation.  It also gives the common output law
$\mu_\star=\nu\circ v_\star^{-1}$.
The exceptional set under $P_0^{\otimes n}\otimes\nu$ is fixed;
domination transfers it separately to each $P$, so no union over the
uncountable class $\cP_{\mathrm{ac}}$ is taken.

Now assume uniform two-sided $\varepsilon$-validity.  Since
$(r,x,y)\mapsto\Gamma(v_\star(r),x,y)$ is measurable and bounded,
parameter integration shows that
\[
  p(x,y):=\int_{\mathsf R}
            \Gamma(v_\star(r),x,y)\,\nu(dr)
\]
is measurable; equivalently
$p(x,y)=\int_{\mathsf V}\Gamma(v,x,y)\,\mu_\star(dv)$ by the
change-of-variables formula for the pushforward.  Independence of the
test pair and Tonelli's theorem yield, for every
$P\in\cP_{\mathrm{ac}}$ (taking the product with the independent test
law preserves the preceding almost-sure representation),
\[
  \operatorname{Cov}_P(\cA)
  =\int_{\mathbb R^d\times\cY}
      \int_{\mathsf R}\Gamma(v_\star(r),x,y)\,\nu(dr)\,P(dx,dy)
  =\E_P[p(X,Y)].
\]

Fix $y\in\cY$, $x_0\in\mathbb R^d$, and $h>0$, and set
$P_{x_0,h,y}:=\operatorname{Unif}(B(x_0,h))\otimes\delta_y
\in\cP_{\mathrm{ac}}$.  Uniform two-sided validity and the coverage
identity imply
\[
  1-\alpha-\varepsilon
  \leq
  \frac{1}{\operatorname{Leb}(B(x_0,h))}
  \int_{B(x_0,h)}p(x,y)\,dx
  \leq
  1-\alpha+\varepsilon.
\]
Because $0\leq p(\cdot,y)\leq1$, this function is locally integrable.
At every Lebesgue point of $p(\cdot,y)$, letting $h\downarrow0$ gives
\[
  1-\alpha-\varepsilon
  \leq p(x_0,y)\leq1-\alpha+\varepsilon.
\]
Combining this with $0\leq p\leq1$ gives
$\ell_\varepsilon\leq p(x_0,y)\leq u_\varepsilon$.  This holds for
Lebesgue-almost every $x_0$ for each fixed $y$.  Since $\cY$ is finite,
the exceptional sets may be united into one Lebesgue-null set on whose
complement the bounds hold simultaneously for all labels.

Finally, outside that common null set,
\[
  \E_{V\sim\mu_\star}|\cC_V(x)|
  =\sum_{y\in\cY}
      \int_{\mathsf V}\Gamma(v,x,y)\,\mu_\star(dv)
  =\sum_{y\in\cY}p(x,y).
\]
Summing the pointwise lower and upper bounds over the $K$ labels proves
the asserted two-sided size bound.
\end{proof}

\begin{remark}[Determinism is not singleton-list replicability]
\label{rem:deterministic-not-list}
For a concrete example, specialize to $\alpha=1/10$, $n=19$, and
$\mathsf V=\mathbb R$.  For twenty iid pairs
$(X_i,Y_i)\sim P\in\cP_{\mathrm{ac}}$, put $Z_i=(X_i)_1$ and let
$\cD=((X_i,Y_i))_{i=1}^{19}$.  Let the deterministic procedure ignore
its seed and output
\[
  \cA(\cD;r):=T:=Z_{(18)},\qquad r\in\mathsf R,
\]
the eighteenth order statistic of $Z_1,\ldots,Z_{19}$.  For
$t\in\mathbb R$, define
\[
  \cC_t(x)
  =
  \begin{cases}
    \cY,        & x_1\leq t,\\
    \varnothing,& x_1>t.
  \end{cases}
\]
Absolute continuity of $P_X$ implies that the first-coordinate marginal
is atomless.  Hence the rank of $Z_{20}$ among $Z_1,\ldots,Z_{20}$ is
uniform on $\{1,\ldots,20\}$, and
\[
  \Pd_P\!\left\{Y_{20}\in
    \cC_{\cA(\cD;R)}(X_{20})\right\}
  =\Pd_P\{Z_{20}\leq Z_{(18)}\}
  =\frac{18}{20}=1-\alpha.
\]
Thus the procedure is deterministic and uniformly two-sided $0$-valid.
For every $t\in\mathbb R$,
\[
  \Pd_P(T=t)
  \leq\sum_{i=1}^{19}\Pd_P(Z_i=t)=0,
\]
so its output law is atomless.  If $s<t$, then any $x$ satisfying
$s<x_1\leq t$ obeys $\cC_s(x)=\varnothing$ and $\cC_t(x)=\cY$.
Thus distinct thresholds encode distinct deployed maps.  Therefore no
singleton list has positive success probability, and the procedure is
not $(1,\delta)$-list replicable for any $\delta<1$.
\end{remark}

\begin{remark}[The exact shared-seed endpoint is nonempty]
\label{rem:exact-oblivious-example}
Let $v_{\mathrm{all}}$ and $v_{\varnothing}$ be distinct outputs encoding,
respectively, the everywhere-$\cY$ and everywhere-empty maps.  Take
$\mathsf R=\{0,1\}$, let $\nu(\{1\})=1-\alpha$, and write
$B\sim\nu$.  Using the same $B$ in both runs, let
\[
  \cA(\cD;B)
  =
  \begin{cases}
    v_{\mathrm{all}},&B=1,\\
    v_{\varnothing},&B=0.
  \end{cases}
\]
This procedure is exactly shared-seed replicable and uniformly
two-sided $0$-valid, because its coverage is
$(1-\alpha)\cdot1+\alpha\cdot0=1-\alpha$, but it is data-oblivious.
Its largest output atom has mass
$\max\{\alpha,1-\alpha\}$, so it is
$(1,\min\{\alpha,1-\alpha\})$-list replicable and is not
$(1,\delta)$-list replicable for any smaller $\delta$.  Thus, when
$\varepsilon=0$, the strict singleton-list boundary in
part~\emph{(a)} is sharp.  This example does not claim sharpness of the
separate seedless-$\rho$ consequence.
\end{remark}

\section{Proofs for Sections~\ref{sec:prelim} and \ref{sec:method}}\label{app:upper}

\paragraph{Conventions.}
Throughout the appendices, $F^{-1}(v):=\inf\{t\in\mathbb R:F(t)\ge v\}$;
$p:=1-\alpha$; and $k_n:=\lceil(1-\alpha)(n+1)\rceil$ is the
manuscript's $k$.  Scores are smaller-is-more-conforming: for a fixed
score $s$, write $\cC_t^s(x):=\{y\in\cY:s(x,y)\le t\}$, suppressing the
superscript when the score is clear, so that increasing the threshold
can only enlarge the prediction set.

\paragraph{Measurable setup.}
Let $(\mathcal X,\mathcal A)$ and $(\mathcal Y,\mathcal B)$ be measurable
spaces, and let $P$ be a probability measure on
$(\mathcal X\times\mathcal Y,\mathcal A\otimes\mathcal B)$.  All score
functions below are measurable with respect to the indicated sigma-fields.
Write $P_X$ for the $\mathcal X$-marginal of $P$.

\begin{remark}[Deterministic perturbation: reproducibility and continuity]
\label{rem:jitter-full}
Let
\[
s:\mathcal X\times\mathcal Y\to\mathbb R,\qquad
g:\mathcal X\times\mathcal Y\to(0,\infty),\qquad
u:\mathcal X\to[0,1)
\]
be fixed measurable functions, used identically by all analysts, and
define
\[
\widetilde s(x,y):=s(x,y)-u(x)g(x,y).
\]
Alternatively, let $\mathsf T$ be common proper-training information,
including any common algorithmic seed, taking values in a standard Borel
space and independent of all calibration and test observations.  The
learned functions are assumed to be jointly measurable in $\mathsf T$ and
their displayed arguments.  Conditional on $\mathsf T$, the realized
functions are then treated as fixed, and every atomlessness or continuity
hypothesis below must hold for almost every realization of $\mathsf T$.
Separate proper-training samples may preserve validity within each
analyst, but they do not imply that the analysts evaluate the same
realized score map.

Under the usual split-conformal exchangeability assumptions, replacing $s$
by $\widetilde s$ in both calibration-score and test-candidate evaluation
preserves the standard finite-sample marginal coverage guarantee.  It need
not preserve the prediction sets, their sizes, their conditional coverage,
or the ordering of unequal scores.  Thus this construction is a deterministic
perturbation, not necessarily a pure tie-breaker.  Two analysts using the
same realized functions compute the same value $\widetilde s(x,y)$ at the
same input $(x,y)$, but thresholds computed from independent calibration
samples need not coincide.

For $(X,Y)\sim P$, set
\[
A:=s(X,Y),\qquad G:=g(X,Y),\qquad U:=u(X),
\]
and let $K((a,\gamma),\cdot)$ be a Borel probability kernel on $\mathbb R$
representing the conditional law of $U$ given $(A,G)=(a,\gamma)$.  Such a
kernel exists because $(A,G)$ and $U$ take values in standard Borel spaces.
Since $U\in[0,1)$ almost surely, the kernel may be chosen to satisfy
$K((a,\gamma),[0,1))=1$ for $P_{A,G}$-almost every $(a,\gamma)$.  Suppose
that there is a Borel set $N\subseteq\mathbb R\times(0,\infty)$ with
$P_{A,G}(N)=0$ such that
\[
K((a,\gamma),\{v\})=0
\qquad
\text{for every }(a,\gamma)\notin N
\text{ and every }v\in\mathbb R.
\]
Equivalently, the conditional law is atomless at every $(a,\gamma)$ outside
one $P_{A,G}$-null set.  Then, for every $t\in\mathbb R$, disintegration
gives
\[
\begin{aligned}
\Pr\{\widetilde s(X,Y)=t\}
&=\Pr\{A-GU=t\}\\
&=\int_{\mathbb R\times(0,\infty)}
  \int_{\mathbb R}
  \mathbf 1\{a-\gamma v=t\}\,
  K((a,\gamma),dv)\,P_{A,G}(da,d\gamma)\\
&=\int_{\mathbb R\times(0,\infty)}
  K\!\left((a,\gamma),
  \left\{\frac{a-t}{\gamma}\right\}\right)
  P_{A,G}(da,d\gamma)
=0.
\end{aligned}
\]
Thus $\widetilde s(X,Y)$ has an atomless distribution, equivalently, a
continuous distribution function.  Consequently, if
$((X_i,Y_i))_{i=1}^m$ is an i.i.d.\ sample from $P$, then
\[
\Pr\!\left\{
\widetilde s(X_i,Y_i)=\widetilde s(X_j,Y_j)
\text{ for some }1\leq i<j\leq m
\right\}=0.
\]
This is a distributional statement about independent observations.  It does
not assert pointwise separation of every deterministic pair, or of several
candidate labels evaluated at one common covariate value.

A simpler, stronger sufficient condition is that $U$ be atomless and
independent of $(A,G)$.  For example, suppose that $X=(X_0,X_1)$,
$X_0$ is independent of $(X_1,Y)$, and
\[
\begin{aligned}
u(x_0,x_1)&=u_0(x_0),\\
s((x_0,x_1),y)&=s_1(x_1,y),\\
g((x_0,x_1),y)&=g_1(x_1,y)>0,
\end{aligned}
\]
where $u_0$ takes values in $[0,1)$ and $u_0(X_0)$ is atomless.  Then $U$ is
atomless and independent of $(A,G)$.

Marginal atomlessness of $U$ alone is insufficient.  Indeed, if
\[
\mathcal X=(0,1),\qquad X\sim\operatorname{Unif}(0,1),\qquad
s(x,y)=x,\qquad u(x)=x,\qquad g(x,y)=1,
\]
then $U=u(X)$ is atomless, whereas
$\widetilde s(X,Y)=X-X=0$
almost surely.  A finite-range computer hash is discrete and therefore
cannot itself satisfy the conditional-atomlessness condition above.  This
does not imply that the final transformed score must have atoms; continuity
of that score would instead require a separate argument.

For clarity, when Proposition~\ref{prop:vanilla} is applied to a score
$s_{\mathsf T}$ learned from common training information, use the value of
$k$ defined there and set
\[
\mathcal C_{t,\mathsf T}(x)
:=\{y\in\mathcal Y:s_{\mathsf T}(x,y)\leq t\},
\qquad
S_{i,\mathsf T}:=s_{\mathsf T}(X_i,Y_i),
\qquad
S_{i,\mathsf T}':=s_{\mathsf T}(X_i',Y_i'),
\]
and let $\tau_{n,\mathsf T}$ and $\tau_{n,\mathsf T}'$ be the respective
$k$th order statistics.  Let $F_{\mathsf T}(\cdot)$ denote a version of the
regular conditional distribution function of $s_{\mathsf T}(X,Y)$ given
$\mathsf T$; thus
$F_{\mathsf T}(t):=\Pr\{s_{\mathsf T}(X,Y)\leq t\mid\mathsf T\}$.
Suppose, for this version, that $t\mapsto F_{\mathsf T}(t)$ is continuous for
almost every realization of $\mathsf T$.  Then, for almost every realized
training value, the proposition applies to the fixed score $s_{\mathsf T}$,
with every conditional probability conditioned additionally on $\mathsf T$.
In particular, conditionally on $\mathsf T$,
$F_{\mathsf T}(\tau_{n,\mathsf T})$ and
$F_{\mathsf T}(\tau_{n,\mathsf T}')$ are independent
$\operatorname{Beta}(k,n+1-k)$ random variables,
\[
\begin{aligned}
&\Pr\!\left\{
\mathbf 1\{Y\in\mathcal C_{\tau_{n,\mathsf T},\mathsf T}(X)\}
\neq
\mathbf 1\{Y\in\mathcal C_{\tau_{n,\mathsf T}',\mathsf T}(X)\}
\,\middle|\,
\mathsf T,\mathcal D_n,\mathcal D_n'
\right\}\\
&\qquad
=\left|
F_{\mathsf T}(\tau_{n,\mathsf T})
-F_{\mathsf T}(\tau_{n,\mathsf T}')
\right|>0
\quad\text{almost surely},
\end{aligned}
\]
and the relevant conditional coverage variable is
\[
\Pr\!\left\{
Y_{n+1}\in
\mathcal C_{\tau_{n,\mathsf T},\mathsf T}(X_{n+1})
\,\middle|\,
\mathsf T,\mathcal D_n
\right\}
=F_{\mathsf T}(\tau_{n,\mathsf T}).
\]
Unconditional continuity alone is insufficient.  For example, if
$\mathsf T\sim\operatorname{Unif}(0,1)$ and
$s_{\mathsf T}(x,y)=\mathsf T$, then the unconditional score law is uniform,
but, conditionally on $\mathsf T$, all scores and both analysts'
thresholds coincide.
\end{remark}

\paragraph{Full statement of Proposition~\ref{prop:vanilla}.}
Let $n\geq1$ and
\[
\alpha\in\left[\frac{1}{n+1},1\right),\qquad
k:=\left\lceil(n+1)(1-\alpha)\right\rceil
\in\{1,\ldots,n\}.
\]
Let $s:\mathcal X\times\mathcal Y\to\mathbb R$ be a fixed measurable
nonconformity score.  Put
$\overline{\mathbb R}:=\mathbb R\cup\{-\infty,+\infty\}$,
and, for $t\in\overline{\mathbb R}$, define
$\mathcal C_t(x):=\{y\in\mathcal Y:s(x,y)\leq t\}$.
Thus $\mathcal C_{+\infty}(x)=\mathcal Y$ and
$\mathcal C_{-\infty}(x)=\varnothing$ for every $x$.  Suppose that
$F(t):=\Pr\{s(X,Y)\leq t\}$, $(X,Y)\sim P$,
is continuous.  Let
$\mathcal D_n:=((X_i,Y_i))_{i=1}^n$ and
$\mathcal D_n':=((X_i',Y_i'))_{i=1}^n$
be independent i.i.d.\ calibration samples from $P$.  Write
$S_i:=s(X_i,Y_i)$, $S_i':=s(X_i',Y_i')$, and let
$\tau_n:=S_{(k)}$, $\tau_n':=S_{(k)}'$
be the corresponding $k$th order-statistic thresholds.  Then
\[
F(\tau_n),\,F(\tau_n')
\stackrel{\mathrm{ind}}{\sim}
\operatorname{Beta}(k,n+1-k),
\qquad
\Pr\{\tau_n=\tau_n'\}=0.
\]
Let $(X,Y)\sim P$ be independent of both calibration samples and define the
conditional operational disagreement probability
\[
\Delta_n
:=
\Pr\!\left\{
\mathbf 1\{Y\in\mathcal C_{\tau_n}(X)\}
\neq
\mathbf 1\{Y\in\mathcal C_{\tau_n'}(X)\}
\,\middle|\,
\mathcal D_n,\mathcal D_n'
\right\}.
\]
Then $\Delta_n=|F(\tau_n)-F(\tau_n')|>0$ almost surely.
Define the following possibly nonmeasurable subsets of the joint
calibration-sample space:
\[
\begin{aligned}
E_{\mathrm{ae}}
&:=\left\{
\begin{array}{l}
\text{there exists }N\in\mathcal A
\text{ with }P_X(N)=0\text{ such that}\\[-1mm]
\mathcal C_{\tau_n}(x)=\mathcal C_{\tau_n'}(x)
\text{ for every }x\notin N
\end{array}
\right\},\\
E_{\mathrm{pt}}
&:=\{\mathcal C_{\tau_n}(x)=\mathcal C_{\tau_n'}(x)
      \text{ for every }x\in\mathcal X\}.
\end{aligned}
\]
If $\Pr^*$ denotes outer probability under the joint law of the two
calibration samples, then
$\Pr^*(E_{\mathrm{ae}})=\Pr^*(E_{\mathrm{pt}})=0$.
Consequently, both subsets are measurable in the completion, where each has
probability zero.  Thus the prediction-set maps fail to replicate exactly
almost surely.
Finally, if $(X_{n+1},Y_{n+1})\sim P$ is independent of $\mathcal D_n$, then
the calibration-conditional population coverage
\[
\Gamma_n
:=
\Pr\!\left\{
Y_{n+1}\in\mathcal C_{\tau_n}(X_{n+1})
\,\middle|\,
\mathcal D_n
\right\}
\]
satisfies $\Gamma_n=F(\tau_n)\sim\operatorname{Beta}(k,n+1-k)$, and
consequently
$\Pr\{Y_{n+1}\in\mathcal C_{\tau_n}(X_{n+1})\}=k/(n+1)\geq1-\alpha$.

\begin{proof}
Continuity of $F$ and the probability integral transform give
\[
V_i:=F(S_i)\stackrel{\mathrm{i.i.d.}}{\sim}\operatorname{Unif}(0,1),
\qquad
V_i':=F(S_i')\stackrel{\mathrm{i.i.d.}}{\sim}\operatorname{Unif}(0,1).
\]
The two collections are independent.  Since $F$ is nondecreasing, applying
$F$ commutes with taking order statistics, even if $F$ has flat intervals:
\[
F(\tau_n)=F(S_{(k)})=V_{(k)},
\qquad
F(\tau_n')=F(S_{(k)}')=V_{(k)}'
\]
almost surely.
The $k$th order statistic of $n$ independent
$\operatorname{Unif}(0,1)$ variables has density
\[
f_k(v)
=
\frac{n!}{(k-1)!(n-k)!}
v^{k-1}(1-v)^{n-k},
\qquad 0<v<1.
\]
Therefore,
$F(\tau_n),F(\tau_n')
\stackrel{\mathrm{ind}}{\sim}\operatorname{Beta}(k,n+1-k)$.
This Beta distribution is atomless, and hence
$\Pr\{F(\tau_n)=F(\tau_n')\}=0$.
Since
$\{\tau_n=\tau_n'\}\subseteq\{F(\tau_n)=F(\tau_n')\}$,
it follows that $\Pr\{\tau_n=\tau_n'\}=0$.

For deterministic $t,t'\in\mathbb R$, the two acceptance events are nested,
so
\[
\Pr\!\left\{
\mathbf 1\{Y\in\mathcal C_t(X)\}
\neq
\mathbf 1\{Y\in\mathcal C_{t'}(X)\}
\right\}
=\Pr\{\min(t,t')<s(X,Y)\leq\max(t,t')\}
=|F(t)-F(t')|.
\]
Substituting the calibration-measurable thresholds and conditioning on the
two independent calibration samples therefore yields
$\Delta_n=|F(\tau_n)-F(\tau_n')|>0$
almost surely.

For any realized $t,t'$, if
$\mathcal C_t(x)=\mathcal C_{t'}(x)$ outside some $N\in\mathcal A$ with
$P_X(N)=0$, then the measurable disagreement set is contained in
$N\times\mathcal Y$.  Hence
\[
\mathbf 1\{Y\in\mathcal C_t(X)\}
=
\mathbf 1\{Y\in\mathcal C_{t'}(X)\}
\qquad P\text{-almost surely}.
\]
The preceding disagreement identity then gives $F(t)=F(t')$.  Consequently,
\[
E_{\mathrm{pt}}
\subseteq E_{\mathrm{ae}}
\subseteq
\{F(\tau_n)=F(\tau_n')\}.
\]
The event on the right is measurable and has probability zero.  Both subsets
on the left therefore have outer probability zero and belong to the
completion with completed probability zero.

For the independent test observation $(X_{n+1},Y_{n+1})$,
$\Gamma_n=\Pr\{s(X_{n+1},Y_{n+1})\leq\tau_n\mid\mathcal D_n\}=F(\tau_n)$
almost surely.  The asserted Beta law follows from the order-statistic
calculation above.  Taking expectations and using the mean of a
$\operatorname{Beta}(k,n+1-k)$ random variable gives
$\Pr\{Y_{n+1}\in\mathcal C_{\tau_n}(X_{n+1})\}
=\mathbb E[\Gamma_n]=k/(n+1)$.
Finally,
$k=\lceil(n+1)(1-\alpha)\rceil\geq(n+1)(1-\alpha)$,
so $k/(n+1)\geq1-\alpha$.
\end{proof}

\begin{remark}[Boundary values of the miscoverage level]
\label{rem:vanilla-boundary}
At the included lower endpoint $\alpha=1/(n+1)$, one has $k=n$; thus
$\tau_n=S_{(n)}$ and $\Gamma_n\sim\operatorname{Beta}(n,1)$, so the
proposition remains valid.
If $0\leq\alpha<1/(n+1)$, then
$\lceil(n+1)(1-\alpha)\rceil=n+1$.
Under the usual finite-sample-valid convention of adjoining $+\infty$ as the
$(n+1)$st order statistic, $\tau_n=+\infty$.  Since $s$ is real-valued,
$\mathcal C_{\tau_n}(x)=\mathcal Y$ for every $x$,
and two analysts replicate trivially with probability one.
At the other end, every $0\leq\alpha<1$ gives $k\geq1$.  In particular, if
$n/(n+1)\leq\alpha<1$,
then $k=1$ and $\Gamma_n\sim\operatorname{Beta}(1,n)$.  The endpoint
$\alpha=1$ is excluded because the formula gives $k=0$, not a calibration
order statistic.  If one extends the convention by setting $\tau_n=-\infty$
at $\alpha=1$, both prediction sets are empty and exact replication again
has probability one.
\end{remark}

\paragraph{Order-statistic facts.}
For continuous $F$ and $k_n=\lceil p(n+1)\rceil\le n$,
\begin{equation}\label{eq:revised-beta-facts}
U:=F(\ttau)=F(S_{(k_n)})
\sim\operatorname{Beta}(k_n,n+1-k_n),
\qquad
m_n:=\E U=\frac{k_n}{n+1},
\end{equation}
with
\begin{equation}\label{eq:revised-rank-facts}
p\le m_n<p+\frac1{n+1},
\qquad
0\le\frac{k_n}{n}-p<\frac2n,
\qquad
\operatorname{Var}(U)=\frac{m_n(1-m_n)}{n+2}.
\end{equation}
Continuity implies that the sample scores have no ties almost surely, so
$\widehat F_n(\ttau)=k_n/n$.

\begin{lemma}[Offset collision]\label{lem:collision}
For fixed $a,b\in\mathbb R$ and
$u\sim\operatorname{Unif}[0,\beta)$,
\begin{equation}\label{eq:collision}
\Pd_u\{\operatorname{rd}_u(a)\ne\operatorname{rd}_u(b)\}
=\frac{\min\{|a-b|,\beta\}}{\beta}.
\end{equation}
\end{lemma}

\begin{proof}
The two round-ups differ exactly when a point of $u+\beta\mathbb Z$ lies
in the half-open interval between $a$ and $b$.  If
$d:=|a-b|\ge\beta$, every grid phase has such a point.  If $d<\beta$,
reducing the phase modulo $\beta$ shows that the set of offsets producing
such a point is an arc of length $d$ in $[0,\beta)$.  Dividing its
Lebesgue measure by $\beta$ proves \eqref{eq:collision}.
\end{proof}

\begin{lemma}[Localization]\label{lem:local}
Let $F$ be continuous, let $S_1,\ldots,S_n$ be i.i.d.\ with cdf $F$,
let $k_n=\lceil(1-\alpha)(n+1)\rceil\le n$, and put
$\ttau=S_{(k_n)}$.  Under Assumption~\ref{ass:margin},
$|F(\ttau)-p|<f_{\min}\Delta$ implies $\ttau\in I$, and for every
$a,b\in I$,
\begin{equation}\label{eq:inverse-lipschitz}
|a-b|\le\frac{|F(a)-F(b)|}{f_{\min}}.
\end{equation}
Moreover, if $n\ge4/(f_{\min}\Delta)$, then
\begin{equation}\label{eq:localization-tail}
\Pd\{|F(\ttau)-p|\ge f_{\min}\Delta\}
\le2\exp\{-2n(f_{\min}\Delta-2/n)^2\}
\le2\exp\!\left(-\frac{n f_{\min}^{2}\Delta^{2}}2\right).
\end{equation}
\end{lemma}

\begin{proof}
Continuity gives $F(q)=p$.  Absolute continuity and the lower density
bound give
\[
F(q+\Delta)-p\ge f_{\min}\Delta,
\qquad
p-F(q-\Delta)\ge f_{\min}\Delta.
\]
Monotonicity therefore proves the localization claim.  If $a<b$ are in
$I$, then
\[
F(b)-F(a)=\int_a^b f(t)\,dt\ge f_{\min}(b-a),
\]
which proves \eqref{eq:inverse-lipschitz} without requiring pointwise
differentiability of $F$.
By \eqref{eq:revised-rank-facts},
$|F(\ttau)-p|\le\|\widehat F_n-F\|_{\infty}+2/n$.
The Dvoretzky--Kiefer--Wolfowitz inequality therefore gives the first
bound in \eqref{eq:localization-tail}.  Since
$n\ge4/(f_{\min}\Delta)$ implies
$f_{\min}\Delta-2/n\ge f_{\min}\Delta/2$, the second bound follows.
\end{proof}

\subsection*{Proof of Theorem~\ref{thm:upper}}
Let $E$ be the event that both unrounded thresholds satisfy
$|F(\ttau_j)-p|<f_{\min}\Delta$.  Lemma~\ref{lem:local} and a union
bound give
\begin{equation}\label{eq:upper-local-two}
\Pd(E^c)\le4\exp\!\left(-\frac{n f_{\min}^{2}\Delta^{2}}2\right).
\end{equation}
Conditional on the two unrounded thresholds, Lemma~\ref{lem:collision},
followed on $E$ by \eqref{eq:inverse-lipschitz}, gives
\begin{align}
\Pd(\tau_A\ne\tau_B)
&\le
\E\!\left[
\frac{\min\{|\ttau_A-\ttau_B|,\beta\}}{\beta}\one_E
\right]+\Pd(E^c)\notag\\
&\le
\frac{\E|F(\ttau_A)-F(\ttau_B)|}{f_{\min}\beta}
+4\exp\!\left(-\frac{n f_{\min}^{2}\Delta^{2}}2\right).
\label{eq:upper-before-beta}
\end{align}
The transformed thresholds are independent copies $U,U'$ of the beta
variable in \eqref{eq:revised-beta-facts}, and hence
\begin{equation}\label{eq:beta-difference}
\E|U-U'|
\le\sqrt{\E(U-U')^2}
=\sqrt{2\operatorname{Var}(U)}.
\end{equation}
Writing $m_n=p+r_n$ with $0\le r_n<1/(n+1)$,
\[
m_n(1-m_n)
=\alpha(1-\alpha)+r_n(2\alpha-1)-r_n^2
\le\alpha(1-\alpha)+\frac{\alpha}{n+1}.
\]
Consequently,
\begin{equation}\label{eq:beta-variance-bound}
\operatorname{Var}(U)
\le\frac{\alpha(1-\alpha)}n+\frac1{n^2},
\end{equation}
and
\[
\E|U-U'|
\le
\sqrt{\frac{2\alpha(1-\alpha)}n+\frac2{n^2}}
\le
\sqrt{\frac{2\alpha(1-\alpha)}n}+\frac{\sqrt2}{n}.
\]
Substitution into \eqref{eq:upper-before-beta} proves part~(i).

For marginal validity, upward rounding gives
$\ttau\le\tau<\ttau+\beta$ for every offset, so
$\cC_{\ttau}(x)\subseteq\cC_{\tau}(x)$ for every $x$.  Among $n+1$
i.i.d.\ continuous scores, the test score has a uniform rank.  Therefore
\[
\Pd\{Y_{n+1}\in\cC_\tau(X_{n+1})\}
\ge\Pd\{S_{n+1}\le S_{(k_n)}\}
=\frac{k_n}{n+1}\ge1-\alpha,
\]
proving the marginal claim of part~(ii).  This argument holds for every
fixed offset and does not use the margin condition.

For the training-conditional statement, DKW implies that, with
probability at least $1-\delta$,
\begin{equation}\label{eq:upper-dkw}
|F(\ttau)-p|
\le\|\widehat F_n-F\|_{\infty}
   +\left|\frac{k_n}{n}-p\right|
\le e_n(\delta).
\end{equation}
Under the condition $e_n(\delta)\le f_{\min}(\Delta-\beta)$, this first
localizes $\ttau$ to $I$ and then, by \eqref{eq:inverse-lipschitz}, gives
\[
|\ttau-q|\le\frac{e_n(\delta)}{f_{\min}}\le\Delta-\beta.
\]
Thus $[\ttau,\ttau+\beta]\subseteq I$.  Monotonicity and upward
rounding give
$F(\tau)\ge F(\ttau)\ge p-e_n(\delta)$,
while absolute continuity and the upper density bound give
\[
F(\tau)
\le F(\ttau)+f_{\max}(\tau-\ttau)
\le p+e_n(\delta)+f_{\max}\beta.
\]
The DKW event depends only on the sample, so these inequalities hold
simultaneously for every offset.  This proves part~(ii).

Finally, if $\tau_A=\tau_B=\tau$ and the score map is shared, then both
analysts deploy
$\cC^{s}_{\tau}(x)=\{y:s(x,y)\le\tau\}$
for every $x$, which proves part~(iii). \qed

\paragraph{Full statement of Corollary~\ref{cor:protocol}.}
Let $s$ be a fixed measurable score map whose score cdf $F$ under $P$ is
continuous and satisfies Assumption~\ref{ass:margin}.  The two
analysts use independent i.i.d.\ calibration samples from
$P^{\otimes n}$.  Fix
$\alpha,\varepsilon,\rho,\delta\in(0,1)$ and assume
\begin{equation}\label{eq:protocol-width}
\varepsilon\le f_{\max}\Delta.
\end{equation}
Set
\[
\beta:=\frac{\varepsilon}{2f_{\max}},
\qquad
h:=f_{\min}(\Delta-\beta)>0,
\qquad
n_{\alpha}:=\left\lfloor\frac1\alpha\right\rfloor,
\]
and define the explicit finite-sample remainder
\begin{equation}\label{eq:protocol-nzero}
n_0:=
\left\lceil
\max\left\{
n_{\alpha},
\frac8\varepsilon,
\frac{8\sqrt2\,\kappa}{\varepsilon\rho},
\frac{2\log(16/\rho)}{f_{\min}^{2}\Delta^{2}},
\frac{2\log(2/\delta)}{h^{2}},
\frac4h
\right\}
\right\rceil.
\end{equation}
If
\begin{equation}\label{eq:protocol-n}
n\ge
\max\left\{
\frac{32\kappa^{2}\alpha(1-\alpha)}
     {\varepsilon^{2}\rho^{2}},
\frac{8\log(2/\delta)}{\varepsilon^{2}},
n_0
\right\},
\end{equation}
then two analysts running \ReCal{} on independent samples of size
$n$, with the same score and the independent shared offset
$u\sim\operatorname{Unif}[0,\beta)$, satisfy
$\Pd_{\cD,\cD',u}(\tau\ne\tau')\le\rho$.
Consequently, their deployed classifiers are $\rho$-replicable.  Each
analyst has marginal coverage at least $1-\alpha$ and, with probability
at least $1-\delta$ over its calibration sample,
\begin{equation}\label{eq:protocol-band}
F(\tau)\in
\left[1-\alpha-\frac\varepsilon2,\;1-\alpha+\varepsilon\right].
\end{equation}
The event giving \eqref{eq:protocol-band} works for every realized shared
offset $u$.

\begin{remark}[Meaning of ``lower order'']\label{rem:nzero}
For fixed $\alpha,\delta,f_{\min},f_{\max},\Delta$, the explicit $n_0$
in \eqref{eq:protocol-nzero} is
$o((\varepsilon\rho)^{-2})$ as $\varepsilon\rho\to0$ along parameter
sequences satisfying \eqref{eq:protocol-width}.  This assertion is not
uniform if $\alpha,\delta$ or the margin constants vary.  A guarantee
uniform over a distribution class requires common certified values of
$f_{\min},f_{\max},\Delta$ throughout that class.
\end{remark}

\subsection*{Proof of Corollary~\ref{cor:protocol}}
Condition \eqref{eq:protocol-width} gives $\beta\le\Delta/2$.  Also,
$n\ge n_\alpha$ implies $\alpha>1/(n+1)$ and hence $k_n\le n$, while
$n\ge4/h\ge4/(f_{\min}\Delta)$ gives the localization sample-size
condition of Theorem~\ref{thm:upper}.
Using $\beta=\varepsilon/(2f_{\max})$ in
Theorem~\ref{thm:upper}(i) gives
\begin{equation}\label{eq:protocol-three-terms}
\Pd(\tau\ne\tau')
\le
\frac{2\kappa}{\varepsilon}
\sqrt{\frac{2\alpha(1-\alpha)}n}
+\frac{2\sqrt2\,\kappa}{\varepsilon n}
+4\exp\!\left(-\frac{n f_{\min}^{2}\Delta^{2}}2\right).
\end{equation}
The first lower bound on $n$ in \eqref{eq:protocol-n} makes the first
term at most $\rho/2$.  The conditions
\[
n\ge\frac{8\sqrt2\,\kappa}{\varepsilon\rho},
\qquad
n\ge\frac{2\log(16/\rho)}{f_{\min}^{2}\Delta^{2}}
\]
make the second and third terms at most $\rho/4$ each.  Thus
$\Pd(\tau\ne\tau')\le\rho$, and equality of the deployed maps on the
complement follows from the shared score.
For the conditional coverage band,
\[
n\ge\frac{8\log(2/\delta)}{\varepsilon^2},
\qquad n\ge\frac8\varepsilon
\]
imply, respectively,
$\sqrt{\log(2/\delta)/(2n)}\le\varepsilon/4$ and $2/n\le\varepsilon/4$,
so $e_n(\delta)\le\varepsilon/2$.  Likewise,
\[
n\ge\frac{2\log(2/\delta)}{h^2},
\qquad n\ge\frac4h
\]
give $e_n(\delta)\le h=f_{\min}(\Delta-\beta)$.
Theorem~\ref{thm:upper}(ii) and $f_{\max}\beta=\varepsilon/2$ now yield
\[
1-\alpha-\frac\varepsilon2
\le F(\tau)
\le1-\alpha+\frac\varepsilon2+f_{\max}\beta
=1-\alpha+\varepsilon.
\]
Marginal coverage follows from Theorem~\ref{thm:upper}(ii). \qed

\paragraph{Full statement of Corollary~\ref{cor:twolabs}.}
Suppose both analysts use the same score $s$, whose continuous score
distribution $F$ satisfies Assumption~\ref{ass:margin}.  Let their
calibration samples be independent i.i.d.\ samples from the same
population $P$, with sizes $n_A,n_B$, and assume
\[
\alpha\ge
\max\left\{\frac1{n_A+1},\frac1{n_B+1}\right\},
\qquad
n_A,n_B\ge\frac4{f_{\min}\Delta}.
\]
For any $\beta>0$, let the analysts use the same independent shared
offset $u\sim\operatorname{Unif}[0,\beta)$ and their respective conformal
ranks
$k_j:=\lceil(1-\alpha)(n_j+1)\rceil$, $j\in\{A,B\}$.
Writing $m:=n_A\wedge n_B$, their rounded thresholds satisfy
\begin{equation}\label{eq:unequal-mismatch}
\Pd(\tau_A\ne\tau_B)
\le
\frac{
\sqrt{\alpha(1-\alpha)}
 (n_A^{-1/2}+n_B^{-1/2})+3/m
}{f_{\min}\beta}
+4\exp\!\left(-\frac{m f_{\min}^{2}\Delta^{2}}2\right).
\end{equation}
Thus the analysts need not coordinate sample sizes.  On
$\{\tau_A=\tau_B\}$ their deployed set maps coincide everywhere, and
each analyst separately retains marginal coverage at least
$1-\alpha$.
No restriction $\beta\le\Delta/2$ is needed for
\eqref{eq:unequal-mismatch}.  If, in addition,
$\delta_A,\delta_B\in(0,1)$, $\delta_A+\delta_B<1$, and
$e_{n_j}(\delta_j)\le f_{\min}(\Delta-\beta)$ for
$j\in\{A,B\}$, then with probability at least
$1-\delta_A-\delta_B$ the following inequalities hold simultaneously
for $j\in\{A,B\}$:
\[
1-\alpha-e_{n_j}(\delta_j)
\le F(\tau_j)
\le1-\alpha+e_{n_j}(\delta_j)+f_{\max}\beta.
\]

\subsection*{Proof of Corollary~\ref{cor:twolabs}}
For $j\in\{A,B\}$, let
\[
U_j:=F(\ttau_j)
\sim\operatorname{Beta}(k_j,n_j+1-k_j),
\qquad
\mu_j:=\E U_j=\frac{k_j}{n_j+1}.
\]
The variables $U_A,U_B$ are independent.  The ceiling definition gives
\[
0\le\mu_j-p<\frac1{n_j+1},
\qquad
|\mu_A-\mu_B|\le\frac1{m+1}.
\]
The variance calculation in \eqref{eq:beta-variance-bound}, with $n_j$
in place of $n$, gives
$\operatorname{sd}(U_j)\le\sqrt{\alpha(1-\alpha)/n_j}+1/n_j$.
Therefore
\begin{align*}
\E|U_A-U_B|
&\le\E|U_A-\mu_A|+|\mu_A-\mu_B|+\E|U_B-\mu_B|\\
&\le
\sqrt{\alpha(1-\alpha)}(n_A^{-1/2}+n_B^{-1/2})
+\frac1{n_A}+\frac1{n_B}+\frac1{m+1}\\
&\le
\sqrt{\alpha(1-\alpha)}(n_A^{-1/2}+n_B^{-1/2})+\frac3m.
\end{align*}
Let $E$ be the event that both raw thresholds localize to
$I=[q-\Delta,q+\Delta]$.  Applying Lemma~\ref{lem:local} separately at
the two sample sizes gives
\[
\Pd(E^c)
\le2e^{-n_Af_{\min}^2\Delta^2/2}
  +2e^{-n_Bf_{\min}^2\Delta^2/2}
\le4e^{-mf_{\min}^2\Delta^2/2}.
\]
On $E$,
$|\ttau_A-\ttau_B|\le|U_A-U_B|/f_{\min}$.  Conditioning on the raw
thresholds and applying Lemma~\ref{lem:collision} now gives
\[
\Pd(\tau_A\ne\tau_B)
\le\frac{\E|U_A-U_B|}{f_{\min}\beta}
+4e^{-mf_{\min}^2\Delta^2/2},
\]
which is \eqref{eq:unequal-mismatch}.  The shared score gives classifier
identity whenever the thresholds agree.  Apply
Theorem~\ref{thm:upper}(ii) separately to the two analysts for the
marginal and conditional coverage claims, and use a union bound for the
simultaneous claim. \qed

\paragraph{Full statement of Corollary~\ref{cor:attack}.}
Let $M\ge2$.  Draw one offset
$u\sim\operatorname{Unif}[0,\beta)$ independently of
$\cD_1,\ldots,\cD_M\stackrel{\mathrm{iid}}{\sim}P^{\otimes n}$, commit
to it before drawing the samples, and use that same offset in all $M$
runs.  Suppose the calibrator $\cA$ is $\rho$-replicable in the ex-ante
sense
$\Pd_{u,\cD,\cD'}\{\cA(\cD;u)\ne\cA(\cD';u)\}\le\rho$.
Then
\begin{equation}\label{eq:shopping-cap}
\Pd_{u,\cD_{1:M}}
\{\exists j\le M:\cA(\cD_j;u)\ne\cA(\cD_1;u)\}
\le c_M:=\min\{1,(M-1)\rho\}.
\end{equation}
Consequently, on an event of probability at least $1-c_M$, every
measurable selection rule
$J=J(u,\cD_1,\ldots,\cD_M)\in\{1,\ldots,M\}$ deploys the same
classifier.
For a classifier $C$, write
$\operatorname{cov}_{P}(C):=\Pd_{(X,Y)\sim P}\{Y\in C(X)\}$.
If an unselected run is marginally valid, so that
$\E_{u,\cD}\operatorname{cov}_{P}\{\cA(\cD;u)\}\ge1-\alpha$,
then every data-dependent selection rule satisfies
\begin{equation}\label{eq:shopping-coverage}
\E_{u,\cD_{1:M}}
\operatorname{cov}_{P}\{\cA(\cD_J;u)\}
\ge \max\{0,1-\alpha-c_M\}.
\end{equation}
For \ReCal{}, the premise holds under exchangeability and the valid-rank
condition $k_n\le n$, because rounding is upward and the threshold sets
are nested.  In general,
one may not replace the right-hand side of
\eqref{eq:shopping-coverage} by $1-\alpha$: marginal validity of each
candidate does not survive arbitrary adaptive selection.

For comparison, let $\ttau_1,\ldots,\ttau_M$ be independent standard
split-conformal thresholds based on samples of size $n$, put
$C_j:=F(\ttau_j)$,
and suppose $F$ is continuous and $k_n\le n$.  Then
$C_j\stackrel{\mathrm{iid}}{\sim}
\operatorname{Beta}(k_n,n+1-k_n)$.  If $B_{a,b}$ denotes the cdf of a
$\operatorname{Beta}(a,b)$ variable, choosing the smallest threshold has
the exact expected population coverage
\begin{equation}\label{eq:shopping-beta-exact}
\E\min_{j\le M}C_j
=\int_0^1\{1-B_{k_n,n+1-k_n}(t)\}^{M}\,dt.
\end{equation}
For fixed $M$ and fixed $\alpha\in(0,1)$, as $n\to\infty$,
\begin{equation}\label{eq:shopping-fixed-M}
\E\min_{j\le M}C_j
=\mu_n-a_M\sigma_n+o(\sigma_n),
\end{equation}
where
\[
\mu_n=\frac{k_n}{n+1},
\qquad
\sigma_n^2=\frac{\mu_n(1-\mu_n)}{n+2},
\qquad
a_M=\E\max_{j\le M}Z_j,
\quad Z_j\stackrel{\mathrm{iid}}{\sim}N(0,1).
\]
Thus $\mu_n=1-\alpha+O(n^{-1})$,
$\sigma_n=\sqrt{\alpha(1-\alpha)/n}\{1+o(1)\}$, and
$a_M\sim\sqrt{2\log M}$ only as $M\to\infty$.  In particular, the
$\sqrt{2\log M}$ expression is a large-$M$ approximation, not an exact
finite-$M$ formula; \eqref{eq:shopping-fixed-M} makes no claim for a
joint regime $M=M_n\to\infty$.

Finally, equal support does not make the selected artifact statistically
indistinguishable from an honest artifact.  If $H_n$ is the continuous cdf
of one honest standard threshold, $T_{\mathrm{hon}}\sim H_n$, and
$T_{\min}:=\min_{j\le M}\ttau_j$, then
\begin{equation}\label{eq:shopping-transform}
H_n(T_{\mathrm{hon}})\sim\operatorname{Unif}(0,1),
\qquad
H_n(T_{\min})\sim\operatorname{Beta}(1,M),
\end{equation}
and
\begin{equation}\label{eq:shopping-tv}
\dtv\!\left(\mathcal L(T_{\min}),\mathcal L(T_{\mathrm{hon}})\right)
=\frac{M-1}{M}\,M^{-1/(M-1)}.
\end{equation}
Here $\dtv$ denotes the usual supremum, equivalently one-half
$L^1$, total-variation distance.
For finite $M$ the two laws are nevertheless mutually absolutely
continuous.  Hence a support-only check cannot certify selection with zero
false-positive probability and positive power, although a distributional
audit with nonzero error can distinguish the laws.

\subsection*{Proof of Corollary~\ref{cor:attack}}
For a fixed offset $u$, define
$\rho(u):=\Pd_{\cD,\cD'}\{\cA(\cD;u)\ne\cA(\cD';u)\}$.
The ex-ante replicability assumption is $\E_u\rho(u)\le\rho$.
Conditional on $u$, each pair $(\cD_1,\cD_j)$, $j\ge2$, is an
independent pair from $P^{\otimes n}$.  Hence
\[
\Pd\{\exists j\ge2:\cA(\cD_j;u)\ne\cA(\cD_1;u)\mid u\}
\le\min\{1,(M-1)\rho(u)\}.
\]
Averaging over $u$ and using
$\min\{1,x\}\le x$ proves \eqref{eq:shopping-cap}.  This is a joint
statement over the random offset and samples; ex-ante replicability does
not imply the same bound for every fixed realized offset.

Let $E$ denote the event that all $M$ outputs equal the first output.  On
$E$, any selection rule deploys $\cA(\cD_1;u)$.  Since population
coverage belongs to $[0,1]$,
\begin{align*}
&\E\operatorname{cov}_{P}\{\cA(\cD_J;u)\}
-\E\operatorname{cov}_{P}\{\cA(\cD_1;u)\}\\
&\quad=
\E\!\left[
\left(
\operatorname{cov}_{P}\{\cA(\cD_J;u)\}
-\operatorname{cov}_{P}\{\cA(\cD_1;u)\}
\right)\one_{E^c}
\right]
\ge-\Pd(E^c).
\end{align*}
Combining this inequality with marginal validity,
\eqref{eq:shopping-cap}, and nonnegativity of coverage proves
\eqref{eq:shopping-coverage}.

The beta law follows from the probability integral transform and
\eqref{eq:revised-beta-facts}.  Since the $C_j$ are nonnegative and
independent,
\[
\E\min_j C_j
=\int_0^1\Pd(\min_j C_j>t)\,dt
=\int_0^1\{1-B_{k_n,n+1-k_n}(t)\}^{M}\,dt,
\]
proving \eqref{eq:shopping-beta-exact}.  For fixed $M$, the standardized
beta variables $(C_j-\mu_n)/\sigma_n$ converge jointly to independent
standard normals.  Their second moments equal one, so the minimum of a
fixed number of them is uniformly integrable.  Consequently,
\[
\E\min_{j\le M}\frac{C_j-\mu_n}{\sigma_n}
\longrightarrow
\E\min_{j\le M}Z_j=-a_M,
\]
which proves \eqref{eq:shopping-fixed-M}.  The standard normal
extreme-value relation $a_M\sim\sqrt{2\log M}$ gives the stated large-$M$
interpretation.

For the final claim,
\[
\Pd\{H_n(T_{\min})\le v\}=1-(1-v)^M,
\qquad 0\le v\le1,
\]
which proves \eqref{eq:shopping-transform}.  Relative to the honest
threshold law, the selected law has Radon--Nikodym derivative
$M\{1-H_n(t)\}^{M-1}$.
This derivative is positive almost surely under both laws, proving mutual
absolute continuity.  Its crossing with $1$ occurs at
$v_*=1-M^{-1/(M-1)}$, and hence
\begin{align*}
\dtv\!\left(\mathcal L(T_{\min}),\mathcal L(T_{\mathrm{hon}})\right)
&=\int_0^{v_*}\{M(1-v)^{M-1}-1\}\,dv
=\frac{M-1}{M}\,M^{-1/(M-1)},
\end{align*}
proving \eqref{eq:shopping-tv}. \qed

\begin{corollary}[Robustness to imperfect sharing]\label{cor:robust}
Let analysts $A,B$ use fixed measurable scores $s_A,s_B$ on the same domain with $\|s_A-s_B\|_\infty\le\eta_s$, and independent size-$n$ calibration samples from $P_A,P_B$. Let $F,G$ be the CDFs of the \emph{common} score $s_A(X,Y)$ under $P_A$ and $P_B$ (then $\dtv(P_A,P_B)\le\eta_P$ implies $\|F-G\|_\infty\le\eta_P$, which fails if $G$ is instead defined through $s_B$), both continuous, with Assumption~\ref{ass:margin} for $F$, effective margin $\lambda:=f_{\min}\Delta-\eta_P>0$, $k\le n$, and $n\ge4/\lambda$. With shared $(\beta,u)$,
$\Pd(\tau_A\ne\tau_B)\le\big[\,b_n/(f_{\min}\beta)+(\eta_s+\eta_P/f_{\min})/\beta+4e^{-n\lambda^{2}/2}\,\big]\wedge1$, where $b_n:=\sqrt{2\alpha(1-\alpha)/n+2/n^{2}}\le B_n$: population mismatch shrinks the localization margin from $f_{\min}\Delta$ to $\lambda$. Each analyst separately retains its own marginal coverage $\ge1-\alpha$. Equal \emph{numerical} thresholds under different scores need not give equal maps: under the window condition $e_n(\delta)/f_{\min}+\beta+\eta_s\le\Delta$, on a $1-\delta$ event, whenever $\tau_A=\tau_B=\tau$ the realized test pair is classified differently with probability $\le\min\{1,2f_{\max}\eta_s\}$ (full-map version in the appendix); $\eta_s=0$ is sufficient, but not necessary, for map equality; this motivates distributing the score function as a hash-verified artifact. (Full statement and proof: Appendix~\ref{app:upper}.)
\end{corollary}

\paragraph{Full statement of Corollary~\ref{cor:robust}.}
Let analysts $A$ and $B$ use independent calibration samples of size
$n$ from $P_A$ and $P_B$, respectively.  Let their score functions
$s_A,s_B:\mathcal X\times\mathcal Y\to\mathbb R$ be fixed measurable
maps on the same domain.  Let $\eta_s,\eta_P\ge0$, and assume
\begin{equation}\label{eq:robust-score}
\|s_A-s_B\|_{\infty}\le\eta_s.
\end{equation}
Define the common-score cdfs
\begin{equation}\label{eq:robust-cdfs}
F(t):=P_A\{s_A(X,Y)\le t\},
\qquad
G(t):=P_B\{s_A(X,Y)\le t\}.
\end{equation}
Assume $F$ and $G$ are continuous and
\begin{equation}\label{eq:robust-cdf-gap}
\|F-G\|_{\infty}\le\eta_P.
\end{equation}
Condition \eqref{eq:robust-cdf-gap} follows, for example, from
$\dtv(P_A,P_B)\le\eta_P$, because the same fixed score $s_A$ appears in
both cdfs.  It would not, in general, follow from total variation if $G$
were instead defined using $s_B$.
Let $p:=1-\alpha$, $q:=F^{-1}(p)$, and suppose $F$ is absolutely
continuous on $I=[q-\Delta,q+\Delta]$, with density satisfying
$f_{\min}\le f\le f_{\max}$ almost everywhere on $I$.  Define
\begin{equation}\label{eq:robust-effective-margin}
\lambda:=f_{\min}\Delta-\eta_P>0,
\qquad
b_n:=\sqrt{\frac{2\alpha(1-\alpha)}n+\frac2{n^2}},
\end{equation}
and assume
$k_n=\lceil p(n+1)\rceil\le n$ and $n\ge4/\lambda$.
If both analysts use the same grid width $\beta>0$ and the same
independent offset $u\sim\operatorname{Unif}[0,\beta)$, then their
numerical rounded thresholds satisfy
\begin{equation}\label{eq:robust-mismatch}
\Pd(\tau_A\ne\tau_B)
\le
\left[
\frac{b_n}{f_{\min}\beta}
+\frac{\eta_s+\eta_P/f_{\min}}{\beta}
+4\exp\!\left\{-\frac n2
 (f_{\min}\Delta-\eta_P)^2\right\}
\right]\wedge1.
\end{equation}
Thus population mismatch reduces the effective localization margin from
$f_{\min}\Delta$ to $f_{\min}\Delta-\eta_P$.  Each analyst separately
retains split-conformal marginal coverage at least $1-\alpha$ for a fresh
test observation from its own population.  Bound
\eqref{eq:robust-mismatch} concerns equality of the numerical thresholds;
when $s_A\ne s_B$, equal thresholds need not define equal set maps.

Now let a fresh test pair $(X,Y)\sim P_A$ be independent of the entire
calibration experiment.  If, for some $\delta\in(0,1)$,
\begin{equation}\label{eq:robust-agreement-condition}
\frac{e_n(\delta)}{f_{\min}}+\beta+\eta_s\le\Delta,
\end{equation}
then, with probability at least $1-\delta$ over the calibration
experiment, every realization for which $\tau_A=\tau_B=: \tau$ satisfies
\begin{equation}\label{eq:robust-pair-agreement}
\Pd_{(X,Y)\sim P_A}\!\left(
\one\{s_A(X,Y)\le\tau\}\ne\one\{s_B(X,Y)\le\tau\}
\,\middle|\,\cD_A,\cD_B,u
\right)
\le \min\{1,2f_{\max}\eta_s\}.
\end{equation}
Consequently,
\begin{equation}\label{eq:robust-pair-unconditional}
\Pd\!\left(
\tau_A=\tau_B,\;
\one\{s_A(X,Y)\le\tau_A\}\ne
\one\{s_B(X,Y)\le\tau_B\}
\right)
\le \min\{1,\delta+2f_{\max}\eta_s\}.
\end{equation}
If $\mathcal Y$ is finite and, for every $y\in\mathcal Y$, the cdf of
$s_A(X,y)$ under $X\sim P_{A,X}$ is absolutely continuous on $I$ with
density bounded above by $L_y$ almost everywhere, then on the same
high-probability event and whenever
$\tau_A=\tau_B=\tau$,
\begin{equation}\label{eq:robust-full-set}
\Pd_{X\sim P_{A,X}}\{\cC^{s_A}_{\tau}(X)\ne
\cC^{s_B}_{\tau}(X)\}
\le \min\left\{1,2\eta_s\sum_{y\in\mathcal Y}L_y\right\}.
\end{equation}
The condition $\eta_s=0$ is sufficient, but not necessary, for exact
equality of the two set maps at a common threshold.  The exact condition
is
$\one\{s_A(x,y)\le\tau\}=\one\{s_B(x,y)\le\tau\}$
for every $(x,y)$.

\subsection*{Proof of Corollary~\ref{cor:robust}}
Let $\ttau_A$ be the $k_n$th order statistic of the $s_A$-scores in
analyst $A$.  On analyst $B$'s sample, let $\ttau_B$ be the
$k_n$th order statistic of the $s_B$-scores and let $\ttau_B^{A}$ be the
$k_n$th order statistic obtained by applying $s_A$ to that same sample.
An order statistic is $1$-Lipschitz in the $\ell_\infty$ norm of the
sample vector, so \eqref{eq:robust-score} gives
\begin{equation}\label{eq:robust-order-lipschitz}
|\ttau_B-\ttau_B^{A}|\le\eta_s.
\end{equation}
By continuity and the probability integral transform,
$U:=F(\ttau_A)$ and $V:=G(\ttau_B^{A})$
are independent copies of
$\operatorname{Beta}(k_n,n+1-k_n)$.  Define
$E:=\{|U-p|<\lambda,\ |V-p|<\lambda\}$.
For either transformed order statistic, DKW and
$|k_n/n-p|<2/n$ give
\[
\Pd(|U-p|\ge\lambda)
\le2\exp\{-2n(\lambda-2/n)^2\}
\le2\exp(-n\lambda^2/2),
\]
where the last inequality uses $n\ge4/\lambda$; the same bound holds for
$V$.  Thus
\begin{equation}\label{eq:robust-localization}
\Pd(E^c)\le4\exp(-n\lambda^2/2).
\end{equation}
On $E$,
$|F(\ttau_A)-p|<\lambda\le f_{\min}\Delta$, so
$\ttau_A\in I$.  Also,
\[
|F(\ttau_B^{A})-p|
\le|G(\ttau_B^{A})-p|+\|F-G\|_\infty
<\lambda+\eta_P=f_{\min}\Delta,
\]
so $\ttau_B^{A}\in I$.  The inverse-Lipschitz inequality for $F$ gives
\begin{align}
|\ttau_A-\ttau_B^{A}|
&\le
\frac{|F(\ttau_A)-F(\ttau_B^{A})|}{f_{\min}}
\le\frac{|U-V|+\eta_P}{f_{\min}}.
\label{eq:robust-raw-distance}
\end{align}
Combining \eqref{eq:robust-order-lipschitz} and
\eqref{eq:robust-raw-distance} yields, on $E$,
\[
|\ttau_A-\ttau_B|
\le\frac{|U-V|}{f_{\min}}+\frac{\eta_P}{f_{\min}}+\eta_s.
\]
Conditioning on the raw thresholds and applying the offset-collision
identity therefore gives
\[
\Pd(\tau_A\ne\tau_B)
\le
\frac{\E|U-V|}{f_{\min}\beta}
+\frac{\eta_s+\eta_P/f_{\min}}{\beta}
+4e^{-n\lambda^2/2}.
\]
Finally,
$\E|U-V|\le\sqrt{2\operatorname{Var}(U)}\le b_n$
by \eqref{eq:beta-variance-bound}, proving
\eqref{eq:robust-mismatch}.  For completeness, continuity of analyst
$B$'s own $s_B$-score cdf is not needed for marginal validity.  Append
independent continuous tie breakers to the $n+1$ exchangeable scores and
rank the resulting pairs lexicographically.  The test pair has a uniform
rank.  If its score is strictly larger than the $k_n$th calibration
score, then its tie-broken rank is larger than $k_n$.  Hence
$\Pd\{S_{n+1}\le S_{(k_n)}\}\ge k_n/(n+1)\ge1-\alpha$.
Upward rounding can only enlarge the prediction set.  Applying this
argument to each analyst's own score and population proves both
marginal-coverage claims.

On the analyst-$A$ DKW event of probability at least $1-\delta$,
\[
|F(\ttau_A)-p|\le e_n(\delta),
\qquad
|\ttau_A-q|\le\frac{e_n(\delta)}{f_{\min}}.
\]
Since $\ttau_A\le\tau_A<\ttau_A+\beta$,
condition \eqref{eq:robust-agreement-condition} implies
\begin{equation}\label{eq:robust-random-interval}
[\tau_A-\eta_s,\tau_A+\eta_s]\subseteq I.
\end{equation}
If $\tau_A=\tau_B=\tau$ and the two membership indicators differ, then
$s_A(X,Y)$ and $s_B(X,Y)$ straddle $\tau$.  By
\eqref{eq:robust-score}, this forces
$|s_A(X,Y)-\tau|\le\eta_s$.  The upper density bound on the random
interval in \eqref{eq:robust-random-interval} therefore gives
\[
P_A\{|s_A(X,Y)-\tau|\le\eta_s\}
\le2f_{\max}\eta_s,
\]
which proves \eqref{eq:robust-pair-agreement}.  Integrating over the
exceptional DKW event proves
\eqref{eq:robust-pair-unconditional}.  Finally, a union bound over labels,
followed by the corresponding per-label density bound, proves
\eqref{eq:robust-full-set}. \qed

\paragraph{Full statement of Proposition~\ref{prop:pilot}.}
Suppose the standing continuous-score setup and
Assumption~\ref{ass:margin} hold.  Let the public pilot
$\mathcal D_p$ consist of $n_p=2m$ i.i.d.\ scores with distribution
$F$, independent of all calibration samples, test observations, and
shared random seeds.  Fix
\[
0<h\le\frac{\Delta}{2},
\qquad
c,\delta_p,\rho\in(0,1),
\qquad
k_p:=\left\lceil(1-\alpha)(m+1)\right\rceil\le m .
\]
Split $\mathcal D_p$ into two samples of size $m$.  Let
$\widehat q_p$ be the $k_p$-th order statistic of the first half and
define
\[
W:=[\widehat q_p-h,\widehat q_p+h],
\qquad
N_W:=\#\{S_i\text{ in the second pilot half}:S_i\in W\},
\]
and the safeguarded estimator
\[
\widehat f_{\mathrm L}
:=
\frac{\max\{N_W,1\}}{(1+c)\,2hm}.
\]
For $j\ge1$ and $\eta\in(0,1)$ write
$e_j(\eta):=\sqrt{\log(2/\eta)/(2j)}+2/j$.
Assume
\begin{equation}\label{eq:pilot-size}
e_m(\delta_p/3)
\le \frac{f_{\min}\Delta}{2},
\qquad
m\ge
\frac{3\log(3/\delta_p)}{2c^2hf_{\min}}.
\end{equation}
Then there is an event $\mathcal G_p\in\sigma(\mathcal D_p)$ with
$\Pd(\mathcal G_p)\ge1-\delta_p$ such that, on $\mathcal G_p$,
\begin{equation}\label{eq:pilot-density-band}
W\subseteq I,
\qquad
N_W\ge1,
\qquad
\frac{1-c}{1+c}f_{\min}
\le \widehat f_{\mathrm L}
\le f_{\max}.
\end{equation}
Let $\widehat\kappa\ge\kappa=f_{\max}/f_{\min}$ be a
pre-specified deterministic protocol input and set
\[
B_n
:=
\sqrt{\frac{2\alpha(1-\alpha)}{n}}+\frac{\sqrt2}{n},
\qquad
\beta
:=
\frac{2\widehat\kappa B_n}{\widehat f_{\mathrm L}\rho}.
\]
Let $V\sim\operatorname{Unif}[0,1)$ be shared by the analysts and
independent of the pilot, calibration samples, and test observations,
and set $u=\beta V$.  On $\mathcal G_p$,
\begin{equation}\label{eq:pilot-beta-band}
\frac{2B_n}{f_{\min}\rho}
\le \beta
\le \overline\beta_n
:=
\frac{2(1+c)}{1-c}\,
\frac{\widehat\kappa B_n}{f_{\min}\rho}.
\end{equation}
Let $k_n=\lceil(1-\alpha)(n+1)\rceil$ and suppose additionally that
\begin{equation}\label{eq:pilot-replication-conditions}
k_n\le n,
\qquad
n\ge\frac{4}{f_{\min}\Delta},
\qquad
\overline\beta_n\le\frac{\Delta}{2},
\qquad
4\exp\!\left(-\frac{nf_{\min}^2\Delta^2}{2}\right)
\le\frac{\rho}{2}.
\end{equation}
For every pilot realization in $\mathcal G_p$, two analysts using
that same public pilot and the same shared $V$, but independent
calibration samples $\mathcal D,\mathcal D'$ of size $n$, satisfy
\begin{equation}\label{eq:pilot-conditional-replication}
\Pd_{\mathcal D,\mathcal D',V}
\!\left[
\tau(\mathcal D)\ne\tau(\mathcal D')
\,\middle|\,\mathcal D_p
\right]
\le\rho.
\end{equation}
The probability in \eqref{eq:pilot-conditional-replication} averages
over the shared random offset $V$; no bound conditional on a fixed
realization of $V$ is asserted.  By Theorem~\ref{thm:upper}(iii), the
same upper bound applies to nonidentity of the two deployed classifier
maps.  If the random pilot is also included in the probability space,
then
\begin{equation}\label{eq:pilot-unconditional-replication}
\Pd(\tau\ne\tau')
\le
\rho+(1-\rho)\delta_p
\le\rho+\delta_p.
\end{equation}
In particular, to guarantee an unconditional target
$\rho_\star\in(0,1)$, choose $0<\delta_p<\rho_\star$ and replace
$\rho$ throughout the construction and conditions above by
$\rho_0:=(\rho_\star-\delta_p)/(1-\delta_p)$;
then the right-hand side of
\eqref{eq:pilot-unconditional-replication} equals $\rho_\star$.
For every realization of $\mathcal D_p$ and $V$, including pilot
realizations outside $\mathcal G_p$, upward rounding retains the
conditional-on-protocol marginal conformal guarantee
\begin{equation}\label{eq:pilot-marginal-coverage}
\Pd_{\mathcal D,(X_{n+1},Y_{n+1})}
\!\left[
Y_{n+1}\in\cC_\tau(X_{n+1})
\,\middle|\,\mathcal D_p,V
\right]
\ge1-\alpha.
\end{equation}
Moreover, for any $\delta\in(0,1)$ satisfying
\begin{equation}\label{eq:pilot-coverage-condition}
e_n(\delta)
\le f_{\min}\bigl(\Delta-\overline\beta_n\bigr),
\end{equation}
for every pilot realization in $\mathcal G_p$ and every realization
of $V$, with probability at least $1-\delta$ over the calibration
sample,
\begin{equation}\label{eq:pilot-coverage-band}
1-\alpha-e_n(\delta)
\le F(\tau)
\le
1-\alpha+e_n(\delta)
+
\frac{2(1+c)}{1-c}\,
\widehat\kappa\kappa\,\frac{B_n}{\rho}.
\end{equation}
Thus the final term in \eqref{eq:pilot-coverage-band} is the
rounding-induced contribution to the upper endpoint; the total upper
excess also contains $e_n(\delta)$.  Jointly over the pilot, shared
offset, and calibration sample, \eqref{eq:pilot-coverage-band} holds
with probability at least $1-\delta_p-\delta$.

\subsection*{Proof of Proposition~\ref{prop:pilot}}
Put $r:=1-\alpha$ and let $\widehat F_1$ be the empirical
distribution function of the first pilot half.  Continuity of $F$
implies that the first-half observations are distinct almost surely,
and hence
$\widehat F_1(\widehat q_p)=k_p/m$.
Since $k_p=\lceil r(m+1)\rceil$ and $r\in(0,1)$,
\begin{equation}\label{eq:pilot-rank-error}
0<\frac{k_p}{m}-r<\frac{2}{m}.
\end{equation}
Let $t_m:=\sqrt{\log(6/\delta_p)/(2m)}$
and define
$E_0:=\{\|\widehat F_1-F\|_\infty\le t_m\}$.
The Dvoretzky--Kiefer--Wolfowitz inequality gives
$\Pd(E_0)\ge1-\delta_p/3$.  On $E_0$,
\begin{align}
\left|F(\widehat q_p)-r\right|
&\le
\|\widehat F_1-F\|_\infty
+\left|\frac{k_p}{m}-r\right| \notag\\
&<
t_m+\frac{2}{m}
=e_m(\delta_p/3)
\le\frac{f_{\min}\Delta}{2}.
\label{eq:pilot-localization}
\end{align}
We claim that \eqref{eq:pilot-localization} implies
$|\widehat q_p-q|\le\Delta/2$.  Indeed, if
$\widehat q_p>q+\Delta/2$, then monotonicity and the lower density
bound on $I$ give
\[
F(\widehat q_p)
\ge F(q+\Delta/2)
=F(q)+\int_q^{q+\Delta/2}f(t)\,dt
\ge r+\frac{f_{\min}\Delta}{2},
\]
contradicting \eqref{eq:pilot-localization}.  The case
$\widehat q_p<q-\Delta/2$ is symmetric.  Because
$h\le\Delta/2$, it follows that, on $E_0$,
$W=[\widehat q_p-h,\widehat q_p+h]\subseteq[q-\Delta,q+\Delta]=I$.
Therefore, conditionally on the first pilot half,
\[
p_W
:=\Pd(S\in W\mid\widehat q_p)
=F(\widehat q_p+h)-F(\widehat q_p-h)
=\int_{\widehat q_p-h}^{\widehat q_p+h}f(t)\,dt
\]
satisfies
\begin{equation}\label{eq:pilot-window-mass}
2hf_{\min}\le p_W\le2hf_{\max}.
\end{equation}
The two pilot halves are independent, so, conditionally on the first
half, $N_W\sim\operatorname{Bin}(m,p_W)$.
For $c\in(0,1)$, the multiplicative Chernoff bounds yield
\begin{align*}
\Pd\!\left[
N_W\ge(1+c)mp_W
\,\middle|\,\widehat q_p
\right]
&\le
\exp\!\left(-\frac{c^2mp_W}{3}\right),\\
\Pd\!\left[
N_W\le(1-c)mp_W
\,\middle|\,\widehat q_p
\right]
&\le
\exp\!\left(-\frac{c^2mp_W}{2}\right).
\end{align*}
On $E_0$, \eqref{eq:pilot-window-mass} and
\eqref{eq:pilot-size} imply
$mp_W\ge2mhf_{\min}\ge3\log(3/\delta_p)/c^2$.
Consequently, after conditioning on the first half, each of the two
tail probabilities on $E_0$ is at most $\delta_p/3$.  Thus the event
\[
\mathcal G_p
:=
E_0
\cap
\{(1-c)mp_W<N_W<(1+c)mp_W\}
\]
has probability at least $1-\delta_p$.  On $\mathcal G_p$,
$N_W>(1-c)mp_W>0$; since $N_W$ is integer, $N_W\ge1$, and the
safeguard $\max\{N_W,1\}$ is inactive.  Dividing the preceding strict
inequalities by $(1+c)2hm$ and applying
\eqref{eq:pilot-window-mass} gives
\[
\frac{1-c}{1+c}f_{\min}
\le
\widehat f_{\mathrm L}
\le
f_{\max}.
\]
This proves \eqref{eq:pilot-density-band}.  The safeguard merely makes
the protocol well defined on $\mathcal G_p^c$.

Fix a pilot realization in $\mathcal G_p$.  Since
$\widehat\kappa\ge\kappa=f_{\max}/f_{\min}$,
\[
\beta
=\frac{2\widehat\kappa B_n}{\widehat f_{\mathrm L}\rho}
\ge
\frac{2\widehat\kappa B_n}{f_{\max}\rho}
\ge
\frac{2B_n}{f_{\min}\rho}.
\]
The lower bound on $\widehat f_{\mathrm L}$ similarly gives
$\beta\le\tfrac{2(1+c)}{1-c}\widehat\kappa B_n/(f_{\min}\rho)=\overline\beta_n$.
This proves \eqref{eq:pilot-beta-band}.

Conditional on the public pilot, $\beta$ is fixed, while
$u=\beta V$ is uniform on $[0,\beta)$ and independent of the two
calibration samples.  Conditions
\eqref{eq:pilot-replication-conditions} therefore permit application
of Theorem~\ref{thm:upper}(i).  The probability in the following
display is over the two calibration samples and the shared $V$:
\begin{align*}
\Pd_{\mathcal D,\mathcal D',V}
\!\left[
\tau(\mathcal D)\ne\tau(\mathcal D')
\,\middle|\,\mathcal D_p
\right]
&\le
\frac{B_n}{f_{\min}\beta}
+4\exp\!\left(-\frac{nf_{\min}^2\Delta^2}{2}\right)
\le
\frac{\rho}{2}+\frac{\rho}{2}
=\rho.
\end{align*}
Theorem~\ref{thm:upper}(iii) then gives the corresponding classifier-map
claim.  On $\mathcal G_p^c$ the mismatch probability is at most one,
so averaging over the pilot gives
\begin{align*}
\Pd(\tau\ne\tau')
&\le
\Pd(\mathcal G_p)\rho+\Pd(\mathcal G_p^c)
=
\rho+(1-\rho)\Pd(\mathcal G_p^c)
\le
\rho+(1-\rho)\delta_p.
\end{align*}
If $\rho$ in the construction is replaced by
$\rho_0=(\rho_\star-\delta_p)/(1-\delta_p)$, the last expression is
exactly $\rho_\star$.

For marginal validity, fix arbitrary realizations of
$\mathcal D_p$ and $V$.  Then $(\beta,u)$ is fixed independently of
the calibration sample and test observation, and upward rounding gives
$\tau\ge\ttau_n$ deterministically.  By exchangeability and
continuity of the $n+1$ scores,
\begin{align*}
\Pd_{\mathcal D,(X_{n+1},Y_{n+1})}
\!\left[
Y_{n+1}\in\cC_\tau(X_{n+1})
\,\middle|\,\mathcal D_p,V
\right]
&\ge
\Pd[S_{n+1}\le S_{(k_n)}]
=\frac{k_n}{n+1}
\ge1-\alpha.
\end{align*}
This proves \eqref{eq:pilot-marginal-coverage} for every pilot and
shared-offset realization, including those outside $\mathcal G_p$.

Finally, fix a pilot realization in $\mathcal G_p$ and a realization
of $V$.  By \eqref{eq:pilot-beta-band}, $\beta\le\overline\beta_n$;
hence \eqref{eq:pilot-coverage-condition} implies
\[
e_n(\delta)
\le
f_{\min}(\Delta-\overline\beta_n)
\le
f_{\min}(\Delta-\beta).
\]
Theorem~\ref{thm:upper}(ii) therefore gives, with probability at least
$1-\delta$ over the calibration sample,
\[
1-\alpha-e_n(\delta)
\le F(\tau)
\le
1-\alpha+e_n(\delta)+f_{\max}\beta.
\]
Using $\beta\le\overline\beta_n$ and
$f_{\max}/f_{\min}=\kappa$,
\[
f_{\max}\beta
\le
f_{\max}\overline\beta_n
=
\frac{2(1+c)}{1-c}\,
\widehat\kappa\kappa\,\frac{B_n}{\rho}.
\]
This proves \eqref{eq:pilot-coverage-band}.  A union bound over pilot
and calibration failure gives the asserted joint probability
$1-\delta_p-\delta$.
\qed

\subsection*{Proof of Proposition~\ref{prop:adaptive} (asymptotically exact plug-in)}
Write $\ttau,\ttau'$ for the two order statistics, $\Delta=\ttau-\ttau'$, and $k=\lceil(1-\alpha)(n+1)\rceil$. Since $f$ is continuous and positive at $q$, the central order statistic satisfies $\sqrt n(\ttau-q)\Rightarrow\mathcal N(0,\alpha(1-\alpha)/f(q)^{2})$ \citep[Cor.~21.5]{vaart1998}; by independence, $\sqrt n\,\Delta\Rightarrow\mathcal N(0,2\alpha(1-\alpha)/f(q)^{2})$.

\emph{(i)} Conditionally on $(\ttau,\ttau',\beta_n)$, the shared offset $u\sim\mathrm{Unif}[0,\beta_n)$ places a grid point in $(\ttau\wedge\ttau',\ttau\vee\ttau']$ with probability exactly $\min(|\Delta|/\beta_n,1)$, and the outputs differ iff this happens. Now
$|\Delta|/\beta_n=\big(\sqrt n f(q)|\Delta|/\sqrt{2\alpha(1-\alpha)}\big)\cdot\sqrt{2\alpha(1-\alpha)}\big/\big(\sqrt n f(q)\beta_n\big)$,
the first factor converges to $|Z|$, and $\sqrt n f(q)\beta_n=2\sqrt nB_nf(q)/(\hat f_n\rho)\to_p2\sqrt{2\alpha(1-\alpha)}/\rho$ since $\sqrt nB_n\to\sqrt{2\alpha(1-\alpha)}$ and $\hat f_n\to_pf(q)$. By Slutsky, $|\Delta|/\beta_n\Rightarrow\rho|Z|/2$; since $\min(\cdot,1)$ is bounded and continuous, $\Pd[\tau\neq\tau']=\E\min(|\Delta|/\beta_n,1)\to\E\min(\rho|Z|/2,1)\le(\rho/2)\E|Z|=\rho\sqrt{2/\pi}/2$. The variant width $\beta_n=\sqrt{2/\pi}B_n/(\hat f_n\rho)$ rescales the limit to $\E\min\!\big(\sqrt{\pi/2}\,\rho|Z|,1\big)=\rho(1+o(1))$ as $\rho\to0$.

\emph{(ii)} Given $(\ttau,\beta_n)$, the up-rounding excess $\tau-\ttau=\beta_n\lceil(\ttau-u)/\beta_n\rceil+u-\ttau$ is exactly $\mathrm{Unif}[0,\beta_n)$ in $u$, independent of $(\ttau,\beta_n)$; write it $U\beta_n$ with $U\sim\mathrm{Unif}[0,1)$. Fix $\varepsilon'>0$ and $\eta>0$ with $|f(t)-f(q)|\le\varepsilon'$ on $[q-\eta,q+\eta]$. On $A_n=\{|\ttau-q|\le\eta/2,\ \beta_n\le\eta/2\}$, $|F(\ttau+U\beta_n)-F(\ttau)-f(q)U\beta_n|\le\varepsilon'U\beta_n$; moreover $\Pd(A_n^{c})\to0$, $F$-increments are at most $1$, and $\beta_n\le2B_n/(f_0\rho)$ deterministically, while $\E[1/\hat f_n]\to1/f(q)$ by bounded convergence ($1/\hat f_n\le1/f_0$). Hence $\E[F(\tau)-F(\ttau)]=f(q)\E[U]\,\E[\beta_n](1+O(\varepsilon'))+o(B_n)=(1+o(1))B_n/\rho$ after letting $\varepsilon'\downarrow0$, and $\E F(\ttau)=k/(n+1)$ exactly.

\emph{(iii)} On $A_n\cap\{\hat f_n\ge f(q)-\varepsilon'\}$, whose probability tends to one, $F(\tau)-F(\ttau)\le(f(q)+\varepsilon')\beta_n\le\tfrac{f(q)+\varepsilon'}{f(q)-\varepsilon'}\cdot\tfrac{2B_n}{\rho}=(2+O(\varepsilon'))B_n/\rho$, and $\Pd[F(\ttau)\le k/(n+1)+e_n(\delta)]\ge1-\delta$ by the exact Beta law as in Theorem~\ref{thm:upper}(ii). A union bound gives (iii). \hfill$\square$

\section{Proof of Theorem~\ref{thm:lower}}\label{app:lower}

\paragraph{Scope of the scalar formulation.}
The theorem concerns exact replication of the scalar output index. It
applies to a fixed, sample-independent sublevel family
$\cC_t(x)=\{y:s(x,y)\le t\}$
when the calibrator outputs the canonical index \(t\), replicability is
measured by equality of that index, and the admissible induced score
laws include the hard family used below. For \(S=s(X,Y)\),
\[
\Pd\{Y\in\cC_t(X)\}
=
\Pd\{S\le t\}
=
F_P(t).
\]
Equal indices imply equal set maps, but the converse need not hold.
Thus the theorem does not, without an additional injectivity or
canonicalization assumption, give a lower bound when replicability is
defined only by equality of set maps. Nor is an arbitrary
sample-dependent or refitted nested family automatically covered.

\paragraph{Proof mechanism.}
The proof embeds a binomial statistic into the calibration problem.
It directly bounds the derivative of the seed-conditional acceptance
probability and integrates the replicability constraint across the
hard family. No external replicable-coin theorem is needed, and the
result has no logarithmic loss.

\begin{proof}
Put
\[
p:=1-\alpha,
\qquad
m_-:=p-2\varepsilon,
\qquad
m_+:=p+2\varepsilon,
\qquad
M:=[m_-,m_+],
\]
and $\gamma:=m_+-m_-=4\varepsilon$.
Since $2\varepsilon\le\frac12\min\{\alpha,p\}$,
we have
$m_-\ge p/2>0$ and $m_+\le1-\alpha/2<1$,
so \(M\subset(0,1)\).

\paragraph{C.1: The hard family and its margin properties.}
For \(m\in M\), let \(P_m\) have the following version of its
Lebesgue density on \([0,1]\):
\[
f_m(s)
=
2m\,\one\!\left\{0\le s<\frac12\right\}
+
2(1-m)\,\one\!\left\{\frac12\le s\le1\right\}.
\]
Its CDF is
\[
F_m(t)
=
\begin{cases}
0,
& t<0,\\[1mm]
2mt,
& 0\le t\le\tfrac12,\\[1mm]
m+2(1-m)(t-\tfrac12),
& \tfrac12\le t\le1,\\[1mm]
1,
& t>1.
\end{cases}
\]
Thus \(F_m\) is continuous and strictly increasing on \([0,1]\).
In particular, \(P_m\) is atomless even though its density may jump
at \(1/2\).
The \(p\)-quantile is
\[
q_m
=
\begin{cases}
\dfrac{p}{2m},
& m\ge p,\\[3mm]
\dfrac12+\dfrac{p-m}{2(1-m)},
& m\le p.
\end{cases}
\]
If \(m\ge p\), then
\(m\le p+2\varepsilon\le\frac32p\),
so \(q_m\ge1/3\). If \(m\le p\), write \(d=p-m\). Since
\(0\le d\le2\varepsilon\le\alpha/2\),
we have
\[
q_m
=
\frac12+\frac{d}{2(\alpha+d)}
\le
\frac23.
\]
The opposite bounds are immediate from the two cases. Hence
\(q_m\in[\frac13,\frac23]\)
for every \(m\in M\),
and therefore
\([q_m-\tfrac14,q_m+\tfrac14]\subset[0,1]\).

Moreover,
\(2m\ge2m_-\ge p\) and \(2(1-m)\ge2(1-m_+)\ge\alpha\),
and both possible density values are at most \(2\). Consequently, the
displayed version of \(f_m\) lies in
\([\min\{\alpha,p\},2]\)
throughout
\([q_m-\tfrac14,q_m+\tfrac14]\).
This proves the common margin claim for the entire family.

At the upper endpoint,
\[
q_{m_+}
=
\frac{p}{2m_+}
<
\frac12,
\qquad
\frac12-q_{m_+}
=
\frac{\varepsilon}{m_+}
>
\varepsilon.
\]
At the lower endpoint,
\[
q_{m_-}
=
\frac12+\frac{\varepsilon}{1-m_-},
\qquad
q_{m_-}-\frac12
=
\frac{\varepsilon}{1-m_-}
>
\varepsilon.
\]
Because \(\varepsilon\le1/8\) and both endpoint quantiles belong to
\([1/3,2/3]\),
\[
[q_{m_+}-\varepsilon,q_{m_+}+\varepsilon]
\subset(0,\tfrac12),
\qquad
[q_{m_-}-\varepsilon,q_{m_-}+\varepsilon]
\subset(\tfrac12,1).
\]
Consequently, \(P_{m_+}\) satisfies
Assumption~\ref{ass:margin} with
\(\Delta=\varepsilon\), \(f_{\min}=f_{\max}=2m_+\),
while \(P_{m_-}\) satisfies it with
\(\Delta=\varepsilon\), \(f_{\min}=f_{\max}=2(1-m_-)\).
Thus each endpoint separately has \(\kappa=1\), and the two constant
density values both belong to
\([2\min\{\alpha,1-\alpha\},2]\).

\paragraph{C.2: Accuracy forces a binary decision.}
Define
\[
B(\cD;r)
:=
\one\!\left\{
\cA(\cD;r)<\frac12
\right\}.
\]
Under \(P_{m_+}\), every \(t\ge1/2\) satisfies
\[
F_{m_+}(t)
\ge
m_+
=
p+2\varepsilon
>
p+\varepsilon.
\]
Thus an accurate output must be below \(1/2\), and hence
\[
\Pd_{\substack{r\sim\nu\\
\cD\sim P_{m_+}^{\otimes n}}}
[B(\cD;r)=1]
\ge1-\delta.
\]
Under \(P_{m_-}\), every \(t<1/2\) satisfies
\[
F_{m_-}(t)
\le
m_-
=
p-2\varepsilon
<
p-\varepsilon.
\]
Thus an accurate output must be at least \(1/2\), and hence
\[
\Pd_{\substack{r\sim\nu\\
\cD\sim P_{m_-}^{\otimes n}}}
[B(\cD;r)=1]
\le\delta.
\]
Finally,
\[
\{B(\cD;r)\ne B(\cD';r)\}
\subseteq
\{\cA(\cD;r)\ne\cA(\cD';r)\}.
\]
Therefore \(B\) is also \(\rho\)-replicable at every
\(P_m\), \(m\in M\).

\paragraph{C.3: Binomial reduction and derivative bound.}
Let
\[
N
:=
\sum_{i=1}^n
\one\!\left\{S_i<\frac12\right\}.
\]
Under \(P_m^{\otimes n}\),
\(N\sim\operatorname{Bin}(n,m)\),
and the joint density of the sample, with respect to Lebesgue measure,
is
\(2^n m^N(1-m)^{n-N}\)
almost everywhere.

For \(k\in\{0,\ldots,n\}\), let \(Q_k\) be the law obtained by first
choosing a \(k\)-subset \(J\) of \(\{1,\ldots,n\}\) uniformly and
then, independently, drawing
\[
S_i\sim\operatorname{Unif}[0,\tfrac12)
\quad(i\in J),
\qquad
S_i\sim\operatorname{Unif}[\tfrac12,1]
\quad(i\notin J).
\]
For every \(m\in M\), the conditional law of \(\cD\) given \(N=k\)
is \(Q_k\), which does not depend on \(m\).

For fixed \(r\), define
\[
\psi_r(k)
:=
\int_{[0,1]^n}
\one\!\left\{B(\mathbf s;r)=1\right\}
\,Q_k(\mathrm d\mathbf s)
\in[0,1]
\]
and
\[
h_r(m)
:=
\Pd_{\cD\sim P_m^{\otimes n}}
[B(\cD;r)=1].
\]
Conditioning on \(N\) gives
\[
h_r(m)
=
\sum_{k=0}^n
\psi_r(k)\binom{n}{k}m^k(1-m)^{n-k}.
\]
Joint measurability of \(\cA\) implies that
\(r\mapsto\psi_r(k)\) is measurable. Hence
\((r,m)\mapsto h_r(m)\) is jointly measurable. The endpoint
conclusions give
\[
\E_r h_r(m_-)\le\delta,
\qquad
\E_r h_r(m_+)\ge1-\delta.
\]

For each fixed \(r\), \(h_r\) is a polynomial. In what follows,
\(\E_m\), \(\operatorname{Cov}_m\), and
\(\operatorname{Var}_m\) refer to
\(N\sim\operatorname{Bin}(n,m)\). Differentiating the binomial
probabilities gives
\[
h_r'(m)
=
\E_m\!\left[
\psi_r(N)\frac{N-nm}{m(1-m)}
\right]
=
\frac{\operatorname{Cov}_m(\psi_r(N),N)}
     {m(1-m)}.
\]
By Cauchy--Schwarz,
\(\operatorname{Var}_m(N)=nm(1-m)\),
and, since \(0\le\psi_r(N)\le1\),
\(\operatorname{Var}_m(\psi_r(N))\le h_r(m)(1-h_r(m))\).
Consequently,
\begin{equation}\label{eq:lower-acceptance-derivative}
|h_r'(m)|
\le
\frac{
\sqrt{
\operatorname{Var}_m(\psi_r(N))
\operatorname{Var}_m(N)
}
}{
m(1-m)
}
\le
\sqrt{
\frac{
n\,h_r(m)(1-h_r(m))
}{
m(1-m)
}
}.
\end{equation}

\paragraph{C.4: A wide transition for every good seed.}
Let
\[
G
:=
\left\{
r:
h_r(m_-)\le\frac14
\ \text{and}\
h_r(m_+)\ge\frac34
\right\}.
\]
Markov's inequality and a union bound give
\[
\Pd_r(G^c)
\le
\Pd_r\!\left(h_r(m_-)>\frac14\right)
+
\Pd_r\!\left(1-h_r(m_+)>\frac14\right)
\le8\delta.
\]
Therefore
\begin{equation}\label{eq:lower-good-seeds}
\Pd_r(G)\ge1-8\delta.
\end{equation}
Set
\(v_*:=\inf_{m\in M}m(1-m)>0\).
Fix \(r\in G\). By continuity and the intermediate value theorem, the
following level sets are nonempty and closed in compact intervals, so
the extrema exist:
\[
b_r
:=
\min\left\{
m\in M:h_r(m)=\frac34
\right\},
\qquad
a_r
:=
\max\left\{
m\in[m_-,b_r]:h_r(m)=\frac14
\right\}.
\]
Minimality of \(b_r\), maximality of \(a_r\), and continuity imply
\[
h_r(m)\in\left[\frac14,\frac34\right]
\qquad
\text{for every }m\in[a_r,b_r].
\]
Indeed, a value above \(3/4\) before \(b_r\), or below \(1/4\) after
\(a_r\), would create an earlier \(3/4\) crossing or a later \(1/4\)
crossing by the intermediate value theorem. No monotonicity is being
assumed.
On \([a_r,b_r]\), equation
\eqref{eq:lower-acceptance-derivative} and
\(h_r(m)(1-h_r(m))\le1/4\) imply
\[
|h_r'(m)|
\le
\frac12\sqrt{\frac n{v_*}}.
\]
It follows that
\[
\frac12
=
h_r(b_r)-h_r(a_r)
\le
\int_{a_r}^{b_r}|h_r'(m)|\,\mathrm dm
\le
\frac12\sqrt{\frac n{v_*}}(b_r-a_r).
\]
Therefore
\begin{equation}\label{eq:lower-transition-width}
b_r-a_r
\ge
\sqrt{\frac{v_*}{n}}.
\end{equation}

\paragraph{C.5: Integrating replicability.}
Conditional on a fixed shared seed \(r\), the two decisions based on
independent samples are independent Bernoulli random variables with
success probability \(h_r(m)\). Their conditional disagreement
probability is therefore
\(2h_r(m)(1-h_r(m))\).
Since decision disagreement implies threshold disagreement,
\[
\E_r\!\left[
2h_r(m)(1-h_r(m))
\right]
\le\rho
\qquad
\text{for every }m\in M.
\]
The jointly measurable, nonnegative function
\[
I(r)
:=
\int_M
2h_r(m)(1-h_r(m))\,\mathrm dm
\]
therefore satisfies, by Tonelli's theorem,
\begin{equation}\label{eq:lower-integrated-replicability}
\E_r I(r)
\le
\rho\,|M|
=
\rho\gamma.
\end{equation}
For every \(r\in G\), equation
\eqref{eq:lower-transition-width} gives
\[
I(r)
\ge
\int_{a_r}^{b_r}\frac38\,\mathrm dm
\ge
\frac38\sqrt{\frac{v_*}{n}}.
\]
Combining this pointwise bound with
\eqref{eq:lower-good-seeds} and
\eqref{eq:lower-integrated-replicability}, without requiring
measurability of \(a_r\) or \(b_r\), gives
\[
\rho\gamma
\ge
(1-8\delta)\frac38
\sqrt{\frac{v_*}{n}}.
\]
Consequently,
\[
n
\ge
\frac{
9(1-8\delta)^2v_*
}{
64\rho^2\gamma^2
}.
\]
For every \(m\in M\),
\(m\ge(1-\alpha)/2\) and \(1-m\ge\alpha/2\).
Hence
\(v_*\ge\alpha(1-\alpha)/4\).
Since \(\gamma=4\varepsilon\),
\[
n
\ge
\frac{9(1-8\delta)^2}{4096}
\frac{\alpha(1-\alpha)}
     {\varepsilon^2\rho^2}.
\]
Finally, \(\delta\le1/16\) implies \(1-8\delta\ge1/2\), and therefore
\[
n
\ge
\frac9{16384}
\frac{\alpha(1-\alpha)}
     {\varepsilon^2\rho^2}.
\]
\end{proof}

\paragraph{C.6: Limitation of the hard-family reduction.}
The preceding hard family alone cannot yield the same lower bound for
outputs that may be arbitrary measurable subsets of the score space
instead of sublevel sets from the fixed nested family. Indeed, the
fixed set
\[
A^\star
=
\left[0,\frac{1-\alpha}{2}\right]
\cup
\left[
\frac12,
\frac12+\frac{1-\alpha}{2}
\right]
\]
satisfies
\[
P_m(A^\star)
=
m(1-\alpha)+(1-m)(1-\alpha)
=
1-\alpha
\]
for every \(m\in M\). Thus a deterministic, perfectly replicable
procedure can attain exact coverage on this particular family if it
is allowed such outputs. This observation concerns the reach of this
reduction; it is not a universal positive result for unrestricted set
outputs.
There is no vanishing-margin-width issue in the construction. The
whole family has the common parameters established in C.1; the
convenient displayed choices correspond to
\(\kappa_{\mathrm{common}}=2/\min\{\alpha,1-\alpha\}\).
At \(m=1-\alpha\), the sharp ratio between the two density values is
\(\max\{\alpha,1-\alpha\}/\min\{\alpha,1-\alpha\}\).
Hence the proof establishes the stated
\(\alpha(1-\alpha)/(\varepsilon^2\rho^2)\)
lower bound, but does not by itself show that the \(\kappa^2\) factor
in the upper bound is minimax necessary.

\subsection*{Proof of Corollary~\ref{cor:sizelb} (set-size cost)}
By a union bound, for every atomless $P$, $\Pd[F_P(\cA(\cD;r))\in[1-\alpha,\,1-\alpha+e]]\ge1-2\delta$. Set $\alpha':=\alpha-e/2$ and $\varepsilon':=e/2$, so the band is $[1-\alpha'-\varepsilon',\,1-\alpha'+\varepsilon']$ and $\cA$ satisfies the hypotheses of Theorem~\ref{thm:lower} at $(\alpha',\varepsilon',\rho,\delta'{=}2\delta)$ once the side conditions hold. They do: $\delta'=2\delta\le\tfrac1{16}$; and since $e\le\min(\alpha,1-\alpha)/4$ we have $\min(\alpha',1-\alpha')\ge\min(\alpha,1-\alpha)-e/2\ge\tfrac{7}{8}\min(\alpha,1-\alpha)\ge\tfrac{7}{2}e>4\varepsilon'$, so $\varepsilon'\le\min(\alpha',1-\alpha')/4$. Theorem~\ref{thm:lower} then gives $n\ge\tfrac{9(1-8\delta')^{2}}{4096}\,\alpha'(1-\alpha')/(\varepsilon'^{2}\rho^{2})$, i.e.\ $e=2\varepsilon'\ge\tfrac{3(1-16\delta)}{32}\sqrt{\alpha'(1-\alpha')/n}/\rho$. For the simplified form: if $\alpha\le\tfrac12$ then $\alpha'\ge\tfrac{7}{8}\alpha$ and $1-\alpha'\ge1-\alpha$, so $\alpha'(1-\alpha')\ge\tfrac{7}{8}\alpha(1-\alpha)$; if $\alpha>\tfrac12$, then $\alpha'$ moves toward $\tfrac12$, so $\alpha'(1-\alpha')\ge\alpha(1-\alpha)$ directly once $\alpha'\ge\tfrac12$, and otherwise $\alpha'\in(\tfrac12-\tfrac{e}{2},\tfrac12]$ gives $\alpha'(1-\alpha')\ge\tfrac14-\tfrac{e^{2}}{16}\ge\tfrac78\cdot\tfrac14\ge\tfrac78\alpha(1-\alpha)$. With $\delta\le\tfrac1{32}$, $\tfrac{3(1-16\delta)}{32}\sqrt{7/8}\ge\tfrac{3}{2\cdot32}\cdot0.935\ge\tfrac{3}{70}$. \hfill$\square$

\emph{Scope.} The corollary inherits the scope of Theorem~\ref{thm:lower}: threshold replicability over all atomless laws, with a shared seed. The two-sided band is supplied by the validity hypothesis on one side and the inflation cap on the other; neither alone suffices, and the oracle rule $\tau\equiv F_P^{-1}(1-\alpha)$ shows the worst-case quantifier cannot be dropped.

\subsection*{Proof of Lemma~\ref{lem:churn} (churn localization)}
For a threshold rule, $\cC_\tau(x)=\{y:s(x,y)\le\tau\}$ is monotone in $\tau$, so for any $\tau,\tau'$ the symmetric difference is $\{y:\tau\wedge\tau'<s(x,y)\le\tau\vee\tau'\}$, which is the stated identity. If both outputs lie on the grid $u+\beta\mathbb{Z}$ and $\tau\neq\tau'$, then $(\tau\wedge\tau',\tau\vee\tau']$ is a union of $|\tau-\tau'|/\beta\ge1$ consecutive full cells, so the count is at least the count over one full cell contained in it. Taking $\E_X$ and using that, on the $1-\delta$ event of Theorem~\ref{thm:upper}(ii), both thresholds lie in the stated window, the conditional bound follows with the minimum over cells meeting the window. \hfill$\square$

\section{Proofs for Section~\ref{sec:twolist}}\label{app:twolist}

\paragraph{Full statement of Theorem~\ref{thm:twolist}.}
Fix
$\alpha\in(0,1)$, $\delta\in(0,1)$,
$0<\varepsilon<\min\{\alpha,1-\alpha\}$.
Let $\Delta>0$ and $0<f_{\min}\le f_{\max}<\infty$ be fixed,
public constants.  Let $\mathfrak F$ be the nonempty class of continuous
score distributions satisfying Assumption~\ref{ass:margin} with these
same constants, and set
$\kappa:=f_{\max}/f_{\min}$,
$\beta:=\varepsilon/(2f_{\max})$.
Assume $\beta\le\Delta/2$.

\emph{(a)} Fix $F\in\mathfrak F$, write
$q:=F^{-1}(1-\alpha)$, and work in the standing split-conformal setup
$\cC_t(x):=\{y:s(x,y)\le t\}$,
$S_i:=s(X_i,Y_i)$,
where $S_1,\ldots,S_{n+1}$ are i.i.d.\ with distribution function $F$;
$S_1,\ldots,S_n$ are the calibration scores and $S_{n+1}$ is the test
score.  Suppose
$n\ge\lceil32\kappa^2\log(2/\delta)/\varepsilon^2\rceil$.
For
\[
k_n:=\left\lceil(1-\alpha)(n+1)\right\rceil,
\qquad
\ttau:=S_{(k_n)},
\qquad
\tau:=\beta\left\lceil\frac{\ttau}{\beta}\right\rceil,
\]
there is a sample-independent list of two adjacent grid thresholds
\[
\cL_{n,\delta}(F)
:=\{g_{F,n,\delta},g_{F,n,\delta}+\beta\}
\subset\beta\mathbb Z,
\]
depending only on $F$ and the public protocol parameters, such that
\[
\Pd_F\!\left(
\tau\in\cL_{n,\delta}(F),\ 
1-\alpha-\frac{\varepsilon}{4\kappa}
\le F(\tau)\le
1-\alpha+\frac{\varepsilon}{2}
+\frac{\varepsilon}{4\kappa}
\right)
\ge1-\delta,
\]
where $\Pd_F$ in this display is over the calibration sample.
Consequently, the deterministic-grid version of \ReCal{} is
$(2,\delta)$-list replicable as a numerical-threshold calibrator and,
on the same event,
$F(\tau)\in[1-\alpha-\varepsilon,\,1-\alpha+\varepsilon]$.
Its classifier-valued output belongs to the list
$\{\cC_{g_{F,n,\delta}},\cC_{g_{F,n,\delta}+\beta}\}$,
whose cardinality is at most two.  For any $m\ge1$ analysts using
size-$n$ calibration samples from the same $F$, all $m$ thresholds
belong to this common list with probability at least
$\max\{0,1-m\delta\}$.
Upward rounding also preserves marginal conformal validity:
$\Pd\{Y_{n+1}\in\cC_\tau(X_{n+1})\}\ge1-\alpha$.

\emph{(b)} Suppose additionally that $\delta<1/2$.  Then list size two
is optimal among numerical-threshold lists over the same fixed class
$\mathfrak F$.  More precisely, for every integer $N\ge1$, there is no
single measurable, possibly randomized threshold procedure
$\cA:\R^N\times\mathsf R\longrightarrow\R$
with a fixed seed law $R\sim\nu$, independent of the data distribution,
that satisfies both of the following properties for every
$F\in\mathfrak F$:
\begin{enumerate}
\item there is a deterministic $v_F\in\R$, which may depend on $F$ but
      not on the sample or on $R$, such that
      \[
      \Pd_{\cD\sim F^{\otimes N},\,R\sim\nu}
      \{\cA(\cD;R)=v_F\}\ge1-\delta;
      \]
\item
      \[
      \Pd_{\cD\sim F^{\otimes N},\,R\sim\nu}\!\left\{
      F\bigl(\cA(\cD;R)\bigr)
      \in[1-\alpha-\varepsilon,\,1-\alpha+\varepsilon]
      \right\}\ge1-\delta.
      \]
\end{enumerate}
In fact, impossibility already holds on the translated-uniform
subfamily
\[
\left\{
Q_t:=\operatorname{Unif}\!\left[t,t+f_{\min}^{-1}\right]
:t\in\R
\right\}\subseteq\mathfrak F,
\]
whose local density is constant and whose intrinsic density ratio is
one.

\subsection{Proof of Theorem~\ref{thm:twolist}}

\begin{proof}
\emph{Part (a).}
Put
\[
L_\delta:=\log\frac{2}{\delta},
\qquad
t_n:=\sqrt{\frac{L_\delta}{2n}},
\qquad
e_n(\delta):=t_n+\frac2n,
\qquad
w:=\frac{e_n(\delta)}{f_{\min}}.
\]
The assumed lower bound on $n$ gives
$t_n\le\varepsilon/(8\kappa)$.
Also, $\kappa\ge1$, $L_\delta>\log 2>1/2$, and
$\varepsilon<1/2$, so $2\kappa L_\delta>\varepsilon$.  Therefore the
same sample-size bound implies $n\ge16\kappa/\varepsilon$, and hence
$2/n\le\varepsilon/(8\kappa)$.
Consequently,
\begin{equation}
e_n(\delta)
\le\frac{\varepsilon}{4\kappa}
=\frac{f_{\min}\beta}{2},
\qquad
w\le\frac\beta2.
\label{eq:twolist-en}
\end{equation}
In particular, $2w\le\beta$.  Moreover,
$n\ge16\kappa/\varepsilon$ and $\varepsilon<\alpha$ imply
\[
\frac1{n+1}<\frac1n
\le\frac{\varepsilon}{16\kappa}<\alpha,
\]
so $k_n\le n$ and $S_{(k_n)}$ is well defined.

Let $\widehat F_n$ be the empirical distribution function and define
the Dvoretzky--Kiefer--Wolfowitz event
$E:=\{\|\widehat F_n-F\|_\infty\le t_n\}$.
The DKW inequality gives $\Pd_F(E)\ge1-\delta$.  Since $F$ is
continuous, the calibration scores are almost surely distinct and
$\widehat F_n(\ttau)=k_n/n$.  The ceiling definition gives
$1-\alpha\le k_n/n\le1-\alpha+2/n$.
Thus, on $E$,
\begin{equation}
\bigl|F(\ttau)-(1-\alpha)\bigr|
\le e_n(\delta).
\label{eq:twolist-dkw}
\end{equation}

Continuity and the positive local density imply $F(q)=1-\alpha$.
By \eqref{eq:twolist-en} and $\beta\le\Delta/2$,
\[
e_n(\delta)
\le\frac{f_{\min}\beta}{2}
\le\frac{f_{\min}\Delta}{4}
<f_{\min}\Delta.
\]
If $\ttau<q-\Delta$ or $\ttau>q+\Delta$, monotonicity of $F$ and
integration of the lower density bound on the corresponding side of
$q$ would give
$|F(\ttau)-F(q)|\ge f_{\min}\Delta$, contradicting
\eqref{eq:twolist-dkw}.  Hence $\ttau\in[q-\Delta,q+\Delta]$ on $E$.
Integrating the same lower bound between $q$ and $\ttau$ then yields
\begin{equation}
|\ttau-q|
\le\frac{|F(\ttau)-F(q)|}{f_{\min}}
\le w.
\label{eq:twolist-localization}
\end{equation}

Let $J:=[q-w,q+w]$ and define
\[
g_{F,n,\delta}
:=\beta\left\lceil\frac{q-w}{\beta}\right\rceil,
\qquad
R_\beta(t):=\beta\left\lceil\frac{t}{\beta}\right\rceil.
\]
The interval $J$ has length $2w\le\beta$.  For every $t\in J$,
monotonicity of $R_\beta$ gives
$R_\beta(t)\ge R_\beta(q-w)=g_{F,n,\delta}$.
Furthermore,
\[
t\le q+w=(q-w)+2w
\le(q-w)+\beta\le g_{F,n,\delta}+\beta.
\]
Since $g_{F,n,\delta}+\beta$ is a grid point,
$R_\beta(t)\le g_{F,n,\delta}+\beta$.  Therefore
\begin{equation}
R_\beta(t)
\in\{g_{F,n,\delta},g_{F,n,\delta}+\beta\}
\qquad(t\in J).
\label{eq:twolist-grid}
\end{equation}
This argument includes all endpoint cases, including $2w=\beta$ and
endpoints of $J$ lying on the grid.  Equations
\eqref{eq:twolist-localization} and \eqref{eq:twolist-grid} prove the
list assertion on $E$.

It remains to prove the displayed accuracy band on the same event.
From $|\ttau-q|\le w\le\beta/2$ and $\beta\le\Delta/2$,
\[
\ttau\ge q-\frac\beta2\ge q-\Delta,
\qquad
\ttau+\beta\le q+\frac{3\beta}{2}
\le q+\frac{3\Delta}{4}<q+\Delta.
\]
Thus $[\ttau,\ttau+\beta]$ lies in the interval on which the upper
density bound holds.  Since
$\ttau\le\tau<\ttau+\beta$,
\[
F(\ttau)\le F(\tau)
\le F(\ttau)+f_{\max}\beta.
\]
Using \eqref{eq:twolist-en}, \eqref{eq:twolist-dkw}, and
$f_{\max}\beta=\varepsilon/2$ gives
\[
1-\alpha-\frac{\varepsilon}{4\kappa}
\le F(\tau)
\le1-\alpha+\frac{\varepsilon}{2}
+\frac{\varepsilon}{4\kappa}.
\]
Because $\kappa\ge1$, this interval is contained in
$[1-\alpha-\varepsilon,1-\alpha+\varepsilon]$.  Since
$\Pd_F(E)\ge1-\delta$, the joint list-and-accuracy assertion follows.

For $m$ analysts, let $E_j$ be the corresponding DKW event.  The
same deterministic list works for every analyst, and the union bound
gives
\[
\Pd\!\left(\bigcap_{j=1}^m
\{\tau_j\in\cL_{n,\delta}(F)\}\right)
\ge1-\sum_{j=1}^m\Pd(E_j^c)
\ge1-m\delta.
\]
Combining this inequality with nonnegativity gives the asserted bound.
Independence is not needed for the union-bound argument.

Finally, $\tau\ge\ttau$ and the prediction sets are nested in the
threshold.  Exchangeability and continuity give
\[
\Pd\{Y_{n+1}\in\cC_{\ttau}(X_{n+1})\}
=\Pd\{S_{n+1}\le S_{(k_n)}\}
=\frac{k_n}{n+1}
\ge1-\alpha.
\]
Therefore upward rounding preserves the marginal guarantee.

\emph{Part (b).}
Choose any $F_0\in\mathfrak F$ and let
$q_0:=F_0^{-1}(1-\alpha)$.  Integrating the lower density bound on the
two sides of $q_0$ gives
\[
f_{\min}\Delta
\le F_0(q_0)-F_0(q_0-\Delta)
\le1-\alpha,
\qquad
f_{\min}\Delta
\le F_0(q_0+\Delta)-F_0(q_0)
\le\alpha.
\]
Hence
\begin{equation}
f_{\min}\Delta\le\min\{\alpha,1-\alpha\}.
\label{eq:twolist-compatibility}
\end{equation}
Set $c:=f_{\min}^{-1}$ and
$Q_t:=\operatorname{Unif}[t,t+c]$.  Its $(1-\alpha)$-quantile is
$q_t=t+c(1-\alpha)$,
whose distances from the left and right support endpoints are
$c(1-\alpha)$ and $c\alpha$.  By
\eqref{eq:twolist-compatibility}, both distances are at least $\Delta$.
The density on $[q_t-\Delta,q_t+\Delta]$ is the constant
$1/c=f_{\min}\le f_{\max}$.  Thus $Q_t\in\mathfrak F$ for every
$t\in\R$.  Lemma~\ref{lem:translate} below rules out an unconditional
singleton threshold list with the required accuracy band even on this
subfamily, proving part~\emph{(b)}.
\end{proof}

\begin{lemma}[No unconditional singleton threshold list under translations]
\label{lem:translate}
Fix an integer $N\ge1$, $\alpha\in(0,1)$,
$0<\varepsilon<\min\{\alpha,1-\alpha\}$,
$0<\delta<1/2$, and $c>0$.  Let
$Q_t:=\operatorname{Unif}[t,t+c]$, $t\in\R$,
with distribution function $F_t$.  Let $\nu$ be a fixed probability
law on a measurable seed space $\mathsf R$.  There is no single
measurable procedure $\cA:\R^N\times\mathsf R\to\R$ for which, for
every $t\in\R$, both
\[
\text{there exists }v_t\in\R\text{ such that }
\Pd_t\{\cA(\cD;R)=v_t\}\ge1-\delta
\]
and
\[
\Pd_t\!\left\{
F_t\bigl(\cA(\cD;R)\bigr)
\in[1-\alpha-\varepsilon,1-\alpha+\varepsilon]
\right\}\ge1-\delta
\]
hold, where
$\Pd_t:=Q_t^{\otimes N}\otimes\nu$ is the joint law of
$\cD$ and an independent seed $R$.
\end{lemma}

\begin{proof}
Suppose, for a contradiction, that such a procedure exists, and let
$\mu_t$ be the law of $\cA(\cD;R)$ under $\Pd_t$.  For all
$s,t\in\R$,
\[
\dtv(Q_t,Q_s)
=\min\left\{\frac{|t-s|}{c},1\right\}.
\]
Tensorization, invariance under adjoining the common independent seed
law, and data processing give
\begin{equation}
\dtv(\mu_t,\mu_s)
\le\dtv(Q_t^{\otimes N}\otimes\nu,
        Q_s^{\otimes N}\otimes\nu)
=\dtv(Q_t^{\otimes N},Q_s^{\otimes N})
\le N\,\dtv(Q_t,Q_s).
\label{eq:translate-tv}
\end{equation}
Define
$\eta:=c(1-2\delta)/(2N)>0$.
If $|t-s|<\eta$, then
$\dtv(\mu_t,\mu_s)<(1-2\delta)/2$.  Therefore
\[
\mu_s(\{v_t\})
\ge\mu_t(\{v_t\})-\dtv(\mu_t,\mu_s)
>1-\delta-\frac{1-2\delta}{2}
=\frac12.
\]
Also $\mu_s(\{v_s\})\ge1-\delta>1/2$.  If $v_t\ne v_s$, these two
disjoint atoms would have total $\mu_s$-mass greater than one.  Hence
\begin{equation}
|t-s|<\eta\quad\Longrightarrow\quad v_t=v_s.
\label{eq:translate-local-constancy}
\end{equation}
For arbitrary $s,t\in\R$, divide the finite interval joining them into
finitely many subintervals of length less than $\eta$.  Iterating
\eqref{eq:translate-local-constancy} shows that $v_s=v_t$.  Thus there
is one $v^\ast\in\R$ such that $v_t=v^\ast$ for every $t\in\R$.

Fix $t$.  The singleton-list event
$E_t:=\{\cA(\cD;R)=v^\ast\}$
and the accuracy event
\[
G_t:=\left\{
F_t\bigl(\cA(\cD;R)\bigr)
\in[1-\alpha-\varepsilon,1-\alpha+\varepsilon]
\right\}
\]
each have $\Pd_t$-probability at least $1-\delta$.  Hence
$\Pd_t(E_t\cap G_t)\ge1-2\delta>0$.  On this intersection the
deterministic number $F_t(v^\ast)$ lies in the accuracy band, so
\begin{equation}
F_t(v^\ast)
\in[1-\alpha-\varepsilon,1-\alpha+\varepsilon]
\qquad\text{for every }t\in\R.
\label{eq:translate-all-t}
\end{equation}
But
\[
F_t(v^\ast)=
\begin{cases}
0, & v^\ast\le t,\\[1mm]
(v^\ast-t)/c, & t<v^\ast<t+c,\\[1mm]
1, & v^\ast\ge t+c.
\end{cases}
\]
Taking $t=v^\ast$ gives $F_t(v^\ast)=0$, whereas taking
$t=v^\ast-c$ gives $F_t(v^\ast)=1$.  The assumed inequality
$\varepsilon<\min\{\alpha,1-\alpha\}$ excludes both values from the
band in \eqref{eq:translate-all-t}, a contradiction.
\end{proof}

\begin{remark}[Threshold lists versus classifier-map lists]
The optimality assertion in Theorem~\ref{thm:twolist}(b) is directly
about numerical thresholds.  It also transfers to extensional equality
of classifier maps whenever the translated-uniform score family above
is admitted by the application and
\[
t_1<t_2
\quad\Longrightarrow\quad
\text{there exists $(x,y)$ with }t_1<s(x,y)\le t_2.
\]
This condition makes $t\mapsto\cC_t$ injective; a dense attainable score
range on the admissible threshold domain is sufficient.
\end{remark}

\paragraph{Full-conformal 1-NN convention.}
Let
\(
\cD=((X_1,Y_1),\ldots,(X_n,Y_n))
\)
and, for a candidate \((x,y)\), write
\[
(\widetilde X_i,\widetilde Y_i)=(X_i,Y_i),
\quad 1\le i\le n,
\qquad
(\widetilde X_{n+1},\widetilde Y_{n+1})=(x,y).
\]
For each \(i\in\{1,\ldots,n+1\}\), let
\[
\nu_i^{x,y}(\cD)
\in
\operatorname*{arg\,min}_{\substack{1\le j\le n+1\\j\ne i}}
|\widetilde X_i-\widetilde X_j|
\]
be chosen using a fixed measurable rule, common to all runs, for which
the resulting score vector is permutation equivariant. For example,
among the distance minimizers one may select the lexicographically
smallest value \((\widetilde X_j,\widetilde Y_j)\), and then use the
smallest index if that same pair occurs more than once. The final index
tie-break need not itself be permutation equivariant, but all indices
remaining at that stage have the same pair value; hence the induced
score vector is permutation equivariant. Define the refitted
leave-one-out 1-nearest-neighbor scores
\[
R_i^{x,y}(\cD)
:=
\bigl|
\widetilde Y_i-
\widetilde Y_{\nu_i^{x,y}(\cD)}
\bigr|.
\]
Set
\[
c_{\cD}(x,y)
:=
\sum_{i=1}^{n}
\one\!\left\{
R_i^{x,y}(\cD)\ge R_{n+1}^{x,y}(\cD)
\right\},
\qquad
\widehat p_{\cD}(x,y)
:=
\frac{1+c_{\cD}(x,y)}{n+1},
\]
and, for \(m\in\{0,1,\ldots,n+1\}\), define the rank-cutoff map
\[
\widehat\cC_{\cD}^{(m)}(x)
:=
\{y\in\R:c_{\cD}(x,y)\ge m\}.
\]
Thus
\(\widehat\cC_{\cD}^{(0)}\equiv\R\) and
\(\widehat\cC_{\cD}^{(n+1)}\equiv\varnothing\).

\paragraph{Full statement of Proposition~\ref{prop:fullconf}.}
Let \(n\ge2\), let \(P\) be a Borel probability law on \(\R^2\) such
that \(P\ll\mathrm{Leb}^2\), and let \(\cD,\cD'\) be independent
\(P^{\otimes n}\) samples. All events involving extensional equality
of set-valued maps are understood in the completion of the relevant
joint law; equivalently, each zero-probability assertion below may be
read as an outer-probability assertion.
Then, for every \(m\in\{1,\ldots,n\}\),
\[
\Pd\!\left[
\widehat\cC_{\cD}^{(m)}
\equiv
\widehat\cC_{\cD'}^{(m)}
\right]=0,
\]
where \(\equiv\) means pointwise equality as set-valued maps on all of
\(\R\).
More generally, let \(U\) be a shared random seed independent of
\((\cD,\cD')\), and let
\(M=M(U)\in\{0,1,\ldots,n+1\}\)
be measurable. If both runs use the same cutoff \(M(U)\), then
\[
\Pd\!\left[
\widehat\cC_{\cD}^{(M)}
\equiv
\widehat\cC_{\cD'}^{(M)}
\right]
=
\Pd\!\left[M\in\{0,n+1\}\right].
\]
In particular, if
\(\Pd[M\in\{1,\ldots,n\}]=1\),
then the exact-agreement probability is zero.
For the usual convention
\[
\widehat\cC_{\cD,\alpha}(x)
:=
\{y\in\R:\widehat p_{\cD}(x,y)>\alpha\},
\qquad 0<\alpha<1,
\]
one has
\[
\widehat\cC_{\cD,\alpha}
=
\widehat\cC_{\cD}^{(m_\alpha)},
\qquad
m_\alpha:=\lfloor\alpha(n+1)\rfloor.
\]
Consequently,
\[
\Pd\!\left[
\widehat\cC_{\cD,\alpha}
\equiv
\widehat\cC_{\cD',\alpha}
\right]=0
\quad\text{for every}\quad
\alpha\in\left[\frac1{n+1},1\right),
\]
whereas \(0<\alpha<1/(n+1)\) gives the trivial always-\(\R\) map.
No condition on whether \(\alpha(n+1)\) is an integer is needed.

\begin{remark}[Scope of Proposition~\ref{prop:fullconf}]
\label{rem:fullconf-scope}
At the level of thresholded acceptance decisions, the proposition
covers precisely every sample-independent transformation of the
attainable \(p\)-value grid that is nondecreasing for each realization
of a shared seed and is followed by a fixed strict threshold.
Indeed, let \(U\) be the independent shared seed from the proposition,
fix \(a\in\R\), and set
\[
\mathcal P_n
:=
\left\{
\frac{1+c}{n+1}:c=0,\ldots,n
\right\}.
\]
For each seed value \(u\), let both runs use the same nondecreasing map
\(q_u:\mathcal P_n\to\R\),
where \(u\mapsto q_u(v)\) is measurable for every
\(v\in\mathcal P_n\). If the rounded rule accepts when
\(q_U(\widehat p_{\cD}(x,y))>a\),
define
\[
M(U)
:=
\min\left\{
c\in\{0,\ldots,n\}:
q_U\!\left(\frac{1+c}{n+1}\right)>a
\right\},
\qquad
\min\varnothing:=n+1.
\]
Because \(\mathcal P_n\) is finite and its coordinate maps are
measurable, \(M(U)\) is measurable. Monotonicity gives, pointwise,
\[
\{y:q_U(\widehat p_{\cD}(x,y))>a\}
=
\widehat\cC_{\cD}^{(M(U))}(x).
\]
Thus its exact-agreement probability is
\(\Pd[M(U)\in\{0,n+1\}]\),
and it is zero if the induced cutoff is nondegenerate almost surely.
Conversely, every cutoff in \(\{0,\ldots,n+1\}\) is representable in
this form; for example, use
\[
q\!\left(\frac{1+c}{n+1}\right)
=
\one\{c\ge m\},
\qquad
a=\frac12.
\]
Degenerate output-collapsing transformations correspond to the two
endpoint cutoffs and are the exceptions to the zero-agreement
conclusion. For example, mapping every \(p\)-value to \(1\) and taking
\(a<1\) gives the exactly replicable always-\(\R\) map. The result
does not cover run-specific data-dependent cutoffs, candidate-wise
randomization, seed-dependent changes to the score map, or arbitrary
postprocessing. Accordingly, it shows only that cutoff-level
stabilization is insufficient; it does not prove that equality of the
entire fitted training maps is logically necessary.
Finally, the conclusion concerns pointwise equality on the whole
domain \(\R\). Because the proof uses a sufficiently far-right input,
it does not by itself rule out equality merely \(P_X\)-almost surely
or only on \(\operatorname{supp}(P_X)\).
\end{remark}

\subsection{Proof of Proposition~\ref{prop:fullconf}}

\begin{proof}
Because \(P\ll\mathrm{Leb}^2\), both of its one-dimensional marginals
are absolutely continuous and
\(P^{\otimes n}\ll\mathrm{Leb}^{2n}\).
Hence, with probability one, the \(X_i\)'s are pairwise distinct, every
\(X_i\) has a unique nearest neighbor among the other training inputs,
and the \(Y_i\)'s are pairwise distinct. Indeed, the exceptional events
are contained in finite unions of affine hyperplanes of the forms
\[
X_i=X_j,
\qquad
2X_i=X_j+X_k,
\qquad
Y_i=Y_j.
\]
Here the first and third forms use distinct pairs of indices, while the
midpoint form uses distinct \(i,j,k\). Let \(\mathcal G\) denote this
probability-one event. The stipulated tie rule defines the prediction
map everywhere; on \(\mathcal G\), it is irrelevant for the
within-sample neighbors used below.

Fix \(\cD\in\mathcal G\), and define
\[
J=J(\cD)
:=
\operatorname*{arg\,max}_{1\le i\le n}X_i,
\qquad
N_i=N_i(\cD)
:=
\operatorname*{arg\,min}_{\substack{1\le j\le n\\j\ne i}}
|X_i-X_j|,
\]
\[
d_i:=|X_i-X_{N_i}|,
\qquad
T(\cD):=\max_{1\le i\le n}(X_i+d_i).
\]
If \(x>T(\cD)\), then \(x>X_i\) for every \(i\), because \(d_i>0\).
The candidate's unique nearest neighbor is therefore \(J\). Moreover,
for every training index \(i\),
\(|x-X_i|=x-X_i>d_i\),
so the candidate is not a nearest neighbor of any training
observation. Consequently, for every \(y\in\R\), augmentation by
\((x,y)\) leaves all training nearest neighbors unchanged and
\[
R_i^{x,y}(\cD)
=
r_i(\cD)
:=
|Y_i-Y_{N_i}|,
\quad 1\le i\le n,
\qquad
R_{n+1}^{x,y}(\cD)=|y-Y_J|.
\]

Write the training residuals in decreasing order, with multiplicities,
as
\(r_{[1]}(\cD)\ge\cdots\ge r_{[n]}(\cD)\),
and fix \(m\in\{1,\ldots,n\}\). For every \(a\ge0\),
\[
\#\{i:r_i(\cD)\ge a\}\ge m
\quad\Longleftrightarrow\quad
a\le r_{[m]}(\cD).
\]
It follows that, for every \(x>T(\cD)\),
\[
\widehat\cC_{\cD}^{(m)}(x)
=
\bigl[
Y_J-r_{[m]}(\cD),
\,Y_J+r_{[m]}(\cD)
\bigr].
\]
Define the eventual far-right upper endpoint
\(B_m(\cD):=Y_J+r_{[m]}(\cD)\).
For completeness, set \(B_m(\cD):=0\) on \(\mathcal G^c\). This
defines a Borel random variable: \(\mathcal G\) is Borel, the unique
finite argmax and argmin maps are Borel on \(\mathcal G\), and absolute
values and finite order statistics are Borel measurable.

We claim that \(B_m(\cD)\) is atomless. Fix \(t\in\R\). On
\(\mathcal G\), the order statistic \(r_{[m]}\) equals \(r_i\) for at
least one \(i\), and \(Y_i\ne Y_{N_i}\). Thus, for some
\(i\in\{1,\ldots,n\}\) and \(\sigma\in\{-1,1\}\),
\(r_{[m]}=\sigma(Y_i-Y_{N_i})\).
Therefore
\[
\{B_m=t\}\cap\mathcal G
\subseteq
\bigcup_{\substack{1\le j,i,k\le n,\ i\ne k\\
                    \sigma\in\{-1,1\}}}
\{Y_j+\sigma(Y_i-Y_k)=t\}.
\]
Each set in this finite union is a proper affine hyperplane. Indeed,
the coefficient vector of \(Y_j+\sigma(Y_i-Y_k)\) is
\(e_j+\sigma(e_i-e_k)\),
whose coordinate sum is \(1\), so it cannot be the zero vector.
Absolute continuity of \(P^{\otimes n}\) makes every such hyperplane
null. Since \(\Pd[\mathcal G^c]=0\),
\[
\Pd[B_m(\cD)=t]
\le
\Pd[\mathcal G^c]
+
\sum_{\substack{1\le j,i,k\le n,\ i\ne k\\
                 \sigma\in\{-1,1\}}}
\Pd\!\left[
Y_j+\sigma(Y_i-Y_k)=t
\right]
=0
\]
for every \(t\in\R\), proving atomlessness.
This argument does not require the residuals to be distinct. In
particular, if \(N_i=k\) and \(N_k=i\), then the structural duplicate
\(r_i=r_k=|Y_i-Y_k|\)
is handled without modification.

Now let \(\cD,\cD'\) be independent. Then
\(B_m(\cD)\) and \(B_m(\cD')\) are independent and atomless, whence
\(\Pd[B_m(\cD)=B_m(\cD')]=0\).
On the probability-one event
\(\{\cD\in\mathcal G,\ \cD'\in\mathcal G\}\),
take the common input
\(x_*:=1+\max\{T(\cD),T(\cD')\}\).
If
\(\widehat\cC_{\cD}^{(m)}\equiv\widehat\cC_{\cD'}^{(m)}\),
then their interval values at \(x_*\) are equal, so their upper
endpoints are equal. Writing
\(E_m:=\{\widehat\cC_{\cD}^{(m)}\equiv\widehat\cC_{\cD'}^{(m)}\}\),
we therefore have
\[
E_m
\subseteq
\{B_m(\cD)=B_m(\cD')\}
\cup
\{\cD\notin\mathcal G\}
\cup
\{\cD'\notin\mathcal G\}.
\]
The right-hand side is a measurable null set. Hence \(E_m\) has outer
probability zero and belongs to the completed product sigma-field,
where it has probability zero. This proves the fixed-\(m\) assertion
without requiring a separate function-space sigma-field.

For the shared-seed claim, write
\[
E_M
:=
\left\{
\widehat\cC_{\cD}^{(M)}
\equiv
\widehat\cC_{\cD'}^{(M)}
\right\},
\qquad
A:=\{M\in\{0,n+1\}\}.
\]
The two endpoint maps agree identically, so \(A\subseteq E_M\).
Moreover,
\[
E_M\setminus A
\subseteq
\bigcup_{m=1}^{n}
\bigl(\{M=m\}\cap E_m\bigr)
\subseteq
\bigcup_{m=1}^{n}E_m.
\]
After lifting the \(E_m\)'s to the joint space carrying
\((U,\cD,\cD')\), the last union is a finite null set. Thus \(E_M\)
differs from \(A\) only by a subset of a measurable null set, and hence,
in the completed joint law,
\[
\Pd[E_M]
=
\Pd[A]
=
\Pd[M\in\{0,n+1\}].
\]
This argument avoids any need to factor probabilities of
completion-measurable map-identity events.

Finally, since \(c_{\cD}(x,y)\) is integer-valued,
\[
\widehat p_{\cD}(x,y)>\alpha
\quad\Longleftrightarrow\quad
c_{\cD}(x,y)>\alpha(n+1)-1
\quad\Longleftrightarrow\quad
c_{\cD}(x,y)\ge\lfloor\alpha(n+1)\rfloor.
\]
This equivalence holds whether or not \(\alpha(n+1)\) is an integer.
For
\(\alpha\in[1/(n+1),1)\),
the cutoff belongs to \(\{1,\ldots,n\}\), so the zero-agreement result
applies. For \(0<\alpha<1/(n+1)\), the cutoff is zero and every
candidate is accepted, yielding the always-\(\R\) map. This completes
the proof.
\end{proof}

\begin{remark}[Sharpness in the sample size]
The condition \(n\ge2\) is sharp for this construction. If \(n=1\),
the sole training point and the candidate are each other's nearest
neighbors, their two scores are equal, and
\(\widehat p_{\cD}(x,y)=1\)
for every \((x,y)\).
Hence the usual rule is the always-\(\R\) map for every
\(0<\alpha<1\).
\end{remark}

\section{The distribution-free fallback}\label{app:distfree}

\begin{remark}[\ReCal$^{+}$, distribution-free fallback, extensions]\label{rem:extensions}
Calibrating at rank $k^{+}=\lceil(1-\alpha+e_n(\delta))(n+1)\rceil$ gives $F(\tau)\ge1-\alpha$ with probability $1-\delta$ whenever $k^{+}\le n$ (training-conditional validity; \citealp{vovk2012,bian2023}). Without Assumption~\ref{ass:margin}, replicable binary search over a $b$-bit score grid retains all guarantees at $n=\tOh(\log^{2}(R)\log(1/\delta)/(\varepsilon^{2}\rho^{2}))$ \citep[Thm.~B.6]{kalavasis2024} (Proposition~\ref{prop:distfree}, Appendix~\ref{app:distfree}): the margin improves constants, not feasibility. Every method calibrating one scalar over a nested family \citep{gupta2022} (APS, RAPS, conformal risk control, conformal-LM stopping) admits \ReCal{} unchanged; Section~\ref{sec:exp} instantiates \ReCal-APS, \ReCal-RAPS, and a risk-controlled top-$p$ variant. Composing with the DP-to-replicability conversion of \citet{bun2023} from private prediction sets \citep{angelopoulos2022private} would cost a quadratic sample overhead and still require shared randomness; the direct construction is preferable.
\end{remark}

\begin{proposition}[Distribution-free \ReCal]\label{prop:distfree}
Suppose scores take values on a known grid of $R$ points (e.g.\ $b$-bit floats, $R=2^{b}$). There is a shared-seed calibrator that is $\rho$-replicable over \emph{all} score distributions on the grid and, with probability $\ge1-\delta$, outputs $\tau$ with $F(\tau)\ge1-\alpha-\varepsilon$ and $F(\tau^{-})\le1-\alpha+\varepsilon$, using $n=\tOh\big(\log^{2}(R)\log(1/\delta)/(\varepsilon^{2}\rho^{2})\big)$ samples and no margin assumption.
\end{proposition}

\emph{Proof sketch}, after \citet{ilps2022} (replicable statistical queries) and \citet[Thm.~B.6]{kalavasis2024} (replicable quantiles over a discretized range). Maintain a binary search over the sorted grid $v_1<\dots<v_R$. Each run draws one calibration sample and controls $\|\widehat F-F\|_\infty\le\varepsilon\rho/(8\log_2R)$ by the DKW inequality, which holds simultaneously at every grid point so that no union bound over the data-dependent search path is needed, at cost $n=O\big(\log_2^{2}(R)\log(1/\delta)/(\varepsilon^{2}\rho^{2})\big)$ (DKW failure probabilities folded into $\rho$ and $\delta$). At round $t$ of the $\lceil\log_2R\rceil$ rounds, decide ``$F(v)\ge1-\alpha$?'' by comparing $\widehat F(v)$ to $1-\alpha+u_t$, where $u_t\sim\mathrm{Unif}[-\varepsilon/2,\varepsilon/2]$ is drawn from the shared seed. On the two runs' joint DKW event, the empirical values at any queried $v$ differ by at most $\varepsilon\rho/(4\log_2R)$, so the round-$t$ comparisons disagree only if $u_t$ falls in an interval of that length: probability at most $\rho/(4\log_2R)$, hence all rounds agree, and the search paths and outputs are identical, except with probability $\le\rho/4$ on that event. The search invariant ($F$ below target to the left of the bracket, above target minus $\varepsilon$ to the right) yields the two-sided display on the $1-\delta$ event. \qed

\section{Constants: derivations for Remark~\ref{rem:constants}}\label{app:constants}

All three rows analyze calibrated rounding to the \emph{same} target $(\alpha,\varepsilon,\rho,\delta)=(0.1,0.02,0.1,0.05)$ with $\kappa=1$.

\paragraph{(a) Exact-Beta analysis (Corollary~\ref{cor:protocol}).} $n\ge\max\{32\kappa^{2}\alpha(1-\alpha)/(\varepsilon^{2}\rho^{2}),\,8\log(2/\delta)/\varepsilon^{2}\}$: the replicability term is $32\cdot0.09/(4\times10^{-4}\cdot10^{-2})=7.2\times10^{5}$; the coverage term is $8\ln(40)/(4\times10^{-4})\approx7.4\times10^{4}$; hence $n\approx7.2\times10^{5}$.

\paragraph{(b) DKW/union analysis of the same scheme.} Control each run's quantile in probability instead of in expectation: by DKW at level $\rho/4$ per run, with probability $\ge1-\rho/2$ both runs satisfy $|F(\ttau)-(1-\alpha)|\le t+2/n$, $t=\sqrt{\ln(8/\rho)/(2n)}$, whence $|\ttau-\ttau'|\le2(t+2/n)/f_{\min}$ and the offset-collision probability adds $\le2(t+2/n)/(f_{\min}\beta)$. With $\beta=\varepsilon/(2f_{\max})$, requiring $2t/(f_{\min}\beta)=4\kappa t/\varepsilon\le\rho/2$ forces $t\le\varepsilon\rho/(8\kappa)$, i.e.
\[
n\;\ge\;\frac{32\kappa^{2}\ln(8/\rho)}{\varepsilon^{2}\rho^{2}}\;=\;\frac{32\cdot\ln 80}{4\times10^{-6}}\;\approx\;3.5\times10^{7}.
\]
The ratio to (a) is $\ln(8/\rho)/(\alpha(1-\alpha))\approx49$: the expectation-level Beta argument beats the high-probability union argument by exactly the variance-vs-tail gap.

\paragraph{(c) Black-box replicable binary search.} Proposition~\ref{prop:distfree} at $R=2^{32}$ (single-precision score grid) costs $\tOh(\log_2^{2}(R)\log(1/\delta)/(\varepsilon^{2}\rho^{2}))$; even suppressing the $\log(1/\delta)$ and all constants, $\log_2^{2}(2^{32})/(\varepsilon^{2}\rho^{2})=1024/(4\times10^{-6})\approx2.6\times10^{8}$, some $\sim350\times$ row (a). Note that (c) requires no margin assumption; the comparison quantifies the cost of Assumption~\ref{ass:margin}, and rows (a)/(b) quantify what the assumption provides.

\paragraph{Empirical anchor.} The measured frontier gives $n^{\ast}(0.04,\,0.1)\approx3.6\times10^{4}$ (Figure~\ref{fig:nstar}); row (a) at those parameters is $32\cdot0.09/(1.6\times10^{-3}\cdot10^{-2})=1.8\times10^{5}$: the certified constants are $\approx5\times$ conservative against practice on this testbed.

\section{Experimental details}\label{app:exp}

\begin{figure}[t]
\centering
\includegraphics[width=0.86\linewidth]{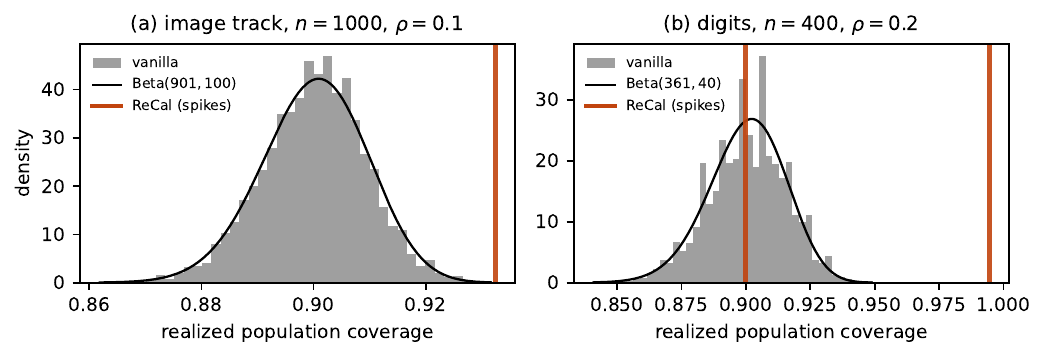}
\caption{Realized population coverage across calibration redraws: standard calibration follows the exact $\mathrm{Beta}(k,n{+}1{-}k)$ law; \ReCal{} collapses it onto at most a few shared grid atoms, all $\ge1-\alpha$; this mechanism underlies both replication and the selection bound.}
\label{fig:beta}
\end{figure}

\begin{figure}[t]
\centering
\includegraphics[width=0.94\linewidth]{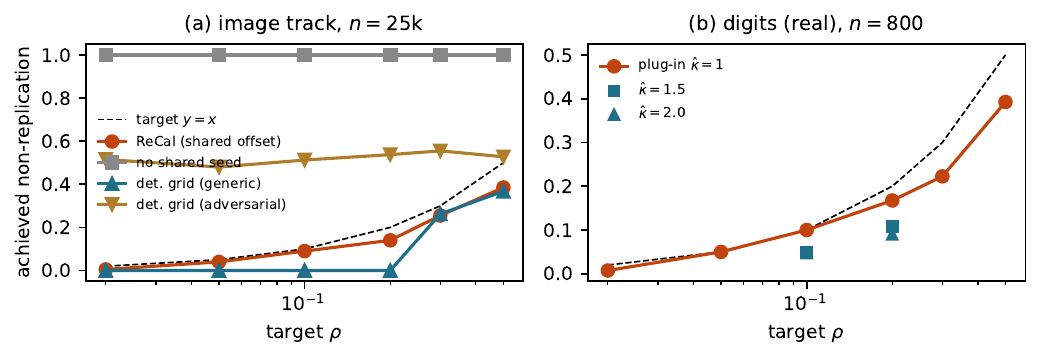}
\caption{Achieved non-replication vs.\ target $\rho$. (a) Image track, $n{=}25{,}000$: \ReCal{} tracks $y{=}x$ from below; no shared seed gives zero replication, and an adversarial $\beta/2$ score shift reduces the deterministic grid to $\approx1/2$ (Theorem~\ref{thm:impossible}a). (b) Digits, $n{=}800$: the $\hat\kappa{=}1$ plug-in shows excursions above the target; $\hat\kappa\in\{1.5,2\}$ restores compliance (Table~\ref{tab:safety}).}
\label{fig:rho}
\end{figure}

\begin{figure}[t]
\centering
\includegraphics[width=0.92\linewidth]{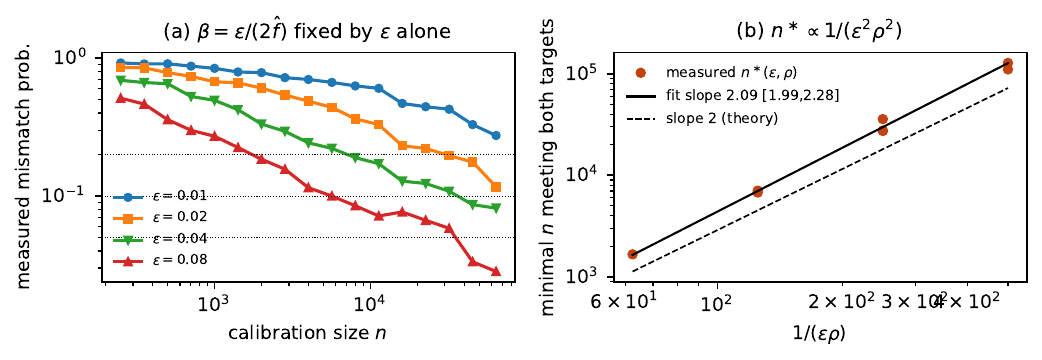}
\caption{\textbf{Measured sample-complexity frontier.} $\beta=\varepsilon/(2\hat f)$ is fixed by $\varepsilon$ alone. (a) Mismatch vs.\ $n$: decay exponents $-.49,-.50,-.43,-.49$ across $\varepsilon$ (predicted $n^{-1/2}$); small-$\varepsilon$ curves saturate where the theory requires ($\beta\ll$ threshold spread). (b) $n^\ast$ vs.\ $1/(\varepsilon\rho)$ fits slope $2.09$ (CI $[1.99,2.28]$; extrapolated points excluded), matching Corollary~\ref{cor:protocol} and Theorem~\ref{thm:lower}.}
\label{fig:nstar}
\end{figure}

\begin{table}[t]
\centering\small
\caption{\textbf{Real-data tracks}: digits, and Llama-3.2-1B next-token calibration on WikiText-103, including the risk-control instantiation (top-mass/top-$p$ sets with $0/1$ risk, calibrated CRC-style) on real model probabilities. Cross-model comparison: Table~\ref{tab:realllm}; the synthetic long-tail testbed is in Appendix~\ref{app:exp}.}
\label{tab:small}
\begin{tabular}{llcccc}
\toprule
 & Method & Coverage & Avg.\ size & Ident. & Agree \\
\midrule
\multirow{4}{*}{\shortstack[l]{digits\\($n{=}800$,\\$600$ pairs)}}
 & APS (derand.) & $.9001\pm.0109$ & $1.08$ & $.02$ & $.98$ \\
 & \ReCal{} $\rho{=}.2$ & $.9365\pm.0233$ & $1.12$ & $.81$ & $.98$ \\
 & \ReCal{} $\rho{=}.1$ & $.9626\pm.0341$ & $3.65$ & $.89$ & $.88$ \\
 & Det.-grid $2$-list ($\beta$ as $\rho{=}.2$) & $.9438\pm.0305$ & $1.13$ & $.63$ & $.94$ \\
\midrule
\multirow{5}{*}{\shortstack[l]{Llama-3.2-1B\\(real, WikiText,\\$n{=}5{,}000$,\\$100$ pairs)}}
 & APS (derand.) & $.9000\pm.0044$ & $83.7$ & $.04$ & $.55$ \\
 & \ReCal{} $\rho{=}.1$ & $.9297\pm.0183$ & $132.8$ & $.89$ & $.93$ \\
 & \ReCal{} $\rho{=}.05$ & $.9573\pm.0299$ & $10{,}586$ & $.96$ & $.92$ \\
 & CRC top-$p$ (standard) & $.9012\pm.0044$ & $1084$ & $.01$ & $.19$ \\
 & \ReCal-CRC $\rho{=}.1$ & $.9356\pm.0140$ & $12{,}717$ & $.96$ & $.96$ \\
\bottomrule
\end{tabular}
\end{table}

\begin{table}[t]
\centering\small
\caption{\textbf{The replication spectrum} (Definition~\ref{def:spectrum}) at $\rho{=}0.1$, $60$ pairs. ``Churn''\ = mean $|\cC(X)\triangle\cC'(X)|$ per test point, marginal and conditional on the \emph{pair} mismatching (with event counts). Baseline ``bag-$10$''\ averages ten bootstrap quantiles: agreement matches plain calibration and identity stays $0$; averaging does not produce replication. \ReCal{} concentrates disagreement into rare events of one grid cell; digits: marginal churn increases from $.019$ to $.60$ and agreement drops $.98\to.93$; Llama: $5.0\to15.1$ tokens for identity $0\to.90$ ($40$ pairs); GPT-2: the cell is smaller than baseline churn and both metrics improve ($20.1\to0.0$, $60$ mismatch-free pairs; Appendix~\ref{app:exp}).}
\label{tab:spectrum}
\begin{tabular}{llccccc}
\toprule
 & Method & Ident. & Agree & Jaccard & Churn (marg.) & Churn $\mid$ pair mism.\ (\#) \\
\midrule
\multirow{3}{*}{image} & derand. & $.02$ & $.66$ & $.92$ & $0.36$ & $0.37$\ \ (59) \\
 & bag-$10$ & $.00$ & $.63$ & $.91$ & $0.39$ & $0.39$\ \ (60) \\
 & \ReCal{} $.1$ & $.97$ & $.97$ & $.94$ & $0.19$ & $5.82$\ \ (2) \\
\midrule
\multirow{3}{*}{digits} & derand. & $.05$ & $.98$ & $.92$ & $0.019$ & $0.02$\ \ (57) \\
 & bag-$10$ & $.00$ & $.98$ & $.92$ & $0.017$ & $0.02$\ \ (60) \\
 & \ReCal{} $.1$ & $.93$ & $.93$ & $.92$ & $0.60$ & $9.01$\ \ (4) \\
\midrule
\multirow{3}{*}{\shortstack[l]{LLM\\(Llama-1B)}} & derand. & $.00$ & $.52$ & $.94$ & $5.0$ & $5.0$\ \ (40) \\
 & bag-$10$ & $.00$ & $.52$ & $.94$ & $4.8$ & $4.8$\ \ (40) \\
 & \ReCal{} $.1$ & $.90$ & $.92$ & $.94$ & $15.1$ & $151$\ \ (4) \\
\bottomrule
\end{tabular}
\end{table}

\paragraph{Environment and reproduction.} Every image-track, digits, frontier, attack, and synthetic-LM number in this paper regenerates deterministically in approximately three minutes on one CPU core; the four real-model tables regenerate from the cached logits in $6$--$10$ CPU-minutes per model, and logit extraction itself used one GPU. The full suite was rerun end-to-end before finalizing this draft, and the paper reports that single run. All randomness derives from fixed master seeds (pool $20260808$, instance-jitter $777$, analyst resampling $424242$) with documented per-experiment offsets, and every pair's shared offset seed is derived deterministically. A common interface (an $(N,K)$ logit matrix with $(N)$ labels) substitutes real cached model outputs for the image and next-token generators with no other change. Next-token extraction computes logits on a WikiText-103 slice from locally cached checkpoints and has been verified end-to-end (measured test top-1: GPT-2 $.402$, Pythia-1.4B $.504$, Qwen2.5-1.5B $.519$, Llama-3.2-1B $.516$, consistent with published values for these models, and spanning vocabularies from $50$k to $152$k tokens); image extraction computes class logits from a ViT-Base checkpoint (google/vit-base-patch16-224) on an ImageNet-layout folder; memory scales as $N\times K\times2$ bytes in half precision ($\approx1$\,GB for $10^{4}$ contexts at a $50$k vocabulary).

\paragraph{Pools.} \emph{Image-classification testbed}: $N{=}50{,}000$, $K{=}1{,}000$; labels uniform; per-example margin $\mu_i=\mathrm{softplus}(m+\xi_i)$, $\xi_i\sim\mathcal N(0,1)$, with $m$ tuned by bisection (Gumbel approximation to the max of $K{-}1$ standard normals) to a top-1 of $0.78$; logits $z_i=\mu_ie_{y_i}+\mathcal N(0,I_K)$. Achieved top-1 $79.0\%$, top-5 $89.9\%$ (a real ResNet-152's top-5 is $\approx94\%$; independent-coordinate logits thin the top-5 mass, a known limitation of this generator class). \emph{digits}: scikit-learn digits, stratified $50/50$ split, multinomial logistic regression ($C{=}0.5$); pool = the $899$ held-out points; test accuracy $95.7\%$. \emph{LLM-scale testbed}: $N{=}10{,}000$, $K{=}8{,}192$; per-context Zipf (exponent $1.1$) logits under a random vocabulary permutation and per-context temperature $e^{0.35\mathcal N(0,1)}$; labels drawn from the context distribution; the ``model'' sees temperature-$1.2$, noise-perturbed logits (deliberate miscalibration). Top-1 $8.9\%$, mean entropy $7.0$ nats.

\paragraph{Scores and methods.} APS score $s_i=V_{i,o_i}-u_ip_{i,(o_i)}$ with $V$ the sorted-probability cumulative sums, $o_i$ the rank of the true label, and $u_i\in[0,1)$ a fixed deterministic per-instance uniform (the hash surrogate of Remark~\ref{rem:jitter}; shared by all analysts, independent of calibration draws). ``Fresh randomization'' redraws $u$ per analyst and per test evaluation, as in the original APS. RAPS adds $\lambda\max(0,o-k_{\mathrm{reg}})$ with $(k_{\mathrm{reg}},\lambda)=(5,0.05)$. Thresholds use $k=\lceil(1-\alpha)(n+1)\rceil$; \ReCal{} rounds up on the shared offset grid; the seedless variant uses the fixed grid $\beta\mathbb Z$. Sizes and agreement metrics are computed from the monotone per-label score matrices on a fixed evaluation subsample ($10{,}000$ points on the image track, the full pool for digits; $4{,}000$ points for the synthetic, GPT-2, and Pythia pools and $2{,}000$ for Qwen and Llama): first $25$ pairs per configuration in Tables~\ref{tab:imagenet}--\ref{tab:small}, $60$ pairs in Table~\ref{tab:spectrum}. Coverage is computed on the full pool.

\paragraph{Grid rule.} $\beta=\hat\kappa\cdot\sqrt{2\alpha(1-\alpha)/n}\,/(\hat f\rho)$ with $\hat f$ from a public pilot of $2{,}000$ pool draws (seed $1234$), window $\pm0.02$ around the pilot quantile; $\hat f=0.84$ (image), $1.31$ (digits), $0.75$/$0.79$ (synthetic LLM testbed, APS/required-mass); per-model pilots for the real-LM runs are listed in the provenance paragraph above ($0.64$--$1.15$ APS, $1.76$--$4.38$ required-mass). $\hat\kappa=1$ unless stated; the paper's recommended default for deployment is $\hat\kappa=1.5$ (F2, Table~\ref{tab:safety}).

\paragraph{Real LLM track: provenance and per-model details.} Logits were produced on locally cached checkpoints (single GPU, WikiText-103 test contexts of length $512$, half precision) and consumed through the common logits interface with no other change; suite runtime $6$--$10$ minutes per model on one CPU core. Per model ($N$, $K$, top-1/top-5, entropy, APS pilot $\hat f$, mass-score pilot): GPT-2 ($10{,}000$, $50{,}257$, $.402/.624$, $3.59$ nats, $1.04$, $1.76$); Pythia-1.4B ($10{,}000$, $50{,}304$, $.504/.729$, $2.46$, $1.10$, $3.83$); Qwen2.5-1.5B ($10{,}000$, $151{,}936$, $.519/.754$, $2.18$, $0.64$, $4.35$); Llama-3.2-1B ($5{,}000$, $128{,}256$, $.516/.751$, $2.31$, $1.03$, $4.38$). Tables~\ref{tab:small} and \ref{tab:realllm} and the LLM rows of Tables~\ref{tab:spectrum} and \ref{tab:safety} come from these per-model runs, the LLM bars of Figure~\ref{fig:hero} from the Llama run, and every other number from the canonical synthetic-suite run. All nine per-model runs (the four families above and the five-point scale sweep below) produce mutually bit-identical image, digits, and frontier sections; the shared-seed design makes this checkable, and we verified it. Across \emph{machines}, every discrete statistic reproduces exactly (the frontier slope to all printed digits, the attack counts, all fixed-pool thresholds) while floating-point aggregates agree to $\sim\!10^{-5}$; the one genuine environment sensitivity is the digits track, which refits its logistic model at run time: a different scikit-learn build fits a marginally different model (test accuracy $.9555$ vs.\ $.9566$, pilot $\hat f$ $1.38$ vs.\ $1.31$), moving the $\rho{=}.1$ identity from $.887$ to $.930$; these alternate-build values come from a separate rerun. We disclose it, and it supplies a third independent stream for the Monte-Carlo paragraph below. For reference, the synthetic long-tail testbed at the matched operating point gives: APS derand.\ cov/size/ident/agree $.9005/2570/.01/.06$; \ReCal{} $\rho{=}.1$: $.9294/3108/.87/.84$ and $\rho{=}.05$: $.9539/4116/.98/.96$; CRC $.8997\to.9278$ with identity $.03\to.94$; spectrum churn $87\to80$ tokens; safety $.093/.053$ at $\hat\kappa{=}1/1.5$; the real models confirm this pattern qualitatively, except for the tail-geometry effect (F8) that only real vocabularies exhibit.

\paragraph{Model-scale sweep (Pythia $70$M--$1.4$B).} Five sizes of one family (one tokenizer, one $50{,}304$-token vocabulary, a $20\times$ parameter range) isolate the effect of model quality (Table~\ref{tab:scale}, Figure~\ref{fig:scale}). First, the protocol required no per-scale tuning across this sweep: the pilot density stays in $[0.96,1.15]$ (the probability-integral-transform heuristic behind Assumption~\ref{ass:margin} holds at every scale), and the identical grid rule (same $\hat\kappa$, no per-model tuning) delivers identity $.85$--$.94$ at $\rho{=}.1$ and $.94$--$.96$ at $\rho{=}.05$, with $\hat\kappa{=}1.5$ at or below target everywhere except the $1.4$B tie already noted. Second, within this sweep, scale alone does not produce replication: a $20\times$ larger model shrinks standard sets $559\to108$ tokens and lifts two-run agreement only from $.16$ to $.40$; larger models reduce churn, but exact agreement remains at most $.40$. The size cost varies non-monotonically within $46$--$84\%$, consistent with the tail-geometry effect of F8 at smaller magnitude.

\begin{table}[h]
\centering\small
\caption{Pythia scale sweep (WikiText-103, $n{=}5{,}000$, $\rho{=}0.1$; $100$ table pairs and $300$ safety pairs per size). Identities at $\rho{=}.05$ are $.94$--$.96$ across the sweep.}
\label{tab:scale}
\begin{tabular}{lccccccc}
\toprule
Model & Top-1 & Entropy & $\hat f$ & Std.\ size / agree & \ReCal{} ident. & Cost & Non-rep.\ ($\hat\kappa{=}1.5$) \\
\midrule
Pythia-70m  & $.343$ & $3.89$ & $1.15$ & $559$ / $.16$ & $.87$ & $+46\%$ & $.070\pm.015$ \\
Pythia-160m & $.412$ & $3.53$ & $1.00$ & $342$ / $.28$ & $.85$ & $+47\%$ & $.060\pm.014$ \\
Pythia-410m & $.467$ & $2.75$ & $0.96$ & $162$ / $.32$ & $.89$ & $+65\%$ & $.070\pm.015$ \\
Pythia-1b   & $.489$ & $2.54$ & $0.97$ & $127$ / $.38$ & $.94$ & $+84\%$ & $.033\pm.010$ \\
Pythia-1.4B & $.504$ & $2.46$ & $1.10$ & $108$ / $.40$ & $.93$ & $+58\%$ & $.103\pm.018$ \\
\bottomrule
\end{tabular}
\end{table}

\begin{figure}[h]
\centering
\includegraphics[width=0.62\linewidth]{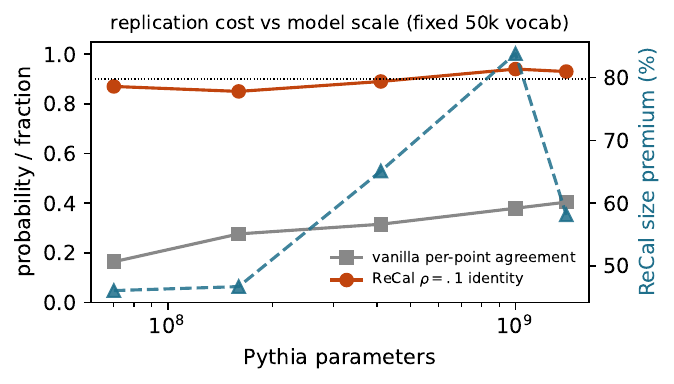}
\caption{Replication cost vs.\ model scale at fixed vocabulary: standard-calibration agreement (gray), \ReCal{} identity at $\rho{=}.1$ (dark), and the \ReCal{} size cost (right axis).}
\label{fig:scale}
\end{figure}

\paragraph{New experiments (specifications).}
\emph{Spectrum} (Table~\ref{tab:spectrum}): $60$ pairs per track at $\rho{=}0.1$; per pair we compute exact set masks on the evaluation subsample and report identity, pointwise agreement, mean Jaccard (per test point $|\cC\cap\cC'|/\allowbreak\max\{|\cC\cup\cC'|,1\}$, averaged over points and pairs), marginal churn $\E_X|\cC\triangle\cC'|$, and churn conditional on the \emph{pair} mismatching (averaged over mismatching pairs only, with event counts; the marginal version, averaged over all pairs, would dilute rare events by construction). ``bag-$10$'' recomputes each analyst's threshold as the mean of ten bootstrap-resample conformal quantiles of that analyst's own sample.
\emph{Safety} (Table~\ref{tab:safety}): fresh pair streams, independent of the main tables ($400$ pairs per digits configuration, $300$ per LLM configuration).
\emph{Cost curve} (Figure~\ref{fig:cost}): $40$ calibration draws per $\rho$, sizes on the first $15$.

\begin{figure}[t]
\centering
\includegraphics[width=0.52\linewidth]{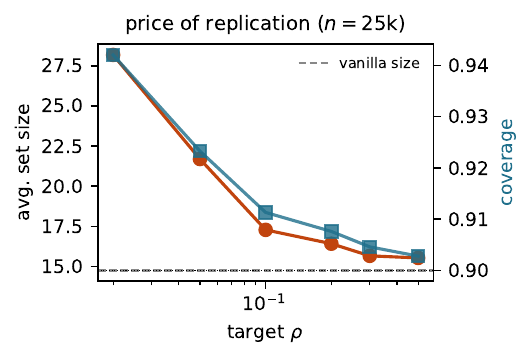}
\caption{The price of replication at $n{=}25$k: average set size and coverage vs.\ target $\rho$.}
\label{fig:cost}
\end{figure}

\emph{Sample-complexity frontier} (Figure~\ref{fig:nstar}): $\varepsilon\in\{.01,.02,.04,.08\}$, $\beta=\varepsilon/(2\hat f)$ fixed by $\varepsilon$ alone; $n$ on a $17$-point geometric grid $250\to64{,}000$; $600$ pairs per $(\varepsilon,n)$ cell measure the mismatch probability $p$ and the coverage inflation. Per $\varepsilon$, $\log p$ is regressed on $\log n$ over the unsaturated window $p\in[0.02,0.45]$ (measured decay slopes $-0.49,-0.50,-0.43,-0.49$; the theoretical exponent is $-\tfrac12$); $n^{\ast}(\varepsilon,\rho)$ solves the fitted line at $p=\rho$. Six of the twelve $(\varepsilon,\rho)$ combinations resolve inside the measured grid and enter the frontier fit (slope $2.09$, bootstrap CI $[1.99,2.28]$); two resolve just beyond it and are shown as open markers, excluded from the fit; four are censored beyond $1.28\times10^{5}$ and reported as such, exactly where the theory places them ($n^{\ast}\gtrsim7\times10^{5}$). The sweep holds $\alpha=0.1$ and $\kappa\approx1$ fixed: the $(\varepsilon\rho)^{-2}$ law is what is measured; the $\alpha(1-\alpha)$ and $\kappa^{2}$ factors of Corollary~\ref{cor:protocol} and Theorem~\ref{thm:lower} are not separately probed.
\emph{Selection attack} (Figures~\ref{fig:hero}b): $n{=}5{,}000$, $M{=}20$ redraws, $300$ trials, pre-registered per-trial offset; the adversary deploys the draw with minimal population coverage (equivalently minimal sets, by monotonicity). At $\rho{=}0.1$: standard selected $.8920\pm.0024$ vs.\ honest $.9003$; \ReCal{} selected $.9170$ vs.\ honest $.9239$; $\ge2$ distinct \ReCal{} outputs in $30.7\%$ of trials (the $(M{-}1)\rho{=}1.9$ union bound is vacuous here; the guarantee-preservation mechanism is what binds). At $\rho{=}0.01$: $\ge2$ distinct outputs in $1.3\%$ of trials against the bound $19\%$; selected $.9793$ vs.\ honest $.9797$.
\emph{Risk control} (Table~\ref{tab:small}): the required-mass score $M_i=V_{i,o_i}$ (cumulative sorted-probability mass through the true label, no jitter); $\hat\lambda$ is its conformal quantile (the CRC calibration of the monotone $0/1$ risk), and the deployed sets are $\{y:\text{mass before }y<\hat\lambda\}$ (calibrated top-$p$); \ReCal{} rounds $\hat\lambda$ on the shared grid in mass space; $100$ pairs.
\emph{Corollary-level checks} (image pool, $\rho{=}0.1$). \emph{Unequal calibration sizes} (Corollary~\ref{cor:twolabs}), $\beta$ set by the smaller analyst's rule: $(n_A,n_B)=(25\text{k},5\text{k})$ measures non-replication $.065\pm.012$ against the corollary's leading term $.102$, and $(25\text{k},1\text{k})$ measures $.052\pm.011$ against $.085$ ($400$ pairs each; both analysts' coverage $\ge.92$); the bound holds with $\approx1.6\times$ slack. \emph{Score mismatch} (Corollary~\ref{cor:robust}), analyst $B$'s scores perturbed within sup norm $\eta_s$: the worst case (constant $+\eta_s$ shift) measures $.087/.123/.190$ at $\eta_s/\beta=.05/.1/.2$ against the additive bound $\text{base}+\eta_s/\beta=.143/.193/.293$ ($300$ pairs, base $.093\pm.017$), tracking the exact collision formula and converging to $\eta_s/\beta$ once $\eta_s$ exceeds the threshold spread; mean-zero noise of the \emph{same} sup norm adds nothing ($.080$--$.110$, statistically flat); the $\eta_s/\beta$ term is a sup-norm worst case, as stated. The population-mismatch term $\eta_P/f_{\min}$ and the membership bound of Corollary~\ref{cor:robust} remain analytic.

\paragraph{Real ImageNet run.} Cross-validated ResNet-50 predicted probabilities on the ILSVRC-2012 validation set ($N{=}50{,}000$, $K{=}1{,}000$; the float16-quantized probability matrix of \citet{northcutt2021}; recomputed top-1 $.727$, top-5 $.895$) enter the pipeline as $\log p$ logits; nothing else changes. At $n{=}25{,}000$, $100$ pairs: baseline identity $0.00$ (fresh and derandomized), agreement $.65/.87$, size $4.32$; \ReCal{} at $\rho{=}.1$: identity $.87$ ($\hat\kappa{=}1$; $.98$ at $\hat\kappa{=}1.5$), coverage $.910$, size $5.05$ ($+17\%$); at $\rho{=}.01$: identity $.99$, size $13.2$; \ReCal-RAPS $.94$ at $5.93$; seedless two-list mass $1.000$. Pilot $\hat f{=}2.03$.

\paragraph{Clinical site-split track.} The four UCI Heart Disease sites \citep{detrano1989} (Cleveland Clinic $303$, Hungarian Institute of Cardiology $294$, V.A.\ Long Beach $200$, University Hospital Zurich $123$; CC BY 4.0; principal investigators A.~Janosi, W.~Steinbrunn, M.~Pfisterer, R.~Detrano). A logistic score on seven cross-site-complete features is trained on a public Cleveland split of $150$ patients ($.85$ training accuracy) and frozen, with deterministic instance jitter and a public pilot on the training scores ($\hat f{=}.67$); analysts calibrate on their own hospital with unequal sizes ($n_A{=}120$ Cleveland, $n_B{=}120$--$200$ elsewhere; Corollary~\ref{cor:twolabs}), $\hat\kappa{=}1.5$, $400$ trials. Baseline identity is $0.0$ for every pair. Cleveland--Budapest at $\rho{=}.3$: identity $.91$, per-site coverage $.982/.986$, sizes $1.84/1.89$ of $2$; at $\rho{=}.5$: identity $.84$, sizes $1.74/1.81$. Across all pairs (score Kolmogorov--Smirnov shift $.08$--$.17$): identity $.78$--$.96$, coverage $.97$--$.99$. Identity is affordable at hospital scale only at coarse $\rho$ and near-vacuous sets: Corollary~\ref{cor:sizelb} in the field.

\paragraph{$\alpha$-dependence probe.} At fixed grid width $\beta^{\ast}$ (the protocol width at $\alpha{=}.1$, $\hat\kappa{=}1$) and $n{=}25{,}000$ on the image pool, mismatch was measured over $800$ pairs at each $\alpha\in\{.05,.1,.2,.3,.4,.5\}$ in two designs. On the raw APS score, whose pilot density at the working quantile varies $22\times$ over the sweep ($\hat f(q_\alpha)$ from $0.45$ to $9.7$), observed mismatch falls from $.069$ to $.0025$ and matches the plug-in prediction $B_n(\alpha)/(\hat f(q_\alpha)\beta^{\ast})$ with through-origin proportionality $0.97$ and $R^{2}=.99$: the bound's joint $(\alpha,f)$ structure holds at unit constant. On a rank-transformed copy of the pool ($f\equiv1$ exactly), which isolates the $\alpha(1-\alpha)$ factor, mismatch increases from $.028$ to $.073$ as predicted and lies below the theoretical bound $B_n(\alpha)/\beta^{\ast}$ at every level ($.037$ to $.083$; against $\sqrt{\alpha(1-\alpha)}$, $R^{2}=.54$ at $800$ pairs, binomial standard errors $\approx.008$).

\paragraph{Adaptive-limit check.} The limits of Proposition~\ref{prop:adaptive} are visible at the protocol's own scale. On the rank-transformed pool ($f\equiv1$, so the pilot ratio drops out), at $n{=}1.6\times10^{5}$, $\rho{=}.1$, $10^{4}$ pairs: mismatch$/\rho=.382\pm.019$ against the limit $\sqrt{2/\pi}/2\approx.399$; inflation$\times\rho/B_n$ equals $1.001$--$1.005$ from $n{=}10^{4}$ on (limit $1$); and the $\sqrt{2/\pi}$ width of the last claim achieves mismatch $.100$ at the $\rho{=}.1$ target. On the raw APS score the pilot ratio enters as predicted: inflation$\times\rho/B_n$ rises $.56\to.93$ toward $f(q)/\hat f=1.04$, and mismatch$/\rho$ moves toward $.399\,\hat f/f(q)=.38$ as $\beta_n\downarrow0$.

\paragraph{Finite-pool ties.} Because the ``population'' is a finite pool, redrawn indices repeat and standard thresholds can tie: the nonzero baseline identity rates ($\le.02$ image, $.02$--$.05$ digits, $.00$--$.04$ on the real LLMs, largest for Llama whose pool is $N{=}5{,}000$) are pool-atom artifacts that vanish for a continuous population (Proposition~\ref{prop:vanilla}); they slightly help all methods equally.

\begin{table}[h]
\centering\small
\caption{Safety factor $\hat\kappa$ (digits: $n{=}800$, $400$ fresh pairs per row; real LLMs: $n{=}5{,}000$, $300$ pairs each, all at $\rho{=}0.1$). ``Non-rep.''\ is achieved non-replication ($\pm$\,s.e.); target is $\rho$. Pre-registered $\hat\kappa=1.5$ is at or below target in every configuration up to one statistical tie (Pythia $.103\pm.018$ against $.10$, indistinguishable from its own $\hat\kappa{=}1$ stream). Qwen's sizes reflect the tail geometry of F8. Average size is \emph{not} monotone in $\hat\kappa$ on the small digits pool: as $\beta$ grows the rounded threshold lands in different cells of a spiky score distribution, so size can dip before jumping (coverage inflation, by contrast, is monotone as theory requires).}
\label{tab:safety}
\begin{tabular}{llcccc}
\toprule
 & Setting & $\hat\kappa$ & Non-rep. & Cov.\ inflation & Avg.\ size \\
\midrule
\multirow{6}{*}{digits}
 & $\rho{=}0.1$ & $1.0$ & $0.075\pm0.013$ & $0.065$ & $3.79$ \\
 & $\rho{=}0.1$ & $1.5$ & $0.050\pm0.011$ & $0.074$ & $2.89$ \\
 & $\rho{=}0.1$ & $2.0$ & $0.048\pm0.011$ & $0.082$ & $6.47$ \\
 & $\rho{=}0.2$ & $1.0$ & $0.160\pm0.018$ & $0.036$ & $1.11$ \\
 & $\rho{=}0.2$ & $1.5$ & $0.107\pm0.015$ & $0.053$ & $1.15$ \\
 & $\rho{=}0.2$ & $2.0$ & $0.090\pm0.014$ & $0.064$ & $4.67$ \\
\midrule
\multirow{8}{*}{\shortstack[l]{LLM (real)\\all $\rho{=}0.1$}}
 & GPT-2 & $1.0$ & $0.103\pm0.018$ & $0.026$ & $526$ \\
 & GPT-2 & $1.5$ & $0.050\pm0.013$ & $0.038$ & $589$ \\
 & Pythia-1.4B & $1.0$ & $0.097\pm0.017$ & $0.027$ & $187$ \\
 & Pythia-1.4B & $1.5$ & $0.103\pm0.018$ & $0.041$ & $253$ \\
 & Qwen2.5-1.5B & $1.0$ & $0.060\pm0.014$ & $0.051$ & $15575$ \\
 & Qwen2.5-1.5B & $1.5$ & $0.053\pm0.013$ & $0.063$ & $60880$ \\
 & Llama-3.2-1B & $1.0$ & $0.073\pm0.015$ & $0.030$ & $139$ \\
 & Llama-3.2-1B & $1.5$ & $0.043\pm0.012$ & $0.044$ & $147$ \\
\bottomrule
\end{tabular}
\end{table}

\paragraph{Monte Carlo precision.} Identity rates carry binomial standard errors: $.95$ over $100$ pairs is $\pm.022$; the digits $\rho{=}.1$ table entry ($.887$ over $600$ pairs, i.e.\ non-replication $.113\pm.013$) and the independent $400$-pair safety stream ($.075\pm.013$) are $2.1$ combined standard errors apart and bracket the $.10$ target; a full regeneration on a second machine, whose scikit-learn fit differs as noted above, gave $.070\pm.010$ on its own $600$-pair stream; the three independent streams straddle the target, reflecting the plug-in variability at $n{=}800$ and which makes $\hat\kappa{=}1.5$ the recommended default. A $0$-failure outcome over $100$ pairs (image track, $\rho{=}.01$) certifies only a $\le.03$ rate at $95\%$ confidence by the rule of three; we state it as such.

\end{document}